\documentclass[11pt]{article} 

\usepackage[letterpaper,margin=1in]{geometry}
\usepackage[letterpaper,margin=1in]{geometry}

\def\colorful{0}
\ifnum\colorful=1

\newcommand{\blue}[1]{{\color{blue} #1}}
\else

\newcommand{\blue}[1]{{#1}}

\fi

\newif\ifred
\redtrue

\usepackage[T1]{fontenc}
\usepackage{microtype}

\usepackage{titlesec}
\titleformat*{\paragraph}{\bfseries}

\usepackage{amsthm,amsmath,amssymb,amsfonts,amssymb,mathtools}
\usepackage{amsmath,amssymb,amsfonts,amssymb,mathtools}
\usepackage{xcolor}
\usepackage{empheq}
\usepackage{dsfont}
\usepackage{mathrsfs}
\usepackage{graphicx}
\usepackage{enumitem}
\usepackage{aliascnt} 
\usepackage{xspace}
\usepackage{bbm}

\usepackage{caption}
\usepackage{tikz}
\usetikzlibrary{patterns}
\usetikzlibrary{patterns}
\usetikzlibrary{arrows,shapes,automata,backgrounds,petri,positioning}
\usetikzlibrary{shadows}
\usetikzlibrary{calc}
\usetikzlibrary{spy}
\usetikzlibrary{angles, quotes}
\usetikzlibrary{matrix}
\usepackage{pgf,pgfplots}
\usepackage{pgfmath,pgffor}
\pgfplotsset{compat=1.17}

\usepackage{framed}
 \usepackage{algorithm}
\usepackage[noend]{algpseudocode}
\definecolor[named]{ACMBlue}{cmyk}{1,0.1,0,0.1}
\definecolor[named]{ACMYellow}{cmyk}{0,0.16,1,0}
\definecolor[named]{ACMOrange}{cmyk}{0,0.42,1,0.01}
\definecolor[named]{ACMRed}{cmyk}{0,0.90,0.86,0}
\definecolor[named]{ACMLightBlue}{cmyk}{0.49,0.01,0,0}
\definecolor[named]{ACMGreen}{cmyk}{0.20,0,1,0.19}
\definecolor[named]{ACMPurple}{cmyk}{0.55,1,0,0.15}
\definecolor[named]{ACMDarkBlue}{cmyk}{1,0.58,0,0.21}

\usepackage[colorlinks,citecolor=blue,linkcolor=magenta,bookmarks=true]{hyperref}
 \usepackage[nameinlink]{cleveref}
\creflabelformat{ineq}{#2{\upshape(#1)}#3}
\crefname{sub}{Subsection}{Subsection}
\creflabelformat{Subsection}{#2{\upshape(#1)}#3}
\crefname{sdp}{SDP}{SDP}
\creflabelformat{sdp}{#2{\upshape(#1)}#3}
\crefname{lp}{LP}{LP}
\creflabelformat{lp}{#2{\upshape(#1)}#3}
\usepackage{aliascnt}
\usepackage{cleveref}
\crefname{ineq}{Inequality}{Inequality}
\creflabelformat{ineq}{#2{\upshape(#1)}#3}
\crefname{sub}{Subsection}{Subsection}
\creflabelformat{Subsection}{#2{\upshape(#1)}#3}
\crefname{sdp}{SDP}{SDP}
\creflabelformat{sdp}{#2{\upshape(#1)}#3}
\crefname{lp}{LP}{LP}
\creflabelformat{lp}{#2{\upshape(#1)}#3}



\newtheorem{theorem}{Theorem}[section]

\newtheorem{lemma}[theorem]{Lemma}

\newtheorem{informal theorem}[theorem]{Theorem (informal statement)}

\newtheorem{corollary}[theorem]{Corollary}
 \newtheorem{claim}[theorem]{Claim}
\newtheorem{fact}[theorem]{Fact}

\newtheorem{definition}[theorem]{Definition}

\newcommand{\eqdef}{\coloneqq}

\newcommand{\lp}{\left}
\newcommand{\rp}{\right}
\newcommand\norm[1]{\left\| #1 \right\|}

\renewcommand\vec[1]{\mathbf{#1}}
\DeclareMathOperator*{\pr}{\mathbf{Pr}}
\DeclareMathOperator*{\E}{\mathbf{E}}
\newcommand{\proj}{\mathrm{proj}}

\newcommand{\cN}{\mathcal{N}}

\DeclareMathOperator*{\argmin}{argmin}
\DeclareMathOperator*{\argmax}{argmax}

\newcommand{\e}{\mathbf{e}}

\newcommand{\err}{\mathrm{err}}

\renewcommand{\H}{\mathcal H }

\newcommand{\R}{\mathbb{R}}
\newcommand{\s}{\mathbb{S}}
\newcommand{\Z}{\mathbb{Z}}

\newcommand{\eps}{\epsilon}
\newcommand{\dtv}{d_{\mathrm TV}}
\newcommand{\poly}{\mathrm{poly}}

\newcommand{\var}{\mathbf{Var}}
\newcommand{\cov}{\mathbf{Cov}}

\newcommand{\sgn}{\mathrm{sign}}
\newcommand{\sign}{\mathrm{sign}}

\newcommand{\Ind}{\mathds{1}}

\newcommand{\x}{\vec x}
\newcommand{\w}{\vec w}

\newcommand{\opt}{\mathrm{opt}}

\newcommand{\iid}{{i.i.d.}\ }
\newcommand{\abs}[1]{\lp| #1 \rp|}

\renewcommand\Pr{\pr}

\begin{document}

\title{Polynomial-Time Robust Multiclass Linear Classification under Gaussian Marginals
}

\footnotetext[1]{Authors are listed in alphabetical order.}
\author{
Ilias Diakonikolas\thanks{Supported by NSF Medium Award CCF-2107079, 
ONR Award number N00014-25-1-2268, and an H.I. Romnes Faculty Fellowship.}\\
University of Wisconsin-Madison\\
\texttt{ilias@cs.wisc.edu}
\and
Giannis Iakovidis\thanks{Supported by ONR award number N00014-25-1-2268.}
\\
University of Wisconsin-Madison\\
\texttt{iakovidis@wisc.edu}
\and
Mingchen Ma\thanks{Supported by NSF Award  CCF-2144298 (CAREER).}\\
University of Wisconsin-Madison\\
\texttt{mingchen@cs.wisc.edu}
}

\maketitle

\begin{abstract}
We study the task of agnostic learning of multiclass linear classifiers under the Gaussian distribution. Given labeled examples $(x, y)$ from a distribution over $\mathbb{R}^d \times [k]$, with Gaussian $x$-marginal, the goal is to output a hypothesis whose error is comparable to that of the best $k$-class linear classifier. While the binary case $k=2$ has a well-developed algorithmic theory, much less is known for $k \ge 3$. Even for $k=3$, prior robust algorithms incur exponential dependence on the inverse of the desired accuracy in both complexity and representation size. In this work, we develop new structural results for multiclass linear classifiers and use them to design fully polynomial-time robust learners with dimension-independent error guarantees. Our first result shows that the standard multiclass perceptron algorithm requires super-polynomially many samples and updates, even with clean labels and Gaussian marginals, revealing a basic obstruction absent in the binary case. Our main positive result is a pairwise improper-learning framework which yields an efficient learner with error $\widetilde O(k^{3/2}\sqrt{\mathrm{opt}})+\epsilon$ for general $k$. Additionally, we develop a sharper localization-based framework which leads to error $O(\mathrm{opt})+\epsilon$ for $k=3$, and error $\mathrm{poly}(k)\mathrm{opt}+\epsilon$ for geometrically regular $k$-class linear classifiers.
\end{abstract}

\thispagestyle{empty}

\newpage

\setcounter{page}{1}

\section{Introduction}

A $k$-multiclass linear classifier ($k$-MLC) 
is any function $f: \R^d \to [k]$ of the form 
$f(x) = \argmax_{i \in [k]} (w_i\cdot x)$, where $w_i \in \R^d$. 
Multiclass Linear Classification---the task of learning 
an unknown multiclass linear classifier from random labeled examples---is 
a textbook machine learning problem~\cite{shalev2014understanding}, 
which has been extensively studied both theoretically and empirically 
\cite{platt1999large,hsu2002comparison,aly2005survey,duan2005best, 
tewari2007consistency,kakade2008efficient,huang2011extreme,beygelzimer2019bandit}.

The sample complexity of PAC learning MLC is fairly well-understood, 
even in the presence of noise. 
Specifically, standard arguments, see, 
e.g.,~\cite{shalev2014understanding}, 
give that $\poly(d, k, 1/\eps)$ samples information-theoretically
suffice to achieve 0-1 error $\opt+\eps$, 
where $\opt$ is the optimal error achievable 
by any function in the class. 
Yet the computational complexity of noisy multiclass linear classification is much less understood. 
In the realizable setting (i.e., in the presence of clean labels), 
the problem is solvable in polynomial-time via 
a reduction to linear programming without any distributional assumptions.

In the presence of label noise, the complexity landscape becomes significantly 
more intricate. In the distribution-free setting, 
the problem is known to be computationally intractable, even for $k=2$. Specifically, in the distribution-free agnostic model~\cite{haussler1992decision, kearns1992toward}, computational hardness is known for approximating $\opt$ to 
within any constant factor~\cite{daniely2016complexity,tiegel2023hardness}. 
Moreover, even for the special case of 
Random Classification Noise~\cite{angluin1988learning}---under which binary linear classification is efficiently solvable---recent work~\cite{diakonikolas2025statistical} shows that as soon as $k$ increases 
from $2$ to $3$, it becomes computationally hard to output any hypothesis that 
outperforms random guessing. 

These computational barriers motivate the study of efficient 
robust algorithms for multiclass linear classification under structured 
distributions, most prominently under the Gaussian distribution.

\begin{definition}[(Semi)-Agnostically Learning Multiclass Linear Classifiers] \label{def:agn}
 Let $D$ be a distribution over $\R^d \times [k]$ with $x$-marginal 
 $D_x=\cN(0,I)$. For a hypothesis $h: \R^d \to [k]$, define $\err(h):= \Pr_{(x,y)\sim D} [h(x)\neq y]$.
 Let $\opt:= \min_{f \in C_k} \err(f)$, where $C_k$ is the class of $k$-multiclass linear classifiers. Given \iid sample access to $D$, the goal is to output some hypothesis $\widehat{h}$ such that with high probability $\err(\widehat{h}) \le F(\opt)+\eps$, for some function $F: \R \to \R$. 
\end{definition}

Definition~\ref{def:agn} is the standard notion of (semi)-agnostic PAC learning. 
The (exact) agnostic setting corresponds to the special case where 
$F(\opt) = \opt$. We use the term ``semi-agnostic'' for the setting where $F(\opt)$ 
is a nondecreasing function of $\opt$---independent of the dimension $d$, 
which satisfies $\lim_{t \to 0} F(t) = 0$.

The design of efficient and robust learners for multiclass linear classifiers 
under the Gaussian distribution has received significant attention in recent years. 
For binary linear classification ($k=2$), known hardness results~\cite{diakonikolas2020near,tiegel2023hardness,diakonikolas2023near} indicate that even under Gaussian marginals, 
achieving error $\opt+\eps$ requires complexity $d^{\poly(1/\eps)}$. 
This suggests that the strongest guarantees attainable in fully polynomial time 
are of the form $C\opt+\eps$ for some $C>1$. 
Building on the classical perceptron algorithm~\cite{rosenblatt1958perceptron}, 
a long line of work~\cite{yan2017revisiting,awasthi2017power,DKS18a, diakonikolas2020non, diakonikolas2022learning,diakonikolas2024active} has developed efficient algorithms achieving error $O(\opt)+\eps$. These advances have further inspired progress on robust learning of generalized linear models and single-index models~\cite{frei2020agnostic, diakonikolas2020approximation,diakonikolas2022learning, vardi2021learning, awasthi2022agnostic, guo2024agnostic, zarifis2024robustly,zarifis2025robustly,wang2026robustly,diakonikolas2026robust}.

This progress notwithstanding, the algorithmic aspects 
of {\em multiclass} linear classification remain poorly understood, 
even in the simplest nontrivial case of $k=3$---due to the significantly more intricate geometry of the class.
Specifically, even under Gaussian marginals, the best-known (and only) algorithm~\cite{DIKZ25robust} achieves error $O(\opt)+\eps$ 
with sample complexity $d 2^{\poly(k/\eps)}$ and sample-polynomial time. 
Moreover, the resulting hypothesis requires $2^{\poly(k/\eps)}$ space to store. Such a gap motivates the following fundamental question:
\begin{quote}
{\em 
What error guarantees are achievable in $\mathrm{poly}(d,k,1/\epsilon)$ time
for (semi)-agnostically learning multiclass linear classifiers 
under Gaussian marginals?}
\end{quote}
In this work, we develop the first fully polynomial-time algorithms 
with {\em dimension-independent} error guarantees for this problem.
Along the way, we identify a fundamental obstruction to perceptron-based methods 
and prove new structural results that enable pairwise and localized learning.

A natural first attempt is to extend the perceptron-based robust learning paradigm from the binary case to multiclass linear classifiers. Our first result shows that this attempt fails in a strong sense: the standard multiclass perceptron \cite{crammer2003ultraconservative,beygelzimer2019bandit} may require super-polynomially many updates even for clean labels and Gaussian marginals.
We informally summarize the result below and refer to \Cref{thm:lb-main body} for a formal statement.

\begin{theorem}[Failure of Multiclass Perceptron under Gaussian]
There exists a distribution $D$ over $\R^d\times [k]$ with Gaussian marginal and labels generated by some $k$-MLC with $k \le d$,
such that the multiclass perceptron requires $k^{\Omega(\log k)}$ samples and iterations to achieve error $1/k$.
\end{theorem}

Thus, the obstacle is not only label noise: even in the realizable Gaussian setting,
the usual multiclass weight-matrix representation can be so ill-conditioned that
perceptron-style dynamics make inverse super-polynomially small progress. 
Motivated by this obstruction, we use an improper pairwise representation. 
Instead of learning a single matrix of class weights, 
the algorithm learns one separator for each pair of labels 
and aggregates the resulting comparisons into a tournament-style
classifier. The main technical challenge is that each binary subproblem is only partially observed:
examples useful for distinguishing labels $i$ and $j$ are obtained only after
conditioning on the event $y \in \{i,j\}$, and the resulting conditional marginal
can be far from Gaussian.
We prove new
structural lemmas showing that, even in this partial learning setting, sufficiently large pairwise error creates a usable correlation signal.
This yields the first fully polynomial-time algorithm with dimension-independent 
robust error guarantee for agnostically learning MLCs under Gaussian marginals.

\begin{theorem}[Robust Learner for MLC]\label{th perceptron}
Consider the problem of agnostically learning a $k$-MLC under the 
Gaussian distribution. There is an algorithm such that for every $\eps \in (0,1)$, it draws 
$m = d\poly(k/\eps)$ examples from the input distribution, runs in $\poly(m)$ time, and 
outputs a hypothesis $\widehat{h}$ such that with probability at least $99\%$, 
$\err(\widehat{h}) \le \Tilde{O}(k^{3/2}\sqrt{\opt})+\eps$.
\end{theorem}

The $\widetilde{O}(\sqrt{\opt})$ guarantee in \Cref{th perceptron} arises from a limitation of global pairwise learning, where the useful signal scales quadratically with the error. 
Our next algorithmic result achieves linear dependence on $\opt$ when the target MLC satisfies mild structural assumptions.
For $k = 3$, this succeeds unconditionally: near any pairwise decision boundary, a constant fraction of misclassified examples belong to the corresponding two classes, ensuring sufficient signal for improvement. For general $k$, this property need not hold in degenerate configurations. We capture the extent to which it holds through two natural quantities (\Cref{def:boundary}): the \emph{effective boundary mass} $\mathcal{T}_{ij}$, measuring how much probability lies on the $i$--$j$ interface, and the \emph{critical angle} $\theta^*_{ij}$, measuring its separation from other classes. Importantly, in high dimensions these quantities are typically bounded away from zero for well-spread classifiers (e.g., when the weight vectors are in general position), and hence the required geometric condition holds broadly. In this regime, localization amplifies the informative signal near the decision boundary and yields an $O(\opt)$ guarantee. In other words, if we restrict $\opt$ over the class of \emph{regular} MLCs, then our algorithm always achieves error $O(\opt)$.
We summarize these results (see \Cref{th k=3} and \Cref{th surface area} for more detailed formal statements) below.

\begin{theorem}[Error Boosting for Multiclass Linear Classification]\label{th o opt}
There is an algorithm for agnostically learning $k$-MLCs under the Gaussian distribution, which draws $m=d\poly(k/\eps)$ examples, runs in $\poly(m)$ time, and has the following guarantee: 
\begin{enumerate}[leftmargin=*]
    \item If $k=3$, the algorithm, with probability at least $99\%$, outputs a hypothesis with error $O(\opt)+\eps$.
    \item For $k>3$, if the optimal $k$-MLC  $f^*$ has regular geometry 
    (see \Cref{sec boost} for a formal definition), then there is some constant $C_{f^*}$ depending on the regularity parameter, such that the algorithm with probability at least $99\%$ outputs a hypothesis with error $C_{f^*} \tilde O(k)  \opt+\eps$.
\end{enumerate}
\end{theorem}

\paragraph{Preliminaries and Notation.}
For $n \in \mathbb{Z}_+$, let $[n] = \{1,\dots,n\}$. For $x \in \mathbb{R}^d$, let $x_i$ denote its $i$-th coordinate, and let $\|x\| = (\sum_{i=1}^d x_i^2)^{1/2}$ denote the Euclidean norm. For $x,y \in \mathbb{R}^d$, we write $x \cdot y$ for the inner product and $\theta(x,y)$ for the angle between $x$ and $y$. We write $e_i$ for the $i$-th standard basis vector.  We use standard asymptotic notation, and $\widetilde{O}(\cdot)$ hides polylogarithmic factors.

\begin{definition}[Halfspaces and Partial Error]
Let $D$ be a distribution over $\R^d \times [k]$. Let $h(x) = \sign(w \cdot x)$ be a halfspace and $H$ be the class of halfspaces. For $i \neq j \in [k]$, define the $(i,j)$-error of $h$ as
$\err_{ij}(h) := \Pr_{(x,y)\sim D}\big[ w \cdot x < 0,\, y = i \big] + \Pr_{(x,y)\sim D}\big[ w \cdot x > 0,\, y = j \big].$
For a multiclass classifier $f^*$, define
$
\err_{ij}(h, f^*) := \Pr_{(x,y)\sim D}\big[ w \cdot x < 0,\, f^*(x) = i \big] + \Pr_{(x,y)\sim D}\big[ w \cdot x > 0,\, f^*(x) = j \big].
$

\end{definition}

Intuitively, $h$ induces a binary classifier over $\{i,j\}$ by predicting $i$ if $w \cdot x > 0$ and $j$ otherwise, and $\err_{ij}$ measures its misclassification probability on examples whose label lies in $\{i,j\}$.
When it is clear from the context, we will say $h(x) \in \{i,j\}$ to simplify the notation.

\section{Failure of the Multiclass Perceptron under Gaussian Marginals}
\label{sec:lower-bound}

The multiclass perceptron algorithm~\cite{crammer2003ultraconservative,shalev2014understanding} is a natural extension of the binary perceptron to the multiclass setting. It maintains a multiclass linear classifier $h_W(x) = \argmax_{i \in [k]} w_i \cdot x$. Upon receiving an online example $(x,y)$, the algorithm predicts $\widehat{y} = h_W(x)$, and if $\widehat{y} \neq y$, it performs the update
\[
w_y \gets w_y + x, \qquad w_{\widehat y} \gets w_{\widehat y} - x.
\]
It is well-known that the multiclass perceptron converges efficiently when the data satisfies a non-trivial multiclass margin condition. 
In the PAC learning setting, we consider a natural implementation (\Cref{alg:multi-perceptron}) where the weight vectors are initialized independently from a standard Gaussian distribution, and at each time step the algorithm receives an i.i.d.\ sample $(x,y)$ from the underlying distribution. Our first result shows that, even in this benign setting with Gaussian marginals and realizable labels, the algorithm may require super-polynomially many updates to achieve small error.

\begin{algorithm}[h]
\caption{\textsc{Multiclass Perceptron}}\label{alg:multi-perceptron}
		\begin{algorithmic} [1]
\State\textbf{Input:} Sample access to distribution $D$ over $\R^d\times [k]$, sample/iteration number $n$.
\State\textbf{Output:} $h_W(x)=\argmax_{i\in [k]} w_i\cdot x$ multiclass linear classifier
\State Initialize $w_i\sim \mathcal{N}(0,I), i\in [k]$. \label{line:init}
\For{$i \in [n]$}
\State  Draw example $(x,y)\sim D$.
\State Predict $\widehat{y}=h_{W}(x)$.
\If{$\widehat{y}\neq y$}
\State $w_{y}\gets w_{y}+x$, $w_{\widehat{y}}=w_{\widehat{y}}-x$.
\EndIf
\EndFor
\State\Return $h_W$.
\end{algorithmic}
\end{algorithm}

\begin{theorem}[Multiclass Perceptron Lower Bound]\label{thm:lb-main body}
Let $c>0$ be a sufficiently small constant,   $d,k\in \Z_+, k\leq d$, $l\in [k]$ and  $\eps\in (0, 1/l^2]$.
There exists a distribution $D$ over $\R^{d}\times [k]$ with $D_x=\mathcal{N}(0,I_{d})$ and $D_{y\mid x}$ that is realizable by a $k$-MLC, such that for every sample size $n\le 1/\eps^{cl}$,  if the multiclass perceptron is given $n$ i.i.d.\ samples from $D$, then with probability at least $2/3$ (over the sample), its output classifier $h_W$ satisfies
$\E_{(x,y)\sim D}[h_W(x)\neq y]\geq \eps/2^{2l}$.
\end{theorem}
We interpret \Cref{thm:lb-main body} as follows.
For \(l=\Theta(\log k)\) and moderate accuracy \(\eps=1/\poly(k)\), the multiclass perceptron requires \( k^{\Omega(\log k)}\) samples/iterations to reach error $1/k^{O(1)}$. 
On the other hand, for \(l=k\) and exponentially small error \(\eps=2^{-\Theta(k)}\), the theorem implies that \(1/\eps^{\Omega(k)}\) samples/iterations are required to reach error $\eps$, which is exponential in \(k\).
This suggests that, even in the noiseless setting, the standard perceptron algorithm fails to learn a hypothesis with error $\eps$ in time $\poly(dk/\eps)$ under the standard Gaussian distribution.

The underlying issue is that the multiclass perceptron succeeds only when the data exhibits a non-trivial \emph{multiclass margin} (see \Cref{def:margins}). Roughly speaking, this requires that for each class $i$, the corresponding weight vector $w^*_i$ is well-scaled, and for any example $x$ with label $i$, the gap $(w^*_i - w^*_j)\cdot x$ is lower bounded by some $\gamma > 0$ for all $j \neq i$. This imposes structural constraints not only on the data distribution but also on the representation of the target classifier. 
However, even under Gaussian marginals, such a condition can easily fail: one can construct an MLC for which the multiclass margin is arbitrarily small on most examples, leading to extremely slow convergence.

\paragraph{Hard Instance.} \label{par:hard instatnce}
Let $f:\R^{d}\to [k]$ with  
$f(x)\eqdef \min_{i\in [k]}\{i: x_i>0 \}$, where the minimum is equal to $k$ if no such $i$ exists.
Specifically, $f$ is the  limit of $f_M(x)= \argmax_{i\in [k]} M^{k-i} x_i$ as $M$ goes to infinity. 
To see this, consider the decision boundary between classes $i$ and  $j, i<j$ of $f$. 
In particular, if $M^{k-i}x_i-M^{k-j}x_j>0$ class $i$ has a higher score than class $j$ under $f_M$. However  $\lim_{M\to \infty}  \theta(M^{k-i}e_i-M^{k-j}e_j,e_i)=0$, hence the decision boundary approaches $x_i\geq 0$.

\begin{figure}[h]
\centering
\begin{tikzpicture}[scale=0.75, transform shape]

\fill[blue!20] (0,-3) rectangle (3,3);
\fill[red!20] (-3,0) rectangle (0,3);

\draw[->, thick] (-3.2,0) -- (3.3,0) node[right] {$e_1$};
\draw[->, thick] (0,-3.2) -- (0,3.3) node[above] {$e_2$};

\node at (1.5,1.5) {\textbf{Class 1}};
\node at (-1.5,1.5) {\textbf{Class 2}};
\node at (-1.6,-1.5) {\textbf{Classes $3,\ldots,k$}};

\end{tikzpicture}
\caption{In our hard instance, the labels are determined by the first positive coordinate.}
\end{figure}
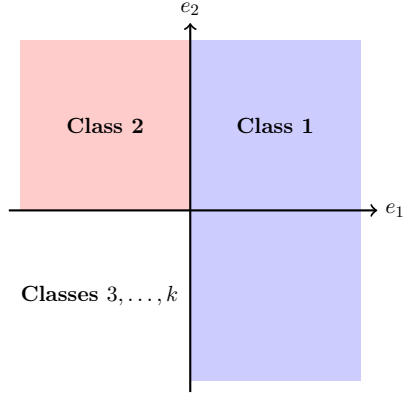

Informally, we show that if an MLC $f_W$ closely approximates $f$, then the norm of at least one of its weight vectors must blow up. The intuition is as follows. The target function $f$ has decision boundaries at $x_i = 0$, separating labels $i$ and $i+1$. If $f_W$ approximates each boundary within angle $\eps$, then the corresponding weight vectors must become increasingly steep: in particular, $\|w_{i+1}\|$ must exceed $\|w_i\|$ by a factor of $\Omega(1/\eps)$. Iterating this argument across all $k$ boundaries implies that some weight vector has norm at least $\eps^{-\Omega(k)}$.
We formalize this intuition in the following lemma.
\begin{lemma}[Good Angle Implies Weight Blow-up]\label{lem:blowup-main}
Let $d,k \in \Z_+,k\leq d$, and for each $i \in [k]$, let $w_i \in \R^{d}$. Let $c>0$ and $0<\eps<1/3$.
If $|w_{k-1,k}-w_{k,k}|\geq c$ and $\theta(w_i-w_j, e_i)<\eps$ for all $i<j\in [k]$,
then there exists $i\in [k]$ such that $\|w_i\|=\Omega( c/(3\eps)^{k-1})$.
\end{lemma}

We sketch the proof of \Cref{thm:lb-main body} below and defer the full proof to \Cref{ap:omitted-lb}.

Notice that the probability mass of class $i$ is $2^{-i}$. Thus, it suffices to focus on the first $l$ classes, effectively reducing the problem from $k$ classes to $l$ classes. 
The proof of \Cref{thm:lb-main body} relies on two key properties of the perceptron updates on this hard instance. First, \Cref{cl: error2anglef} shows that if a classifier $f_W$ achieves sufficiently small error, then each difference vector $w_i - w_j$ must be nearly aligned with $e_i$. In particular, if the condition of \Cref{lem:blowup-main} is maintained throughout the updates, then any such classifier must have at least one weight vector with large norm. 
Second, \Cref{cl:anticoncetration} shows that, starting from a random initialization, the condition of \Cref{lem:blowup-main} holds throughout the execution with constant probability. Since each update has a bounded length, with high probability, this gives the proof of \Cref{thm:lb-main body}.

To prove the first property, we note that for a general MLC, a small classification error does not imply that $w_i - w_j$ is nearly aligned with $w^*_i - w^*_j$. However, for our hard instance, every pair of classes forms a local neighborhood with a simple structure. Specifically, for each $i < j$, there exists a region $B_{ij}$ with probability mass at least $2^{-O(j)}$ such that all examples in $B_{ij}$ have labels in $\{i,j\}$, and the label is determined solely by the sign of $x_i$. 
Let $u := w_i - w_j$ and suppose it makes an angle $\theta$ with $e_i$. Then, conditioned on $B_{ij}$, one can show that an $\Omega(2^{-O(j)} \theta)$ fraction of examples are misclassified by $f_W$. This establishes the desired property.

The second property is established via a stochastic analysis of the perceptron updates using Gaussian anti-concentration. 
With random initialization, the condition of \Cref{lem:blowup-main} holds with constant probability. Conditioned on this event, a failure at the next step occurs only if the update example $x$ has its $l$-th coordinate falling in a small interval, which happens with low probability by anti-concentration.
A key challenge is that this argument cannot be applied independently across all time steps. To address this, we carefully partition the failure event in a way that allows us to leverage fresh Gaussian randomness, thereby controlling the overall failure probability.

\section{Pairwise Partial Learning Framework: Proof of \Cref{th perceptron}}\label{sec upper bound technique}

In this section, we develop techniques for designing efficient robust learners for MLCs under the Gaussian distribution and present the algorithm establishing \Cref{th perceptron}.

\paragraph{Pseudo Multiclass Classifier and Error Decomposition.}

To bypass the difficulty of \Cref{thm:lb-main body}, we instead consider improper learning.
In this paper, we will use a pseudo  $k$-MLC as the output hypothesis, defined as follows.

\begin{definition}[Pseudo $k$-Multiclass Linear Classifier]
    Let $k \in \Z_+$ and $W = \{w_{ij} \in \s^{d-1} \mid i,j \in [k], i< j\}$ be a set of $\binom{k}{2}$ unit vectors. A pseudo $k$-multiclass linear classifier $h_W: \R^d \to [k]$ is defined as $h_W(x) = \argmax_{i \in [k]} s_i(x)$, where for $i \in [k], x \in \R^d$, $s_i(x):= \sum_{j \neq i}\Ind\{w_{ij} \cdot x \ge 0\}$. 
\end{definition}

\begin{algorithm}[htbp]
		\caption{\textsc{Multiclass Learning} (Learn a pseudo $k$-multiclass linear classifier)}\label{alg pairwise to pseudo}
		\begin{algorithmic} [1]
\State\textbf{Input:} Sample access to distribution $D$, error parameter $\eps \in (0,1)$, confidence parameter $\delta\in (0,1)$, access to partial halfspace learning algorithm \textsc{Pairwise Learning}
\State\textbf{Output:} Pseudo $k$-multiclass linear classifier $h_{\widehat W}$

\For{$i,j \in [k], i< j$}
\State $\widehat{w}_{ij} \gets \textsc{Pairwise Learning}(i,j,\eps/k^2,\delta/k^2)$ \Comment{Partially learn a halfspace for label $(i,j)$}
\EndFor
\State Let $\widehat{W} = \{\widehat{w}_{ij} \mid i,j \in [k], i<j\}$ \Comment{Combine the partially learned halfspaces to a pseudo MLC}
\State\Return $h_{\widehat{W}}$
\end{algorithmic}
\end{algorithm}

Intuitively, a pseudo \(k\)-MLC assigns a linear separator to each pair of labels. For every \((i,j)\), the vector \(w_{ij}\) determines whether \(x\) favors class \(i\) over \(j\) via the sign of \(w_{ij}\!\cdot\! x\). Each class receives one ``vote'' for every pairwise comparison it wins, and the classifier outputs the class with the most votes. In this sense, the hypothesis class implements a one-vs-one tournament, aggregating all pairwise linear decisions rather than relying on a single weight matrix.
To evaluate a pseudo MLC under noise, we present the following error decomposition; 
the proof is deferred to \Cref{app proof error decomposition}. 
\begin{lemma}[Error Decomposition]\label{lm error decomposition}
    Let $D$ be a distribution over $\R^d \times [k]$. Let $h_W$ be a pseudo $k$-MLC and 
    let $\eps \in (0,1)$. Suppose that there exist constants $C>1$ and $c \in [0,1/2]$ such that
    for every $i \neq j \in [k]$, $\err_{ij}(h_{ij}) \le C ( {\err^*_{ij}}^{1-c}\log^{c/2}(1/\err^*_{ij}))+\eps$, where $h_{ij} = \sign(w_{ij}\cdot x)$ and $\err^*_{ij} = \inf_{h \in H}(\err_{ij}(h))$. 
    Then, 
    $\err(h_W)  \le C k^{1+c}\opt^{1-c}\log^{c/2}(k/\opt)
 +k^2\eps$.
\end{lemma}

Informally, the lemma shows that if the classifier is near-optimal for \emph{partially} learning every label pair, then the resulting pseudo multiclass classifier is also near-optimal overall. The degradation in the bound arises from aggregating pairwise errors: a single noisy example can affect many binary subproblems—namely all comparisons of the form \((i,r)\) for \(r \neq i\)—and thus may be counted up to \(k\) times, leading to the amplification in the final error.
By \Cref{lm error decomposition}, it suffices to partially learn halfspaces for each label pair and combine them into a pseudo MLC as in \Cref{alg pairwise to pseudo}.

\paragraph{Learning Pairwise Separators from Partial Labels.}
Unlike standard halfspace learning, where all examples are informative, distinguishing labels $i$ and $j$ relies only on examples whose labels lie in $\{i,j\}$. This restriction makes the \emph{partial} learning problem---learning a separator when the labels of some examples are hidden---significantly more challenging, especially under adversarial label noise.
In binary classification with Gaussian marginals, a central approach to robust learning is Chow-parameter matching: in the noiseless setting, the vector $\E[yx]$ aligns with the target normal $w^*$, enabling stable estimation under corruption. In our setting, however, this structure breaks down. Informative examples arise from two disjoint regions corresponding to labels $i$ and $j$, potentially mixed with outliers. As a result, the induced distribution is far from being a Gaussian,
and their geometry can be highly complex. In particular, there is no simple global statistic analogous to Chow parameters that captures the underlying separator.
This lack of structure poses a key obstacle for analysis. To overcome it, we establish a new structural result for partial learning, which may be of independent interest and applies more broadly. The proof is deferred to \Cref{app proof convex body}.

\begin{lemma}[Correlation Bound under the Gaussian Distribution]\label{lm convex body}
Let $w \in \s^{d-1}$ be any unit vector and let $S$ be any measurable event such that
$
S \subseteq \{x \in \R^d : w \cdot x \ge 0\}.
$
Then
\[\E_{x \sim \cN(0,I)}[(w \cdot x)\Ind\{x \in S\}]
\ge
\sqrt{\pi/2}\,\Pr_{x \sim \cN(0,I)}[x \in S]^2. \]  
\end{lemma}

In the absence of label noise, if the current halfspace incurs large $(i,j)$-error, then the expectation over misclassified examples yields a vector $v$ satisfying
$
w^*_{ij} \cdot v \ge \Omega\big(\err_{ij}(h_{w_{ij}})^2\big),
$
where $w^*_{ij}$ is the normal vector of the optimal decision boundary between classes $i$ and $j$. 
Thus, although $v$ need not align with $w^*_{ij}$, it still has a positive correlation with $w^*_{ij}$ as long as the noise level on $(i,j)$ does not dominate $\err_{ij}(h_{w_{ij}})^2$. This correlation provides a useful descent direction. We therefore use $v$ to perform a projected gradient update, reducing $\norm{w_{ij} - w^*_{ij}}$. Formally, we rely on the following standard property of projected gradient descent.

\begin{lemma}\label{lm gradient descent}
    Let $w^*,w_i \in \s^{d-1}$ and let $G_i$ be a random vector drawn from some distribution $D$ such that with probability $1$, $G_i \perp w_i$. Let $g_i$ be the mean of $G_i$. Let $\widehat{g_i}$ be the empirical mean of $G_i$ and $\mu_i>0$. The update rule $w_{i+1} = \proj_{\s^{d-1}} (w_{i} + \mu_i\widehat{g_i})$ satisfies the following property:
    \begin{align*}
        \norm{w_{i+1}-w^*}^2 \le \norm{w_{i}-w^*}^2 - 2\mu_i g_i \cdot w^* + 2\mu_i\norm {g_i-\widehat{g_i}} +\mu_i^2\norm{\widehat{g_i}}^2.
    \end{align*}

\end{lemma}

\begin{algorithm}[htbp]
\caption{\textsc{Pairwise Initialization} (Initialize a halfspace to distinguish class $i$ from class $j$)}\label{alg pairwise}
\begin{algorithmic}[1]
\State \textbf{Input:} sample access to $D$, labels $i,j$, error $\eps \in (0,1)$, confidence $\delta \in (0,1)$
\State \textbf{Output:} $w_{ij} \in \s^{d-1}$
\State $N \gets d\log(1/\delta)\poly(1/\eps)$,\quad $T \gets \poly(1/\eps)$
\State Initialize $w_{ij}^{(0)} \sim \s^{d-1}$ \Comment{repeat $\log(1/\delta)$ times if needed}
\For{$t=0,\dots,T-1$}
    \State Draw $\{(x^{(r)},y^{(r)})\}_{r=1}^N \sim D$
    \State $\widehat v \gets \frac{1}{N}\sum_{r=1}^N x^{(r)}\!\left(\Ind\{y^{(r)}=i,\, w_{ij}^{(t)}\!\cdot x^{(r)}\le 0\}-\Ind\{y^{(r)}=j,\, w_{ij}^{(t)}\!\cdot x^{(r)}\ge 0\}\right)$
    \State $w_{ij}^{(t+1)} \gets \proj_{\s^{d-1}}\!\left(w_{ij}^{(t)} + \mu\,\proj_{(w_{ij}^{(t)})^\perp}\widehat v\right)$,\quad $\mu=\eps^2/C$
\EndFor
\State Draw fresh $\{(x^{(r)},y^{(r)})\}_{r=1}^N \sim D$
\State \Return $\displaystyle w_{ij}^{(t^*)}$, where 
$t^*=\argmin_{t\in[T]} \sum_{r=1}^N \Ind\{\sign(w_{ij}^{(t)}\!\cdot x^{(r)})\neq y^{(r)},\, y^{(r)}\in\{i,j\}\}$
\end{algorithmic}
\end{algorithm}

To apply \Cref{lm convex body,lm gradient descent}, we rely on two additional facts. First, under an $\eta$-fraction of adversarial label noise, the correlation $\E[\widehat{v}\cdot w^*_{ij}]$ is perturbed by at most $O(\eta\sqrt{\log(1/\eta)})$ (see \Cref{fact:simpleperturbation}). Hence, as long as the current error is $\tilde\Omega(\sqrt{\eta})$, the update yields progress. Second, $\norm{\E[\widehat{v}]} = O(1)$ (see \Cref{lm gaussian event mean}), which ensures a favorable convergence rate.
We summarize the resulting pairwise-distinguishing algorithm in \Cref{alg pairwise} and its guarantee in \Cref{lm pairwise}; the proof of \Cref{lm pairwise} is deferred to \Cref{app proof pairwise}.

\begin{lemma}\label{lm pairwise}
    Consider the problem of agnostically learning a $k$-MLC under the Gaussian distribution. Given any $i \neq j \in [k]$ and $\eps \in (0,1),\delta \in (0,1)$, \Cref{alg pairwise} draws $d\log(1/\delta)\poly(1/\eps)$ examples from $D$, runs in $\poly(m)$ time, and outputs a halfspace $\widehat{h}_{ij}$ such that with probability at least $1-\delta$, $\err_{ij}(\widehat{h}_{ij}) \le \tilde O (\sqrt{\err^*_{ij}}) + \eps$, where $\err^*_{ij} = \inf_{h \in H}(\err_{ij}(h))$.
\end{lemma}
Therefore, the proof of \Cref{th perceptron} follows by
using \Cref{alg pairwise} as a subroutine in \Cref{alg pairwise to pseudo}.

\section{Error Boosting via Localization}\label{sec boost}

\Cref{th perceptron} yields an $\tilde O(\sqrt{\opt})$ guarantee via global partial learning. We improve this to $O(\opt)$ by localizing near the current decision boundary, where the signal from misclassified examples is stronger.
In the multiclass setting, such localization can fail in degenerate configurations. We address this by introducing geometric quantities—the effective boundary mass and critical angle—that ensure sufficient signal near the boundary.

\paragraph{Warm-up: Error Boosting for Ternary Classification.}  
As a warm-up, we give a polynomial-time algorithm for agnostic ternary classification $(k=3)$ with $O(\opt)$ error.
\begin{theorem}\label{th k=3}
Consider the problem of agnostic learning MLC over $\R^d$ with 3 labels. There is an algorithm that draws $m=d\poly(\log(1/\delta)/\eps)$ samples, runs in $\poly(m)$ time, and outputs a hypothesis with error $O(\opt) +\eps$ with probability at least $1-\delta$.    
\end{theorem}
The proof of \Cref{th k=3} is deferred to \Cref{app proof k=3}.
The key to achieving this improvement is an improved partial learning algorithm, \Cref{alg pairwise localization}.

\begin{lemma}\label{lm localization}
Consider the problem of agnostic learning a $3$-MLC under the Gaussian distribution. There is an algorithm such that
given any $i \neq j$ and $\eps,\delta \in (0,1)$, it draws $m=d\log(1/\delta)\poly(1/\eps)$ examples from $D$, runs in $\poly(m)$ time, and outputs a halfspace $\widehat{h}_{ij}$ such that $\err_{ij}(\widehat{h}_{ij}) \le O(\mathcal{O}_{ij})+\eps$, with probability $1-\delta$, where 
$\mathcal{O}_{ij} = \sum_{a\neq b:a \in \{i,j\}\; \mathrm{or }\; b \in \{i,j\}}\opt_{a,b}$, $\opt_{a,b} = \Pr_{(x,y) \in D}[f^*(x) = a, y=b]$.  
\end{lemma}

We give an overview of \Cref{lm localization} and defer the full proof to \Cref{app proof localization}. Compared with \Cref{lm pairwise}, \Cref{lm localization} provides an error guarantee with respect to a different benchmark, namely $\mathcal{O}_{ij}$, which measures the corruption of the optimal MLC $f^*$ on labels $i$ and $j$. This weaker benchmark is still sufficient to obtain a pseudo-MLC with total error $O(k\opt)$. The advantage of this formulation is that it allows us to exploit the additional structural properties of multiclass linear classifiers.

While this is not immediate, in the case $k=3$ the key structural insight is the following: if the current halfspace $h_w$ incurs a significant $(i,j)$-error on clean examples, then after conditioning on a sufficiently small band around $w$, a constant fraction of the retained examples remain misclassified.
For a cleaner analysis, we implement this localization via a soft truncation. Specifically, we consider $x \sim \cN(0,\Sigma)$, where $\Sigma^{1/2}w=\sigma w$ for some $\sigma\in(0,1)$ and $\Sigma^{1/2}u=u$ for all $u \perp w$. This distribution can be realized by rejection sampling from $\cN(0,I)$ with acceptance probability depending on $w\cdot x$ (see \Cref{lem:rejection-sampling}).
Writing $x=\Sigma^{1/2}z$ with $z\sim\cN(0,I)$, and defining $y^*(z)=f^*(x)$, we may equivalently view $z$ as labeled by a transformed MLC with boundary $\tilde w^*_{ij}=\Sigma^{1/2}w^*_{ij}$. Let
\[
E^*_{ij}(w)=\{(z,y^*) : y^*=i,\, w\cdot z\le 0\}\ \cup\ \{(z,y^*) : y^*=j,\, w\cdot z\ge 0\}.
\]
We claim that if $\err_{ij}(h_w)>\Omega(\mathcal{O}_{ij})+\eps$, then for $\sigma=\Theta(\mathcal{O}_{ij}+\eps)$, $\Pr_{z\sim\cN(0,I)}[E^*_{ij}(w)] > 1/8.$
The argument proceeds in two steps. First, large error implies that $\theta=\theta(w,w^*_{ij})=\Omega(\mathcal{O}_{ij}+\eps)$. Writing $w^*_{ij}=\cos\theta\, w+\sin\theta\, u$ with $u\perp w$, the choice of $\sigma$ ensures that $\tilde\theta=\theta(w,\tilde w^*_{ij})>\pi/4$, and hence $\Pr[(w\cdot z)(\tilde w^*_{ij}\cdot z)<0] \ge 1/4.$
Second, not all such disagreements correspond to labels in $\{i,j\}$ due to competing classes. However, by symmetry, a constant fraction of these points satisfy the remaining constraints (e.g., $\tilde w^*_{ki}\cdot z\le 0$), and thus belong to labels $\{i,j\}$. 
For these points, if $y^*=i$ and $w\cdot z\le 0$, then $u\cdot z\ge 0$, and by \Cref{lm convex body} we can estimate a vector with constant correlation with $u$, which yields progress via \Cref{lm gradient descent}. 
Finally, since we only use examples with labels in $\{i,j\}$, the effective corruption level is at most $\mathcal{O}_{ij}$. By the choice of $\sigma$, the fraction of corrupted samples in the localized region is negligible compared to the signal, ensuring robustness.
This concludes the proof overview of \Cref{lm localization}.

\paragraph{Boosting through Effective Decision Boundary.}
Having established error boosting for $k=3$, a natural question is whether localization remains effective for larger $k$. In general, the answer is no: localization can fail in degenerate configurations. For example, suppose $f^*$ is such that labels $i$ and $j$ depend only on a low-dimensional subspace $V$, while the current direction $w$ satisfies $w \perp V$. Then, regardless of the chosen band around $w$, the fraction of informative examples is at most $\Pr_{x\sim\cN(0,I)}[f^*(x)\in\{i,j\}]$, which may be arbitrarily small, leaving insufficient signal for improvement.
Our key observation is that such configurations are atypical in high dimensions. For a broad class of well-behaved multiclass linear classifiers, we show that a refined localization scheme still yields sufficient signal to boost the error from $\sqrt{\opt}$ to $O(\opt)$, with the constant depending on a suitable regularity parameter of $f^*$.
We formalize this through geometric quantities that capture the regularity of an MLC.

\begin{definition}[Effective Decision Boundary and Critical Angle]\label{def:boundary}
    Let $f^*(x) = \argmax_{i \in [k]} (w^*_i \cdot x)$ be a $k$-MLC. For $i,j \in [k], i \neq j$, we define the effective decision boundary of $f^*$ for classes $i,j$ as 
    \begin{align*}
     B_{ij}(f^*):=\{x \mid w^*_{ij}\cdot x = 0, w^*_{ir}\cdot x \ge 0, w^*_{jr}\cdot x \ge 0, \forall r \in [k] \setminus \{i,j\}\},   
    \end{align*}
    where for $i\neq j \in [k]$, $w^*_{ij} = (w^*_i-w^*_j)/\|w^*_i-w^*_j\|$.
The Gaussian measure of $B_{ij}$ is defined as $\mathcal{T}_{ij} \eqdef \Pr_{x \sim \cN(0,I)}[B_{ij}]$, where $x$ is drawn from a $(d-1)$-dimensional standard Gaussian supported over $H_{ij}:=\{x \mid w^*_{ij}\cdot x = 0\}$. Define the critical angle $\tan \theta^*_{ij} = \min_{r \in [k]\setminus\{i,j\}}\{\min\{|\tan\theta(w^*_{ij},w^*_{ir})| ,|\tan\theta(w^*_{ji},w^*_{jr})|\}\}$ and $\Phi_{ij} = \min\{\tan \theta^*_{ij},1\}.$
\end{definition}

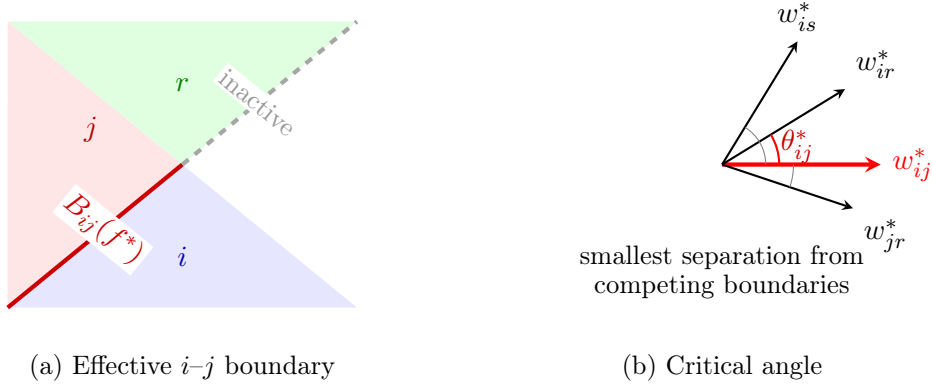
\begin{figure}[htbp]
\centering
\begin{tikzpicture}[scale=1.05, >=stealth]

\begin{scope}[shift={(-3.4,0)}]

\coordinate (O) at (0,0);
\coordinate (A) at (-2.2,-1.8);  
\coordinate (B) at (2.2,-1.8);
\coordinate (C) at (0,2.1);

\fill[blue!10]   (A) -- (O) -- (B) -- cycle;
\fill[red!10]    (-2.2,1.8) -- (O) -- (A) -- cycle;
\fill[green!12]  (2.2,1.8) -- (O) -- (-2.2,1.8) -- cycle;

\draw[ultra thick, red!80!black] (O) -- (A);
\draw[ultra thick, dashed, gray!70] (O) -- (2.2,1.8);

\node[blue!70!black] at (0,-1.15) {$i$};
\node[red!70!black] at (-1.15,0.45) {$j$};
\node[green!50!black] at (0,1) {$r$};

\node[
  red!80!black,
  rotate=-39,
  fill=white,
  inner sep=1.2pt
] at (-1.0,-0.82) {$B_{ij}(f^*)$};

\node[
  gray!80,
  rotate=-39,
  fill=white,
  inner sep=1.2pt
] at (0.95,0.78) {\small inactive};

\node at (0,-2.55) {\small (a) Effective $i$--$j$ boundary};

\end{scope}

\begin{scope}[shift={(3.4,0)}]

\coordinate (O) at (0,0);

\draw[->, ultra thick, red] (O) -- (2.0,0) node[right] {$w^*_{ij}$};
\draw[->, thick] (O) -- (1.55,0.95) node[above right] {$w^*_{ir}$};
\draw[->, thick] (O) -- (0.95,1.55) node[above] {$w^*_{is}$};
\draw[->, thick] (O) -- (1.65,-0.55) node[below right] {$w^*_{jr}$};

\draw[red, thick] (0.72,0) arc[start angle=0,end angle=31.5,radius=0.72];
\node[red!80!black] at (0.95,0.23) {$\theta^*_{ij}$};

\draw[gray] (0.55,0) arc[start angle=0,end angle=58,radius=0.55];
\draw[gray] (0.9,0) arc[start angle=0,end angle=-18.5,radius=0.9];

\node[align=center] at (0,-1.35) {\small smallest separation from\\[-1mm]\small competing boundaries};

\node at (0,-2.55) {\small (b) Critical angle};

\end{scope}

\end{tikzpicture}
\caption{
Illustration of the geometric quantities used for localization.
{\bf Left:} the effective decision boundary $B_{ij}(f^*)$ is the active portion of the $i$--$j$ boundary where classes $i$ and $j$ are adjacent; the dashed continuation is inactive because another class dominates there.
{\bf Right:} the critical angle $\theta^*_{ij}$ measures the minimum angular separation between the $i$--$j$ boundary normal and competing boundary normals; the smaller the critical angle is, the more sensitive label changes near the decision boundary.
}
\label{fig:effective-boundary-critical-angle}
\end{figure}

The above definition formalizes the portion of the decision boundary between classes $i$ and $j$ that is actually relevant under the multiclass rule. While the hyperplane $\{x \mid w^*_{ij} \cdot x = 0\}$ describes the pairwise separator between $i$ and $j$, not every point on this hyperplane corresponds to an ambiguity between these two labels in the full multiclass classifier. Indeed, at a point $x$ on this hyperplane, another class $r \notin \{i,j\}$ may have a larger score and thus determine the label.

To capture only the meaningful part of the boundary, we define the effective decision boundary $B_{ij}(f^*)$ as the subset of the hyperplane where both $i$ and $j$ dominate all other classes. In other words, $B_{ij}(f^*)$ consists of points where $i$ and $j$ tie, and no third class interferes. This set represents the true interface across which the classifier switches between labels $i$ and $j$ (see \Cref{fig:effective-boundary-critical-angle} (a)). 

We quantify the “size” of this interface via its Gaussian measure $\mathcal{T}_{ij}$, taken with respect to the $(d-1)$-dimensional standard Gaussian distribution supported on the hyperplane $H_{ij} =\{x \mid w^*_{ij} \cdot x = 0\}$. Intuitively, $\mathcal{T}_{ij}$ captures how prominent the $i$–$j$ boundary is under the data distribution.

Finally, the critical angle $\theta^*_{ij}$ measures how well-separated this boundary is from the decision boundaries involving other classes. It is defined as the minimum angle between $w^*_{ij}$ and any competing direction $w^*_{ir}$ or $w^*_{jr}$ for $r \not\in \{i,j\}$. A larger critical angle indicates that the $i$–$j$ boundary is geometrically well-isolated, which in turn ensures that the effective boundary $B_{ij}(f^*)$ is stable and less prone to interference from other classes (see \Cref{fig:effective-boundary-critical-angle} (b)).

In high dimensions, these quantities are unlikely to be too small for well-spread classifiers, unless the weight vectors are arranged in a highly degenerate way (e.g., nearly collinear or confined to a low-dimensional subspace). In particular, we show that the quantities $\mathcal T_{i,j}$ and $\theta_{i,j}^*$ are at least $1/k^{O(1)}$ for all pairs $i,j\in[k]$, with high probability for an MLC with randomly sampled weight vectors. 
We refer the reader to \Cref{sec:randommlcs} for more details.

Based on these two definitions, we have the following theorem.
\begin{theorem}\label{th surface area}
    Consider the problem of agnostically learning $k$-MLC over $\R^d$, there is an algorithm that draws $n=d\poly(k/\eps)\log(1/\delta)$ samples, runs in $\poly(n)$ time, and with probability at least $1-\delta$ outputs a hypothesis with error $k\log k\poly(\mathcal{T}^{-1}\Phi^{-1})\opt +\eps$, where $\mathcal{T} = \min_{(i,j):i \neq j \in [k]}\mathcal{T}_{ij}$ and $\Phi = \min_{(i,j):i \neq j \in [k]}\Phi_{ij}$, with $\Phi_{ij}:=\min\{|\tan \theta^*_{ij}|,1\}$.
\end{theorem}
We defer the proof of \Cref{th surface area} to \Cref{app proof surface area}.
The main algorithmic ingredient is \Cref{alg pairwise localization surface} with the guarantee stated in \Cref{lm surface localization}. We defer the proof of \Cref{lm surface localization} to \Cref{app proof surface localization}.

\begin{lemma}\label{lm surface localization}
Consider the problem of agnostic learning a $k$-multiclass linear classifier under the standard Gaussian distribution. Given any $i \neq j \in [k]$ and $\eps,\delta \in (0,1)$, \Cref{alg pairwise localization surface} draws $m=d\poly(k/\eps,\log(1/\delta))$ examples from $D$, runs in $\poly(m)$ time and with probability at least $1-\delta$ outputs a halfspace $\widehat{h}_{ij}$ such that $\err_{ij}(\widehat{h}_{ij}) \le \poly(\mathcal{T}_{ij}^{-1}\Phi_{ij}^{-1})\log k \mathcal{O}_{ij}+\eps$, where 
$\mathcal{O}_{ij} = \sum_{a\neq b:a \in \{i,j\}\; \mathrm{or }\; b \in \{i,j\}}\opt_{a,b}$, $\opt_{a,b} = \Pr_{(x,y) \in D}[f^*(x) = a, y=b]$.  
\end{lemma}

\begin{algorithm}[htbp]
\caption{\textsc{Pairwise Localization} (Error Boosting through Localization)}\label{alg pairwise localization surface}
\begin{algorithmic}[1]
\State \textbf{Input:} sample access to $D$ over $\R^d\times [k]$, labels $i,j$, error $\eps \in (0,1)$, confidence $\delta \in (0,1)$
\State \textbf{Output:} $w_{ij} \in \s^{d-1}$
\State Set $N \gets d(1/\eps)^C\log(1/\delta)$, $T \gets (1/\eps)^C$, $\eta \gets \eps^C$
\State Initialize $w_{ij}^{(0)}$ using \Cref{alg pairwise} with error parameter $\eps/10$. 
\State $\phi_0 \gets \tilde \Theta(\mathcal{T}_{ij}^2 \Phi_{ij}/\sqrt{\log k})$ \Comment{Parameters $\mathcal{T}_{ij},\Phi_{ij}$ can be obtained via gridding.}
\For{$t=0,\dots,T-1$}
    \State $\sigma_t \gets \min\{3\sin \phi_t/\sqrt{\mathcal{T}_{ij}},\,1\}$
    \State Define acceptance probability $p_t(x)=\exp\!\left(-(\sigma_t^{-2}-1)(w_{ij}^{(t)}\cdot x)^2/2\right)$
    \State Let $D'(w_{ij}^{(t)},\sigma_t)$ be the distribution induced by accepting $x\sim D_x$ with probability $p_t(x)$
    \State Draw $\{(x^{(r)},y^{(r)})\}_{r=1}^N \sim D'(w_{ij}^{(t)},\sigma_t)$ and compute
    \[
    \widehat v_t = \frac{1}{N}\sum_{r=1}^N x^{(r)}\!\left(\Ind\{y^{(r)}=i,\, w_{ij}^{(t)}\cdot x^{(r)} \le 0\}
    - \Ind\{y^{(r)}=j,\, w_{ij}^{(t)}\cdot x^{(r)} \ge 0\}\right)
    \]
    \State $w_{ij}^{(t+1)} \gets \proj_{\s^{d-1}}\!\left( w_{ij}^{(t)} + \mu \proj_{(w_{ij}^{(t)})^\perp} \widehat v_t \right)$,\quad $\mu=\eps^3/C$, \quad $\phi_{t+1} \gets \phi_t - \eta$
\EndFor
\State Draw $\{(x^{(r)},y^{(r)})\}_{r=1}^N \sim D$
\State \Return $w_{ij}^{(t^*)}$ where 
$t^* = \argmin_{t\in[T]} \sum_{r=1}^N \Ind\{\sign(w_{ij}^{(t)}\cdot x^{(r)}) \neq y^{(r)},\, y^{(r)}\in\{i,j\}\}$
\end{algorithmic}
\end{algorithm}

To prove \Cref{lm surface localization}, we establish several novel structural results for MLC.
Our first structural result relates the error of the current hypothesis to the geometry of the decision boundary, captured by the effective boundary mass and critical angle. 
In contrast to binary linear classification---where the error is directly proportional 
to the angle---no such simple geometric characterization is available in the multiclass setting, which poses a key obstacle to algorithm design.

\begin{lemma}\label{col surface}
     Let $f^*(x) = \argmax_{i \in [k]} (w^*_i \cdot x)$ be a $k$-MLC on $\R^d$ 
     and let $h_w(x) = \sign(w \cdot x)$. For $i,j \in [k], i \neq j$, let $\theta = \theta(w^*_{ij},w)$. 
    Then, it holds
    \begin{align*}
        & \Pr_{x\sim \cN(0,I)}[(w^*_{ij}\cdot x)(w\cdot x)<0, f^*(x) \in \{i,j\}] \le 2\sqrt{e} \tan \theta \mathcal{T}_{ij}\sqrt{\log(1/\mathcal{T}_{ij})}+\tilde O(\sqrt{\log k}\tan^2\theta/\Phi_{ij}) \\
        & \Pr_{x\sim \cN(0,I)}[(w^*_{ij}\cdot x)(w\cdot x)<0, f^*(x) \in \{i,j\}] \ge \tan \theta \,  \mathcal{T}_{ij}^2/20 - \tilde O(\sqrt{\log k}\tan^2\theta/\Phi_{ij}) \;.
    \end{align*}
In particular, if $\theta \ge \theta_0 = \tilde\Theta  (\mathcal{T}^2_{ij}\Phi_{ij}/\sqrt{\log k})$,
then it always holds 
\[\Pr_{x\sim \cN(0,I)}[(w^*_{ij}\cdot x)(w\cdot x)<0, f^*(x) \in \{i,j\}] \ge \tilde \Omega(\mathcal{T}_{ij}^4\Phi_{ij}/\sqrt{\log k}). \]
\end{lemma}

\Cref{col surface} provides a first-order approximation of the error in terms of these geometric quantities, and may be of independent interest. The proof of \Cref{col surface} is technical, and we give its overview below. We refer to \Cref{app proof surface} for a complete proof.

The key technical lemma for proving \Cref{col surface} is \Cref{lm error count}, whose proof is deferred to \Cref{app proof error count}.

\begin{lemma}\label{lm error count}
    Let $f^*(x) = \argmax_{i \in [k]} (w^*_i \cdot x) : \R^d \to [k]$ be a multiclass linear classifier and let $h_w(x) = \sign(w \cdot x)$ be a halfspace. For $i,j \in [k], i \neq j$, let $\theta = \theta(w^*_{ij},w)$. 
    Write $w = a w^*_{ij} + b u$, with $a,b \in (0,1)$, $a^2+b^2 = 1$, $u \perp w^*_{ij}$ and $u \in \s^{d-1}$.
    Then, it holds
    \begin{align*}
        \abs{\err_{i,j}(h_w,f^*)-\int_{z \in B_{ij}} \phi_{d-1}(z) d z\int_0^ {\tan\theta \abs{u\cdot z} }\phi_1(t) dt }  \le \tilde O(\sqrt{\log k}\tan^2\theta/\Phi_{ij}).
    \end{align*}
Furthermore, for every $\theta_0<\theta$, it holds 
\begin{align*}
       \err_{i,j}(h_w,f^*)\ge \int_{z \in B_{ij}} \phi_{d-1}(z) d z\int_0^ {\tan\theta_0 \abs{u\cdot z} }\phi_1(t) dt   - \tilde O(\sqrt{\log k}\tan^2\theta_0/\Phi_{ij}),
\end{align*}
where $\Phi_{ij} = \min\left\{ |\tan \theta^*_{ij}|,1 \right\}$.
\end{lemma}

We first provide an overview of \Cref{lm error count}. Since $x\sim \cN(0,I)$, we write $x = z+tw^*_{ij}$, where $z \in H_{ij}$ is a $(d-1)$-dimensional gaussian, while $t \sim \cN(0,1)$. For any fixed $z \in H_{ij}$, we have $\sign(w_{ij}^*\cdot x) = \sign(t)$, while $\sign(w\cdot x) = \sign(t + \tan\theta (u\cdot z))$. 
This gives an exact expression of the error as 
\begin{align*}
    \int_{z \in H_{ij}} \phi_{d-1}(z) d z\int_0^ {\tan\theta \abs{u\cdot z} } \Ind\{f^*(x) \in \{i,j\}\} \phi_1(t) dt
\end{align*}
We decompose the mistake region into two contributions: those from points $x$ whose projection $z$ lies in $B_{ij}$, corresponding to perturbations of points whose label is $i$ or $j$, and those from points whose projection lies in $H_{ij}\setminus B_{ij}$, where the label at the projected point is not $i$ or $j$.
First, we claim that the second term is of a lower order term in  $\tan\theta$. To see this, for every $w^*_{ir}$, we decompose it along $w^*_{ij}$ as $a_rw^*_{ij} + b_r u_r$, where $u_r \perp w^*_{ij}$, $u_r \in \s^{d-1}$ and $a_r^2 + b_r^2 = 1$.  
Then, for any fixed $z\in H_{ij}\setminus B_{ij}$, we have
$w^*_{ir}\cdot x=b_ru_r\cdot z+a_rt$.  This implies that if $\abs{u_r\cdot z} = s$, then only when $\abs{t}> \tan \theta^*_{ij} s$, $f^*(x)$ can be in $\{i,j\}$. As a result, all examples that are $\abs{u\cdot z} \tan\theta/\tan \theta^*_{ij}$ far away from any facet of the boundary of $B_{ij}$ have no error contribution. Furthermore, as $B_{ij}$ is an intersection of $2k$ halfspaces, the Gaussian surface area of $B_{ij}$ is at most $O(\sqrt{\log k})$. This implies that examples near $\partial B_{ij}$ would have an error contribution at most $O(\sqrt{\log k}\tan^2\theta/\tan\theta^*_{ij})$.
The first contribution is analyzed similarly. For those examples in $B_{ij}$ but far from the decision boundary, the error contribution is exactly the same as the approximation we use in \Cref{lm error count}, while examples that are close to the decision boundary only have a tiny contribution to the error.
In fact, by the anti-concentration of a Gaussian distribution, conditioned on $B_{ij}$, at least half of the $z$ will have $\abs{u\cdot z} \ge \Omega( \mathcal{T}_{ij})$, which gives \Cref{col surface}.

However, even with \Cref{col surface}, it is not immediate that a significant fraction of misclassified examples with labels in $\{i,j\}$ will appear near the decision boundary. Two issues arise. 
First, due to the second-order term, \Cref{col surface} is only informative when the angle $\theta$ is sufficiently small, with the threshold depending on $f^*$. How can we obtain an initialization in this regime? 
Second, localization alters the geometry of the decision boundary. How should the parameter $\sigma$ be chosen so that the resulting localized distribution preserves sufficient signal?
We next address these two challenges.

First, we show that the output $w_{ij}$ of \Cref{alg pairwise} already lies in the regime where \Cref{col surface} is effective. In particular, whenever $\poly(\mathcal{T}_{ij}\Phi_{ij})^{-1}\mathcal{O}_{ij} < \sqrt{\mathcal{O}_{ij}}$, it holds that $\theta(w_{ij}, w^*_{ij}) \le \tilde \Theta\!\left(\frac{\mathcal{T}_{ij}^2 \Phi_{ij}}{\sqrt{\log k}}\right).$
To see this, we may assume without loss of generality that $\tilde O(\sqrt{\mathcal{O}_{ij}}) > \eps/2$, since otherwise running \Cref{alg pairwise} with $\eps'=\eps/2$ already yields a sufficiently accurate hypothesis. If instead $\theta$ exceeds the above threshold, then \Cref{col surface} implies $\tilde O(\sqrt{\mathcal{O}_{ij}}) \ge \tilde \Omega\!\left(\frac{\mathcal{T}_{ij}^4 \Phi_{ij}}{\sqrt{\log k}}\right),$
contradicting the assumed regime. 
Thus, \Cref{alg pairwise} provides a suitable initialization.

Second, we claim that setting $\sigma = \Theta(\sin \theta/\sqrt{\mathcal{T}_{ij}})$ suffices to make the error after localization large enough for us to make a correct update. Formally, we have the following lemma that quantifies the correct choice of parameters for our localization. We defer the proof to \Cref{app proof localization surface}.
\begin{lemma}\label{lm localization surface}
   Let $f^*(x) = \argmax_{i \in [k]} (w^*_i \cdot x) : \R^d \to [k]$ be a multiclass linear classifier and let $h_w(x) = \sign(w \cdot x)$ be a halfspace. For $i,j \in [k], i \neq j$, let $\theta = \theta(w^*_{ij},w)$ with $\sin \theta < \sqrt{\mathcal{T}_{ij}}/C$ for a sufficiently large constant $C>0$. 
    Then we have 
    \begin{align*}
        \Pr_{x\sim \cN(0,\Sigma)}[(w^*_{ij}\cdot x)(w\cdot x)<0, f^*(x) \in \{i,j\}] \ge  \tilde \Omega(\mathcal{T}_{ij}^{4.5}\Phi_{ij}/\sqrt{\log k})\;,
    \end{align*}
where $\Sigma = I+(\sigma^2-1)ww^\top$ for any $\sigma \in [2\sin\theta/\sqrt{\mathcal{T}_{ij}},4\sin\theta/\sqrt{\mathcal{T}_{ij}}]$.
\end{lemma}

As in the case $k=3$, write $x=\Sigma^{1/2}z$ with $z\sim\cN(0,I)$ and define $y^*(z)=f^*(x)$. Then $z$ can be viewed as labeled by a transformed MLC with boundary $\tilde w^*_{ij}=\Sigma^{1/2}w^*_{ij}$. By \Cref{col surface}, 
\[
\Pr_{x\sim\cN(0,\Sigma)}[(w^*_{ij}\cdot x)(w\cdot x)<0,\ f^*(x)\in\{i,j\}]
\;\ge\;
\tilde\Omega\!\left(\frac{\tilde{\mathcal{T}}_{ij}^4 \tan \tilde{\theta}^*_{ij}}{\sqrt{\log k}}\right),
\]
which depends on the geometry of the transformed classifier.
A key issue is that if $\sigma$ is too small, the quantities $\tilde{\mathcal{T}}_{ij}$ and $\tan \tilde{\theta}^*_{ij}$ may shrink significantly. We avoid this by choosing $\sigma=\Theta(\sin\theta/\sqrt{\mathcal{T}_{ij}})$, for which these quantities remain comparable to their original values (see \Cref{lm critical angle}).
Combined with \Cref{lm localization surface}, this implies that, starting from a warm-start, 
whenever $\err_{ij}(h_w)$ is large, the angle $\theta$ is also sufficiently large, and the localized region contains a nontrivial fraction of misclassified examples with labels in $\{i,j\}$. At the same time, the fraction of corrupted samples in this region is small, so the resulting update direction remains reliable. Hence, this choice of $\sigma$ enables safe progress.
The remaining issue is that $\theta$ is unknown. To address this, we maintain an upper bound $\phi_t$ on $\theta_t$ and decrease it gradually. If $\theta_t \ll \phi_t$, then each update only slightly perturbs $w$, so $\phi_{t+1}$ remains a valid upper bound. If instead $\theta_t$ is close to $\phi_t$, the update reduces $\theta_t$. After $\poly(1/\eps)$ iterations, $\theta$ becomes sufficiently small, yielding a low-error hypothesis.

\section{Conclusions and Limitations}
We study agnostic learning of multiclass linear classifiers under Gaussian marginals and show that, unlike the binary case, standard approaches such as the multiclass perceptron can fail even in noiseless settings. To address this, we develop a pairwise learning framework together with structural results connecting classification error to the geometry of effective decision boundaries. This yields the first fully polynomial-time algorithms achieving error $\widetilde{O}(k^{3/2}\sqrt{\opt})+\epsilon$, with improvements to $O(\opt)+\epsilon$ for $k=3$ and $\poly(k)\opt+\epsilon$ under natural regularity conditions. Our analysis highlights the role of localization and Gaussian geometry in overcoming the $\sqrt{\opt}$ barrier. On the other hand, our results rely crucially on the Gaussian assumption, incur nontrivial dependence on $k$, and require geometric regularity for optimal (in $\opt$) error guarantees. We believe that characterizing the best error guarantee that can be achieved in $\poly(d,k,1/\eps)$ time is an important future direction.
Additionally, our algorithms are improper, and it remains open whether similar guarantees can be achieved by proper learners or extended to broader distributional settings.

\bibliographystyle{alpha}

\bibliography{mydb}

\newpage

\appendix

\section*{Supplementary Material}
The supplementary material is organized as follows.
In \Cref{app facts}, we provide necessary background on multiclass linear classifiers as well as probability lemmas that will be used in the proofs.
In \Cref{ap:omitted-lb}, we give the proof of \Cref{thm:lb-main body}, the lower bound on the perceptron updates.
In \Cref{app upper bound technique}, we prove the technical results developed in \Cref{sec upper bound technique} and give the proof of \Cref{th perceptron}.
In \Cref{app boost}, we provide the omitted details for the results in \Cref{sec boost}.

\section{Omitted Facts and Preliminaries}\label{app facts}
\begin{definition}[Binary and Multiclass margins \cite{shalev2014understanding}]\label{def:margins} 
\item[\textbf{Binary Margin}:] Let $(x,y)\in \R^d\times \{\pm 1\}$ and let  $h_w(x)=\sign( w\cdot x)$, for $w\in\s^{d-1}$ be a halfspace. 
We define the binary margin of $(x,y)$ to be
$$m_{\mathrm{bin}}(x,y,w)= \frac{y(w\cdot x)}{\norm{x}}\;.$$
\item{\textbf{Multiclass Margin}:} Let $(x,y)\in \R^d\times [k]$ and let $h_W(x)=\argmax_{i\in [k]} w_i\cdot x$ for $W=(w_1,\dots,w_k)$ with $w_i\in\R^d, i\in [k]$ and $\sum_{i=1}^k \norm{w_i}^2 \le 1$ be a multiclass linear classifier.
We define the multiclass margin of $(x,y)$ to be 
\[
m_{\mathrm{multi}}(x,y,W)=\min_{j\neq y} (w_{y}\cdot x-w_{j}\cdot x)/\norm{x}\;.
\]
\end{definition}
\begin{fact}[see e.g., \cite{vershynin2018high}]\label{fact:gaussiannorm}
For every $\tau > 0$ and $x \sim \mathcal{N}(0,I_d)$, it holds
\[
\Pr\!\left[\left|\|x\|^2-d\right|\ge \tau\right]\le e^{-\Omega(\min\{\tau,\tau^2/d\})}\;.
\]
\end{fact}

\begin{fact}[Fact B.4 in \cite{diakonikolas2024active}]\label{fact:simpleperturbation}
Let $D$ be a distribution on $\R^d \times \{0,\pm 1\}$ with standard normal $\x$-margin and let $\w,\mathbf{u}$ be two orthogonal unit vectors. Let $B$ be any interval over $\R$ and let $S(\x,y)$ be any event over $\R^d \times \{0,\pm 1\}$, such that $S(\x,y) \subseteq \{\w\cdot \x \in B\}$ then it holds
\begin{align*}
    \E_{(x,y)\sim D}\left[\abs{u\cdot x} \mathbf{1}\{S(x,y)\} \right]\le 2\sqrt{e} \Pr_{(x,y)\sim D}[S(x,y)] \sqrt{\log\left(\frac{\Pr_{(x,y)\sim D}[w\cdot x \in B]}{\Pr_{(x,y)\sim D}[S(x,y)]}\right)}.
\end{align*}
\end{fact}

\begin{fact}[Paley--Zygmund inequality]\label{fact:paley-zigumnd}
Let $Z$ be a non-negative random variable. Then
\[
\Pr\!\left[ Z > \frac{\mathbb{E}[Z]}{2} \right]
\;\ge\;
\frac{1}{4}\cdot \frac{\mathbb{E}[Z]^2}{\mathbb{E}[Z^2]}.
\]
\end{fact}

\begin{fact}[Lemma 7.11 \cite{diakonikolas2023algorithmic} ]\label{lem:rejection-sampling}
Let $v \in \R^d$ be a unit vector and let $\sigma \in (0,1]$.
Suppose $x \sim \cN(0,I)$ is accepted with probability $p(x)=\exp(-(\sigma^{-2}-1)(v\cdot x)^2/2)$. 
Then a random sample is accepted with probability $\sigma$ and  the distribution of $x$ conditioned on acceptance is $\cN(0, \Sigma_{v,\sigma})$,
where
\[
\Sigma_{v,\sigma} \;=\; I + (\sigma^2-1)vv^\top.
\]
\end{fact}

\begin{fact}[Komatsu's Inequality]\label{fact bias}
    For any $t \in \R$ the bias $p$ of a halfspace $h(x) = \sgn(w^* \cdot x+ t)$ can be bounded as 
    \begin{align*}
        \sqrt{\frac{2}{\pi}} \frac{\exp(-t^2/2)}{t+\sqrt{t^2+4}} \le p \le \sqrt{\frac{2}{\pi}} \frac{\exp(-t^2/2)}{t+\sqrt{t^2+2}}. 
    \end{align*}
\end{fact}

\section{Omitted Details from \Cref{sec:lower-bound}}
\label{ap:omitted-lb}

\subsection{Proof of \Cref{lem:blowup-main}}

Here we present our good angle implies  weight blow-up lemma that is an integral part of our lower bound. 
\begin{lemma}[Restatement of \Cref{lem:blowup-main}]\label{lem:blowup}
Let $d,k \in \Z_+,k\leq d$, and for each $i \in [k]$, let $w_i \in \R^{d}$. Let $c>0$ and $0<\eps<1/3$.
If $|w_{k-1,k}-w_{k,k}|\geq c$ and $\theta(w_i-w_j, e_i)<\eps$ for all $i<j\in [k]$,
then there exists $i\in [k]$ such that $\|w_i\|=\Omega( c/(3\eps)^{k-1})$.
\end{lemma}
\begin{proof}[Proof of \Cref{lem:blowup-main}]
For simplicity of notation, for any vector $v \in \mathbb{R}^{\blue{d}}$ and any coordinate $i \in [k]$, we write $\|v\|_{-i} := \|(v_1,\dots,v_{i-1},v_{i+1},\dots,v_{\blue{d}})\|$ for the Euclidean norm of $v$ after removing its $i$-th coordinate.
First wlog we prove the statement for $w_k=e_k$ this is because we can subtract $w_k-e_k$ from all the weights. This will not change the separation $|w_{k-1,k}-w_{k,k}|$ and moreover if for some $i\in [k]$, $\|w_i -w_k+e_k\|=\Omega ( c/(3\eps)^{k-1})$ then $\|w_i\|=\Omega( c/(3\eps)^{k-1})$ or $\|w_k\|=\Omega( c/(3\eps)^{k-1})$.

Now note that for any classes $i,j\in [k],i<j$ we have that  $\theta(w_i-w_j, e_i)<\eps$ implies $\cos^2\theta(w_i-w_j, e_i) \geq 1-\eps^2$. 
Moreover, $$\cos^2\theta(w_i-w_j, e_i)=\frac{(w_{i,i}-w_{j,i})^2}{\|w_i-w_{j}\|^2}=\frac{(w_{i,i}-w_{j,i})^2}{(w_{i,i}-w_{j,i})^2+\|w_i-w_{j}\|_{-i}^2}\;.$$ 
Hence, $(w_{i,i}-w_{j,i})^2\geq (1-\eps^2)\|w_{i} -w_{j}\|_{-i}^2/\eps^2$. 
As a result, when $\eps\in(0,1/3)$ we have 
$$|w_{i,i} -w_{j,i}| \geq \frac{2\sqrt{2}}{3\eps}\|w_{i} -w_{j}\|_{-i} \;.$$
Applying the above for the classes $k-1$ and $k$ we get 
$|w_{k-1,k-1}|\geq |w_{k-1,k}-w_{k,k}|/\eps\geq c2\sqrt{2}/(3\eps)$.
Moreover, for every $i$ we have that $\theta(w_{i}-w_{k}, e_i)\leq \eps$ implies that 
$|w_{i,i}|\geq (2\sqrt{2}/(3\eps))|w_{i,j}|$ for all  $i,j\in [k],j\neq i$. 

Now for a fixed label $1\leq i\leq k-2$ we will show that $|w_{i,i}|\geq (1/(3\eps))|w_{i+1,i+1}|$.

From $\theta(w_{i}-w_{i+1}, e_i)\leq \eps$, we have that 
$|w_{i,i}-w_{i+1,i}|\geq (2\sqrt{2}/(3\eps))|w_{i,i+1}-w_{i+1,i+1}|$.
Moreover, from the above we have $|w_{i+1,i}|\leq (3\eps/(2\sqrt{2}))|w_{i+1,i+1}|$.
Hence 
$$|w_{i,i}|\geq \frac{2\sqrt{2}}{3\eps}|w_{i,i+1}-w_{i+1,i+1}| -\frac{3\eps}{2\sqrt{2}}|w_{i+1,i+1}| \;.$$

We consider two cases. If $|w_{i,i+1}-w_{i+1,i+1}|\ge |w_{i+1,i+1}|/2$, then
$$|w_{i,i}|\ge \frac{2\sqrt{2}}{3\eps}|w_{i,i+1}-w_{i+1,i+1}|-\frac{3\eps}{2\sqrt{2}} |w_{i+1,i+1}|\ge \left(\frac{\sqrt{2}}{3\eps}-\frac{3\eps}{2\sqrt{2}}\right)|w_{i+1,i+1}| \ge \frac{1}{3\eps}|w_{i+1,i+1}|\;,$$
where the last inequality uses $\eps<1/3$. Otherwise, $|w_{i,i+1}-w_{i+1,i+1}|<|w_{i+1,i+1}|/2$, and so by the reverse triangle inequality,
$|w_{i,i+1}|\ge |w_{i+1,i+1}|-|w_{i,i+1}-w_{i+1,i+1}|\ge |w_{i+1,i+1}|/2$.
Since $\theta(w_i-w_k,e_i)\le \eps$, it follows that $|w_{i,i}|\ge \frac{2\sqrt{2}}{3\eps}|w_{i,i+1}|\ge \frac{1}{3\eps}|w_{i+1,i+1}|$.
Thus, in both cases, $|w_{i,i}|\ge \frac{1}{3\eps}|w_{i+1,i+1}|$.

Finally, applying the above statement inductively we get that $$\|w_{1}\|\geq |w_{1,1}|\geq |w_{k-1,k-1}|/(3\eps)^{k-2} =\Omega( c/(3\eps)^{k-1})\;,$$ which concludes the proof of \Cref{lem:blowup-main}.  
\end{proof}

Next we need the following lemma that relates the angle to error discrepancies. It is true that any discrepancy in error implies a discrepancy in decision boundary angles, however the inverse is not always true. This is because the discrepancy can happen to lower score classes that are not actually predicted. 
A phenomenon that does not happen for the binary case the error and angle discrepancies are the same up to a constant (for the gaussian distribution). 
In particular, in the multiclass case we can count only a part of the angular sector of discrepancy.
\begin{lemma}\label{lem:error2angle}
Let $d,k\in \Z_+$.
Let $f_W, f_{V}:\R^d\to [k]$ be multiclass linear classifiers and $D$ a distribution over $\R^d$ that assigns measure $0$ to all origin aligned hyperplanes $\{v\cdot x=0\}$ for all $v\in \R^d$. 
Denote by $R_{i,j}=\{x: \sign((w_i-w_j)\cdot x)\neq \sign((v_i-v_j)\cdot x) \}$.
It holds that 
$\Pr_{x\sim D}[f_W(x)\neq f_V(x)]\geq \pr_{x\sim D}[f_V(x)\in \{i,j\} , x\in R_{i,j}]$.
\end{lemma}
\begin{proof}
Fix $i,j\in[k]$ and define
\[
A \;\eqdef\; \{x\in\R^d : f_V(x)\in\{i,j\}\}\ \cap\ R_{i,j}.
\]
We claim that
\[
A \subseteq \{x\in\R^d : f_W(x)\neq f_V(x)\}.
\]
Indeed, fix any $x\in A$. Since $x\in R_{i,j}$, we have
\[
\sign((w_i-w_j)\cdot x)\neq \sign((v_i-v_j)\cdot x).
\]

Consider the case where $f_V(x)=i$. Then $v_i\cdot x \ge v_r\cdot x$ for all $r\in[k]$, and in particular
$v_i\cdot x \ge v_j\cdot x$, i.e.\ $(v_i-v_j)\cdot x \ge 0$. Since $(v_i-v_j)\cdot x\neq 0$ a.s., we have
$(v_i-v_j)\cdot x>0$, hence $\sign((v_i-v_j)\cdot x)=+1$. Because $x\in R_{i,j}$, it follows that
$\sign((w_i-w_j)\cdot x)=-1$, i.e.\ $(w_i-w_j)\cdot x<0$ and thus $w_j\cdot x>w_i\cdot x$.
Therefore $f_W(x)\neq i=f_V(x)$.
The same argument holds for $f_V(x)=j$.

Consequently,
\[
\Pr_{x\sim D}[f_W(x)\neq f_V(x)]
\;\ge\;
\Pr_{x\sim D}[A]
\;=\;
\Pr_{x\sim D}\big[f_V(x)\in\{i,j\},\ x\in R_{i,j}\big]\;.
\]
\end{proof}

Now we can prove our theorem about the complexity of the multiclass perceptron.

\subsection{Proof of \Cref{thm:lb-main body}}

\begin{proof}[Proof of \Cref{thm:lb-main body}]
Assume that $D_{y\mid x}$ is realizable by $f$ (even though $f$ is not realizable by a multiclass linear classifier receiving samples from an $f_M$ with $M$ sufficiently large suffices to for any finite sample size). 
Denote the number of mistakes by $m$ and by $n$ the total number of iterations/samples, $m<n$. 
Let $x_1,\dots,x_n$ to be the samples.
Also denote by $h_W$ the classifier returned by \Cref{alg:multi-perceptron} and by $w_i^{(t)}\in \R^{d},i\in [k]$ its weight vectors at iteration $t\in [n]$.

Let $l\in [k]$ denote the number of classes over which we measure the error. 
Note that we will apply \Cref{lem:blowup} only for the first $l$ classes for the projections of $w_{i}^{(t)}$ onto the space $V_{l}\eqdef \mathrm{span}(\{e_i\}_{i=1}^l)$.

Essentially we need to show that (1) the update vectors have bounded norm with high probability and (2) with high probability for $1/\eps^{O(l)}$ iterations $|w_{l-1,l}-w_{l,l}|\geq \eps^{O(l)}$ in order to apply \Cref{lem:blowup}. 
We begin by proving (2) using the anticoncetration of the gaussian. 
\begin{lemma}[Anticoncetration of $|w_{l-1,l}^{(t)}-w_{l,l}^{(t)}|$]\label{cl:anticoncetration}
For every $c\in (0,1)$,
it holds that $\Pr[\exists t\in [n]: |w_{l-1,l}^{(t)}-w_{l,l}^{(t)}|\leq c]=O( cn)$.
\end{lemma}
\begin{proof} 
Denote by $w^{(t)}_i\in \R^{d}$ the vectors $w_i,i\in [k]$ after the $t$'th update $t\in [n]$. 
Denote the quantity $\Delta^{(t)}\eqdef w_{l-1,l}^{(t)}-w_{l,l}^{(t)}$.
We will prove the statement by induction.

First for the base case note that because of the random initialization (Line \ref{line:init}) we have that $\Delta^{(0)}\sim \mathcal{N}(0,2)$ and therefore  $\Pr[|\Delta^{(0)}|\leq c]\leq \frac{2}{\pi}c/\sqrt{2}\leq 8c$. 

Now assume that for some $t\in \{0,\dots,n-1\}$, $\Pr[\exists i\in [t] : |\Delta^{(i)}|\leq c] \leq 8ct$,  we will argue that $\Pr[\exists i\in [t+1]: |\Delta^{(i)}|\leq c] \leq 8c(t+1)$.
Define $s^{(t)}=\Ind\{f(x_t)=l-1\}-\Ind\{f(x_t)=l\}-\Ind\{f_{w^{(t-1)}}(x_t)=l-1\}+\Ind\{f_{w^{(t-1)}}(x_t)=l\}$ and note that $\Delta^{(t)}= \Delta^{(t-1)} + x_{t,l}s^{(t)}$.
Note that the following event relationship trivially holds
$\{|\Delta^{(t+1)}| \leq c \}\subseteq  \{|\Delta^{(t)}|\leq c \}\cup  \{|\Delta^{(t)}|> c \land  |\Delta^{(t+1)}| \leq c\}$. 
Therefore,  
\begin{align*}
\{\exists i\in [t+1] :|\Delta^{(i)}|\leq c\}&=\{\exists i\in [t+1] :|\Delta^{(i)}|\leq c\} \\
&\subseteq\{\exists i\in [t] :|\Delta^{(i)}|\leq c\}\cup \{|\Delta^{(t)}|> c \land  |\Delta^{(t+1)}| \leq c\}
\end{align*}
Therefore, because $s^{(t)}\in \{0,\pm 1,\pm 2\}$ we have that
\begin{align*}
\Pr[\exists i\in [t+1] :|\Delta^{(i)}|\leq c]&
\leq 8ct+ \sum_{b\in \{0,\pm 1,\pm 2\}}\Pr[|\Delta^{(t)}+bx_{t,l}|\leq c, |\Delta^{(t)}|>c] \\
&\leq 8ct+ \sum_{b\in \{\pm 1,\pm 2\}}\Pr[|\Delta^{(t)}+bx_{t,l}|\leq c, |\Delta^{(t)}|>c] 
\end{align*}
Now fix $b\in \{\pm 1,\pm 2\}$ given that $|\Delta^{(t)}|>c$ and $Z\sim \mathcal{N}(0,1)$ we have 
\begin{align*}
\Pr[|\Delta^{(t)}+bZ| \le c]
= \Pr\left[Z \in \left[\frac{-c-\Delta^{(t)}}{b},\,\frac{c-\Delta^{(t)}}{b}\right]\right]
\le \frac{2c}{|b|}\cdot \sup_{z}\varphi(z)
\le 2c \;.
\end{align*}
Union bounding over the value of $b$, we have 
$$\sum_{b\in \{\pm 1,\pm 2\}}\Pr[|\Delta^{(t)}+bx_{t,l}|\leq c, |\Delta^{(t)}|>c]\leq \sum_{b\in \{\pm 1,\pm 2\}}\Pr[|\Delta^{(t)}+bx_{t,l}|\leq c|\; \mid\; \Delta^{(t)}|>c]\leq 8c \;.$$
Therefore,
$\Pr[\exists i\in [t+1] :|\Delta^{(i)}|\leq c]\leq 8c(t+1)\;.$
As a result we have that the statement follows by induction. 
\end{proof}

Now we apply our \Cref{lem:error2angle} and prove that under some angle separation we have comparable error with respect to $f$.
\begin{claim}[Angle implies error]\label{cl: error2anglef}
Let $1\le i<j\le k$, let $u\coloneqq w_i-w_{j}$ and let $\theta_{i,j}\eqdef \theta(u,e_i)$. 
It holds that $\Pr[f_W(x)\neq f(x)]\geq \theta_{i,j}/2^{2j}$.
\end{claim}
\begin{proof}
Define the event
\[
B_{i,j} \;\coloneqq\; \{x_1\le 0,\ldots,x_{i-1}\le 0,\ x_{i+1}\le 0,\ldots,x_{j-1}\le 0,\ x_j>0\}.
\]
Then $\Pr[B_{i,j}]= 2^{-(j-1)}$ by independence of Gaussian coordinates.
Moreover, on $B_{i,j}$ it holds $f(x)=i$ iff $x_i>0$ and $f(x)=j$ iff $x_i\le 0$.

Define
\[
E_{i,j} \;\coloneqq\; B_{i,j} \cap \{\sign(u\cdot x)\neq \sign(x_i)\}.
\]
By \Cref{lem:error2angle} $E_{i,j}\subseteq\{f_W(x)\neq f(x)\}$. Therefore
\[
\Pr[f_W(x)\neq f(x)] \;\ge\; \Pr[E_{i,j}] \;=\; \Pr[B_{i,j}]\cdot \Pr[\sign(u\cdot x)\neq \sign(x_i)\mid B_{i,j}].
\]

It remains to lower bound $\Pr[\sign(u\cdot x)\neq \sign(x_i)\mid B_{i,j}]$ in terms of $\theta_{i,j}$.
Write $u=\alpha e_i + b$ with $b\perp e_i$ and $\alpha=u_i$; set $r\coloneqq b\cdot x$, so $u\cdot x=\alpha x_i+r$.
Note that for $\alpha=0$ the $\Pr[\sign(u\cdot x)\neq \sign(x_i)\mid   B_{i,j}]=1/2$. 
Therefore it suffices to consider the case where $\alpha\neq 0$.
The coordinate $x_i\sim\mathcal N(0,1)$ is independent of $B_{i,j}$, since $B_{i,j}$ does not constrain $x_i$.
Conditioning on $r$, the event $\{\sign(\alpha x_i+r)\neq \sign(x_i)\}$  if $\alpha\neq 0$ contains an interval adjacent to $0$ of length $|r|/\alpha$.
Hence
\[
\Pr[\sign(\alpha x_i+r)\neq \sign(x_i)\mid r,B_{i,j}]\;=\Omega\left(\;\min\Big\{1,\frac{|r|}{|\alpha|}\Big\}\right)\;.
\]
Taking expectation over $r$ yields
\[
\Pr[\sign(u\cdot x)\neq \sign(x_i)\mid B_{i,j}]
\;=\Omega\;
\left(\E\Big[\min\Big\{1,\frac{|r|}{|\alpha|}\Big\}\ \Bigm|\ B_{i,j}\Big]\right).
\]
Note that $r=\|b\| \widehat{b}\cdot x$ for a unit vector $\widehat{b}\in \R^d$ and $\|b\|/\alpha=\tan(\theta_{i,j})$.
We will lower bound the variance of $\widehat{b}\cdot x$ and that will lead us to lower bound the probability of $|\widehat{b}\cdot x|$ being small from Paley-Zigmund (\Cref{fact:paley-zigumnd}).  
For simplicity we may assume that $B=\{x_1\geq 0,x_2\geq 0,\dots, x_{i-1}\geq 0,x_{i+1}\geq 0,\dots,x_j\geq 0\}$ this is because we can change the sign of the corresponding entries of $\widehat{b}$ without affecting its length. 
Specifically, denote $S_{i,j}=\{1,\dots, i-1,i+1,\dots,j\}$ we have
\begin{align*}
\E[|\widehat{b}\cdot x|^2 \mid B_{i,j}]
&= \E\!\left[ \sum_{j=1}^d \widehat{b}_j^2 x_j^2 +2\sum_{j<l}\widehat{b}_j\widehat{b}_l x_jx_l  \,\middle|\, B_{i,j}\right]\\
&=\sum_{j=1}^d \widehat{b}_j^2 +2\sum_{j<l}\widehat{b}_j\widehat{b}_l\E[x_jx_l\mid B_{i,j}]\\
&= \sum_{j=1}^d \widehat{b}_j^2 +\frac{4}{\pi}\sum_{\substack{j<l\\ j,l\in S_{i,j}}}\widehat{b}_j\widehat{b}_l\\
&=1-\frac{2}{\pi}\|\widehat b_{S_{i,j}}\|_2^2 +\frac{2}{\pi}\left(\sum_{j\in S_{i,j}}\widehat{b}_j\right)^2\\
&\ge 1-\frac{2}{\pi}=\Omega(1)\;, 
\end{align*}
where  in the third equality we used that, conditioned on $B_{i,j}$, the coordinates remain independent,
$\E[x_j^2\mid B_{i,j}]=1$ for all $j$, and for $j\neq l$ we have
$\E[x_jx_l\mid B_{i,j}]=\frac{2}{\pi}$ if $j,l\in S_{i,j}$, and $0$ otherwise.
In the fourth equality we completed the square.

Hence, from \Cref{fact:paley-zigumnd} we have that 
\begin{align*}
    \pr[ |\widehat{b}\cdot x|^2\geq \frac{1}{2}\E[|\widehat{b}\cdot x|^2 \mid B_{i,j}]] \geq \frac{1}{4} \frac{\E[|\widehat{b}\cdot x|^2 \mid B_{i,j}]^2 }{\E[|\widehat{b}\cdot x|^4 \mid B_{i,j}]} =\Omega( 2^{-(j-1)})\;, 
\end{align*}
where we used  the fact that $\E[|\widehat{b}\cdot x|^4 \mid B_{i,j}]=\E[|\widehat{b}\cdot x|^4 \Ind\{x\in B_{i,j}\}]/\Pr[B_{i,j}]=O( 2^{j-1})$.

Therefore\[
\E\Big[\min\Big\{1,\frac{|r|}{|\alpha|}\Big\}\ \Bigm|\ B_{i,j}\Big]
\;=\Omega\;
(2^{-(j-1)})\min\Big\{1,\frac{\|b\|_2}{|\alpha|}\Big\}
\;=\;
2^{-(j-1)}
\min\{1,\tan\theta_{i,j}\}
\;=\Omega(2^{-(j-1)}
\theta_{i,j})
\]
Putting everything together,
\[
\Pr[f_W(x)\neq f(x)] \;\ge\; \Pr[B_{i,j}]\cdot \Pr[\sign(u\cdot x)\neq \sign(x_i)\mid B_{i,j}]
=\Omega( 2^{-2(j-1)}\theta_{i,j})\;,
\]
which completes the proof of \Cref{cl: error2anglef}.
\end{proof}

Now we will need the following standard fact about the concentration of the norm of the gaussian to relate norm blow-up to large number of iterations.

First recall from \Cref{lem:blowup} for any $\eps<1/9$ we have that if for $t\in [n]$ $|w_{l-1,l}^{(t)}-w_{l,l}^{(t)}|<c$ and 
$\|w_i^{(t)}\|\leq c/\eps^{(l-1)/2}$ for all $i\in [l]$, then all  pairwise angles for the first $l$ classes can not be smaller than $\eps$.

Consider number of samples $n\leq 1/\eps^{(l-1)/6}$.
If we apply \Cref{cl:anticoncetration}  and $c=\eps^{(l-1)/4}$ we have that with probability at least $1-1/6$,  for all $t\in [n]$
$|w_{l-1,l}^{(t)}-w_{l,l}^{(t)}|\geq c$.
Furthermore, applying \Cref{fact:gaussiannorm} we have that with probability at least $1-1/6$ for all $t\in [n]$ $\|\proj_{V_l} x^{(t)}\|=O( l\sqrt{l}\log(1/\eps))$, since $\proj_{V_l}x$ is distributed as $\cN(0,I_l)$. 

Thus, by the union bound, 
with probability $2/3$ both events apply.
Hence, we have that for all $t\in [n]$ $|w_{l-1,l}^{(t)}-w_{l,l}^{(t)}|\geq \eps^{(l-1)/4}$
and $\|\proj_{V_l}w_i^{(t)}\|=O( n l\sqrt{l}\log(1/\eps))=O( l\sqrt{l}/e^{-(l-1)/5})$.
As a result, by \Cref{lem:blowup} we have that not all pairwise angles between the first $l$ classes can be smaller than $\eps/(l)^{1/\Omega(l)}$, which is at least $\eps^2$ for $\eps\leq 1/l^2$, since $\|\proj_{V_l}w_i^{(t)}\|=O( \eps^{(l-1)/4}l\sqrt{l}/\eps^{(l-1)/2} )$ and $\theta(\proj_{V_l}(w_i^{(t)}-w^{(t)}_{j}),e_i)\leq \theta(w_i^{(t)}-w^{(t)}_{j},e_i) $ for all $i\in [l]$. 
Which from \Cref{cl: error2anglef} implies that the error at least  $\eps^2/2^{2l}$. A simple reparameterization $\eps^2$ to $\eps$   concludes the proof of \Cref{thm:lb-main body}.
\end{proof}

\section{Omitted Details from \Cref{sec upper bound technique}}\label{app upper bound technique}

\subsection{Proof of \Cref{lm error decomposition}}\label{app proof error decomposition}

\begin{lemma}[Restatement of \Cref{lm error decomposition}]
    Let $D$ be a distribution over $\R^d \times [k]$ and $f^*(x) = \argmax_{i \in [k]}(w^{(i)}\cdot x)$ be a multiclass linear classifier such that $\err(f^*) = \opt$. Let $h_W$ be a pseudo $k$-multiclass linear classifier. Let $\eps \in (0,1)$. Suppose that there exist constant $C>1$ and $c \in [0,1/2]$ such that
    for every $i \neq j \in [k]$, $\err_{ij}(h_{ij}) \le C ( {\err^*_{ij}}^{1-c}\log^{c/2}(1/\err^*_{ij}))+\eps$, where $h_{ij} = \sign(w_{ij}\cdot x)$ and $\err^*_{ij} = \inf_{h \in H}(\err_{ij}(h))$, 
    then 
    $\err(h_W)  \le C k^{1+c}\opt^{1-c}\log^{c/2}(k/\opt)
 +k^2\eps$.
\end{lemma}

\begin{proof}[Proof of \Cref{lm error decomposition}]
For $i, j \in [k]$, let $h^*_{ij} = \sign((w^*_i-w^*_j)\cdot x)$.
Notice that for $i \in [k]$,
$f^*(x) = i$ if and only if $h^*_{ij}(x) = 1$ for all $j \neq i$. This implies that  
\begin{align*}
    \err_{ij}(h^*_{ij}) & := \Pr_{(x,y)\sim D}[h^*_{ij}(x) \neq y, y \in \{i,j\}] 
    = \sum_{r \in [k]} \Pr_{(x,y)\sim D}[h^*_{ij}(x) \neq y, f^*(x) = r,  y \in \{i,j\}] \\
    & \le \Pr_{(x,y)\sim D}[f^*(x) = i, y = j] + \Pr_{(x,y)\sim D}[f^*(x) = j, y = i] + \sum_{r \not \in \{i,j\}}\Pr_{(x,y)\sim D}[f^*(x)=r, y \in \{i,j\} ]   \\
    & = \opt_{ij}+\opt_{ji}+ \sum_{r \not \in \{i,j\}} (\opt_{r,i}+\opt_{r,j}). 
\end{align*}
Here $\opt_{ij}:=\Pr_{(x,y)\sim D}[f^*(x) =i, y = j]$.
This implies that if for $i,j \in [k]$, we have hypothesis $h_{ij}$ such that 
$\err_{ij}(h_{ij}) \le  C ( {\err^*_{ij}}^{1-c}\log^{c/2}(1/\err^*_{ij})) + \eps$, then the induced pseudo $k$-multiclass linear classifier satisfies
\begin{align*}
    \err(h_W)  & = \sum_{j = 1}^k \Pr_{(x,y)\sim D}[h_W(x) \neq y, y=j] \le \sum_{j = 1}^k \sum_{i \neq j} \Pr_{(x,y)\sim D}[w_{ji}\cdot x \le 0, y=j] \\
    & = \frac{1}{2} \sum_{(i,j):i \neq j} (\Pr_{(x,y)\sim D}[w_{ji}\cdot x <0, y=i ]+\Pr_{(x,y)\sim D}[w_{ji}\cdot x >0, y=j ]) \\
    &\le \frac{1}{2} \sum_{(i,j):i \neq j} \left(C( \sum_{r \neq i} \opt_{ri} + \sum_{r \neq j} \opt_{rj} )^{1-c}\log^{c/2}(1/(\sum_{r \neq i} \opt_{ri} + \sum_{r \neq j} \opt_{rj}))+\eps\right) \\
    & = \frac{1}{2}C \sum_{i \neq j}(A_i+B_j)^{1-c}\log^{c/2}(1/(A_i+B_j)) + k^2\eps.
\end{align*}

Here, for each $i,j \in [k], i \neq j$, we denote by $A_i = \sum_{r \neq i}\opt_{ri}$ and $B_j=\sum_{r \neq j}\opt_{rj}$. Notice that $\sum_{i \in [k]} A_i = \sum_{j \in [k]}B_j = \opt$. 
This implies that $\sum_{i \neq j}(A_i + B_j) = 2(k-1)\opt.$
Using Jensen's inequality, we have 
\begin{align*}
    &\sum_{i \neq j}\frac{1}{k(k-1)}(A_i+B_j)^{1-c}\log^{c/2}(1/(A_i+B_j))\\& \le \left(\sum_{i \neq j}\frac{1}{k(k-1)}(A_i+B_j)\right)^{1-c}\log^{c/2}\left(\sum_{i \neq j}\frac{1}{k(k-1)}(A_i+B_j)\right)^{-1} \\
    &= (2\opt/k)^{1-c}\log^{c/2}(k/(2\opt)).
\end{align*}
Together, this implies that 
\begin{align*}
    \err(h_W) \le  C k(k-1)(\opt/k)^{1-c}\log^{c/2}(k/\opt) + k^2\eps \le C k^{1+c}\opt^{1-c}\log^{c/2}(k/\opt)
 +k^2\eps. 
 \end{align*}    
\end{proof}

\subsection{Proof of \Cref{lm convex body}}\label{app proof convex body}

\begin{lemma}[Restatement of \Cref{lm convex body}]
Let $w \in \s^{d-1}$ be any unit vector and let $S$ be any measurable event such that
$
S \subseteq \{x \in \R^d : w \cdot x \ge 0\}.
$
Then
\begin{align*}
\E_{x \sim \cN(0,I)}[(w \cdot x)\Ind\{x \in S\}]
\ge
\sqrt{\pi/2}\,\Pr_{x \sim \cN(0,I)}[x \in S]^2.    
\end{align*}
\end{lemma}

\begin{proof}[Proof of \Cref{lm convex body}]
First note that 
\begin{align*}
\E_{x \sim \cN(0,I)}[(w \cdot x)\Ind\{x \in S\}] = \E_{z \sim \cN(0,1)}[z\E_{x\sim \cN(0,I)}[\Ind\{x \in S\}|w\cdot x=z]]
\end{align*}
Define $g(z) \eqdef \E_{x\sim \cN(0,I)}[\Ind\{x \in S\}\mid w\cdot x=z].$
Note that $\E_{z\sim N(0,1)}[g(z)]=\pr_{x\sim \cN(0,I)}[x\in S]$ and $0\le g\leq 1$ almost everywhere.
Furthermore, since $S\subseteq \{x\in \R^d:w \cdot x\geq 0\}$ we have that 
$g(z)=0$ for every $z<0$.

Thus, it suffices to show the following statement: 
if $g:[0,\infty)\to[0,1]$ is measurable, then
$$
\int_0^\infty z g(z)\phi(z)\,dz
\ge
\sqrt{\pi/2}\left(\int_0^\infty g(z)\phi(z)\,dz\right)^2,
$$
where $\phi(z)$ is the density function of $N(0,1)$.
Fix
$$
p := \int_0^\infty g(z)\phi(z)\,dz.
$$
Choose $b \ge 0$ such that
$$
\int_0^b \phi(z)\,dz = p.
$$
We claim that among all measurable $g:[0,\infty)\to[0,1]$ with $\int_0^\infty g(z)\phi(z)\,dz = p$, the quantity
$
\int_0^\infty z g(z)\phi(z)\,dz
$
is minimized by $g=\Ind\{z \in [0,b]\}$.

Indeed,
\begin{align*}
\int_0^\infty z g(z)\phi(z)\,dz - \int_0^b z\phi(z)\,dz
&=
\int_0^b z(g(z)-1)\phi(z)\,dz + \int_b^\infty z g(z)\phi(z)\,dz.
\end{align*}
On $[0,b]$ we have $z \le b$ and $g(z)-1 \le 0$, hence
$$
z(g(z)-1) \ge b(g(z)-1).
$$
On $[b,\infty)$ we have $z \ge b$ and $g(z)\ge 0$, hence
$$
z g(z) \ge b g(z).
$$
Therefore
\begin{align*}
\int_0^\infty z g(z)\phi(z)\,dz - \int_0^b z\phi(z)\,dz
&\ge
b\int_0^b (g(z)-1)\phi(z)\,dz + b\int_b^\infty g(z)\phi(z)\,dz \\
&=
b\left(\int_0^\infty g(z)\phi(z)\,dz - \int_0^b \phi(z)\,dz\right) \\
&= 0\;,
\end{align*}
where in the last inequality we used the definition of $b$. 
So a minimizing one-dimensional event is the event $z\in[0,b]$.

It remains to prove
$$
\int_0^b z\phi(z)\,dz
\ge
\sqrt{\pi/2}\left(\int_0^b \phi(z)\,dz\right)^2.
$$
Define
$$
p(b) := \int_0^b \phi(z)\,dz,
\qquad
m(b) := \int_0^b z\phi(z)\,dz.
$$
View $m$ as a function of $p$, namely $M(p):=m(b)$ where $p=p(b)$. Then by the Leibnitz derivative rule
$$
M'(p)=\frac{m'(b)}{p'(b)}=\frac{b\phi(b)}{\phi(b)}=b,
$$
and hence
$$
M''(p)=\frac{db}{dp}=\frac{1}{\phi(b)} \ge \frac{1}{\phi(0)}=\sqrt{2\pi}.
$$
Also note that 
$M(0)=0$, $M'(0)=0.$
Therefore, by the fundamental theorem of calculus for every $p\in[0,1/2]$,
$$
M'(p)=M'(0)+\int_0^p M''(u)\,du
\ge
\int_0^p \sqrt{2\pi}\,du
=
\sqrt{2\pi}\,p,
$$
and so
$$
M(p)=M(0)+\int_0^p M'(u)\,du
\ge
\int_0^p \sqrt{2\pi}\,u\,du
=
\frac{\sqrt{2\pi}}{2}p^2
=
\sqrt{\pi/2}\,p^2.
$$ 
This completes the proof of the lemma.
\end{proof}

\subsection{Proof of \Cref{lm gaussian event mean}}\label{app proof conditional mean}

\begin{fact}\label{lm gaussian event mean}
    Let $x\sim \cN(0,\Sigma)$, where $\Sigma\succeq 0$ and $\norm{\Sigma}_2\le 1$. Then for any event $A$,
    $$
        \norm{\E[x\Ind\{x\in A\}]}
        \le
        \frac{1}{\sqrt{2\pi}}.
    $$
\end{fact}

\begin{proof}
    If $\E[x\Ind\{x\in A\}]=0$, the claim is immediate. Otherwise, let
    $$
        u:=\frac{\E[x\Ind\{x\in A\}]}{\norm{\E[x\Ind\{x\in A\}]}}.
    $$
    Then
    $$
        \norm{\E[x\Ind\{x\in A\}]}
        =
        \E[(u\cdot x)\Ind\{x\in A\}]
        \le
        \E[(u\cdot x)_+].
    $$
    Since $u\cdot x\sim \cN(0,u^\top\Sigma u)$, we have
    $$
        \E[(u\cdot x)_+]
        =
        \frac{\sqrt{u^\top\Sigma u}}{\sqrt{2\pi}}
        \le
        \frac{\sqrt{\norm{\Sigma}_2}}{\sqrt{2\pi}}
        \le
        \frac{1}{\sqrt{2\pi}}.
    $$
    This concludes the proof.
\end{proof}

\subsection{Proof of \Cref{lm gradient descent}}\label{app proof gradient descent}

\begin{lemma}[Restatement of \Cref{lm gradient descent}]
Let $w^*,w_i \in \s^{d-1}$ and let $G_i$ be a random vector drawn from some distribution $\mathcal{D}$ such that with probability $1$, $G_i \perp w_i$. Let $g_i$ be the mean of $G_i$. Let $\widehat{g_i}$ be the empirical mean of $G_i$ and $\mu_i>0$. The update rule $w_{i+1} = \proj_{\s^{d-1}} (w_{i} + \mu_i\widehat{g_i})$ satisfies the following property,
    \begin{align*}
        \norm{w_{i+1}-w^*}^2 \le \norm{w_{i}-w^*}^2 - 2\mu_i g_i \cdot w^* + 2\mu_i\norm {g_i-\widehat{g_i}} +\mu_i^2\norm{\widehat{g_i}}^2.
    \end{align*}
In particular, if $g_i \cdot w^* \ge \alpha>0$ and $\norm{g_i} \le \beta$  and $\norm{g_i-\widehat{g_i}} \le \min(\alpha,\beta)/2$, then by setting $\mu_i = \alpha/(C\beta^2)$ for some large enough constant $C>0$, we have $\norm{w_{i+1}-w^*}^2 \le \norm{w_{i}-w^*}^2 - \alpha^2/(C\beta^2)$.
\end{lemma}

\begin{proof}[Proof of \Cref{lm gradient descent}]
    We first observe that 
\begin{align*}
    \norm{w_{i+1}-w^*}^2 = \norm{\proj_{\s^{d-1}}(w_{i} + \mu_i\widehat{g_i}) - \proj_{\s^{d-1}} (w^*)}^2 \le \norm{w_{i} + \mu_i\widehat{g_i}-w^*}^2. 
\end{align*}
It remains to upper bound $\norm{w_{i} + \mu_i\widehat{g_i}-w^*}^2$. We have 
\begin{align*}
    \norm{w_{i} + \mu_i\widehat{g_i}-w^*}^2 & = \norm{w_{i}-w^*}^2 + 2\mu_i\widehat{g_i}\cdot(w_i-w^*)+ \mu_i^2\norm{\widehat{g_i}}^2 \\ 
    & \le  \norm{w_{i}-w^*}^2 - 2\mu_i\widehat{g_i}\cdot w^*+ \mu_i^2\norm{\widehat{g_i}}^2 \\
    & =  \norm{w_{i}-w^*}^2 - 2\mu_ig_i \cdot w^*+ 2\mu_i(g_i-\widehat{g_i}) \cdot w^* +\mu_i^2\norm{\widehat{g_i}}^2 \\
    & \le \norm{w_{i}-w^*}^2 - 2\mu_ig_i \cdot w^*+ 2\mu_i\norm {g_i-\widehat{g_i}} +\mu_i^2\norm{\widehat{g_i}}^2. 
\end{align*}
Here, in the second equality, we use the fact that $\widehat{g_i} \cdot w_i = 0$. Since $g_i \cdot w^* \ge \alpha>0$ and $\norm{g_i} \le \beta$  and $\norm{g_i-\widehat{g_i}} \le \min(\alpha,\beta)/2$, we have $2\mu_i\norm{g_i-\widehat{g_i}} \le \mu_i\alpha$ and $\norm{\widehat{g_i}} \le \norm{g_i} + \norm{g_i-\widehat{g_i}} \le 3\beta/2$. Therefore,
\begin{align*}
    \norm{w_{i+1}-w^*}^2
    &\le \norm{w_{i}-w^*}^2 - \mu_i\alpha + \frac{9}{4}\mu_i^2\beta^2.
\end{align*}
By setting $\mu_i = \alpha/(C\beta^2)$ for a sufficiently large enough constant $C>0$,  we have $\norm{w_{i+1}-w^*}^2 \le \norm{w_{i}-w^*}^2 - \alpha^2/(C\beta^2)$. 
This completes the proof of \Cref{lm gradient descent}.
\end{proof}

\subsection{Proof of \Cref{lm pairwise}}\label{app proof pairwise}

\begin{lemma}[Restatement of \Cref{lm pairwise}]
    Consider the problem of agnostically learning a $k$-multiclass linear classifier under the standard Gaussian distribution. Given any $i \neq j \in [k]$ and $\eps \in (0,1),\delta \in (0,1)$, \Cref{alg pairwise} draws $d\poly(1/\eps)\log(1/\delta)$ examples from $D$, runs in $\poly(m)$ time and outputs a halfspace $\widehat{h}_{ij}$ such that with probability at least $1-\delta$, $\err_{ij}(\widehat{h}_{ij}) \le \tilde O (\sqrt{\err^*_{ij}}) + \eps$, where $\err^*_{ij} = \inf_{h \in H}(\err_{ij}(h))$.
\end{lemma}

\begin{proof}[Proof of \Cref{lm pairwise}]
    Let $h^*(x) = \sign(w^*\cdot x)$ be the optimal halfspace such that $\err^*_{ij}:=\err_{ij}(h^*) = \inf_{h \in \H} \err_{ij}(h)$. 
    In the rest of the proof, we will also treat the image of $h^*$ to be $\{i,j\}$. That is to say $h^*(x) = i$ if $w^* \cdot x \ge 0$ and $h^*(x) = j$, otherwise.

Let $w \in \s^{d-1}$ such that $w\cdot w^* \ge 0$.
We write $w^* = a w + b u$, where $a,b \in (0,1)$, $a^2 + b^2 = 1$ and $u \perp w, u \in \s^{d-1}$. 
We define the pair wise mistake cones with respect to the clean labels
\begin{align*}
    M_i(w) := \{x \in \mathbb{R}^d : h^*(x)=i,\; w\cdot x \le 0\},
\qquad
M_j(w) := \{x \in \mathbb{R}^d : h^*(x)=j,\; w\cdot x \ge 0\}\;.
\end{align*}

Next, we define random vector 
\begin{align*}
G_{ij}(w):=\proj_{w^\perp}(x\Ind\{w\cdot x\le 0,  y=i\}-x\Ind\{w \cdot x \ge 0, y = j\}).    
\end{align*}
Denote by $g_{ij}(w) = \E_{(x,y)\sim D}[G_{ij}(w)]$ and $g^*_{ij}(w) = \E_{x\sim \cN(0,I)}[x \Ind\{x \in M_i(w), h^*(x) = y\} - x \Ind\{x \in M_j(w), h^*(x) = y\}]$.
We next lower bound the correlation between $g_{ij}(w)$ and $w^*$.
Since $G_{ij}(w) \perp w$, we have that 
\begin{align*}
    g_{ij}(w) \cdot w^* & = g_{ij}(w) \cdot (a w + b u) = b g_{ij}(w) \cdot u \\&\ge b g^*_{ij}(w) \cdot u - b \E_{(x,y)\sim D}[\abs{x\cdot u}\Ind\{h^*(x) \neq y, y \in \{i,j\}\}]  \\
    & = b \left(g^*_{ij}(w) \cdot u -  \E_{(x,y)\sim D}[\abs{x\cdot u}\Ind\{h^*(x) \neq y, y \in \{i,j\}\}]\right) \;.
\end{align*}
Since $\Pr_{(x,y)\sim D}[h^*(x) \neq y, y \in \{i,j\}] = \err_{ij}^*$, by \Cref{fact:simpleperturbation}, we have that 
\begin{align*}
    \E_{(x,y)\sim D}[\abs{x\cdot u}\Ind\{h^*(x) \neq y, y \in \{i,j\}\}] \le O(\err^*_{ij}\sqrt{\log(1/\err^*_{ij})}).
\end{align*}
Notice that for every $x \in M_i(w)$, $w^*\cdot x \ge 0$ and for every $x \in M_j(w)$, $w^*\cdot x \le  0$.
This implies that for every $x \in M_i(w)$, $ u \cdot x \ge 0$ and for every $x \in M_j(w)$, $ u \cdot x \le 0$.
Thus, by \Cref{lm convex body}, and Jensen's inequality, we have 
\begin{align*}
    g^*_{ij}(w) \cdot u & \ge \sqrt{\pi/2}\left( \Pr_{x \sim \cN(0,I)}[x \in M_i(w), h^*(x) = y]^2 + \Pr_{x \sim \cN(0,I)}[x \in M_j(w), h^*(x) = y]^2\right) \\
    & \ge \Omega \left(\Pr_{x \sim N(0,I)}[x \in M_i(x) \cup M_j(x), h^*(x) = y]^2\right) \;.
\end{align*}
Notice that for the halfspace $h_w(x) = \sign(w\cdot x)$
\begin{align*}
  \Pr_{(x,y) \sim D}[x \in M_i(x) \cup M_j(x), h^*(x) = y] + \err^*_{ij} \ge   \err_{ij}(h_w) \ge \Pr_{(x,y) \sim D}[x \in M_i(x) \cup M_j(x), h^*(x) = y].
\end{align*}
This implies that when $\err_{ij}(h_w)>\tilde \Omega(\sqrt{\err^*_{ij}}) + \eps$,
we have 
\begin{align*}
  g_{ij}(w) \cdot w^* \ge b( g^*_{ij}(w) \cdot u - \E_{(x,y)\sim D}[\abs{x\cdot u}\Ind\{f^*(x) \neq y, y \in \{i,j\}\}] ) \ge b\eps^2 .
\end{align*}
Notice that if $b<\eps/2$, then we have $\Pr_{(x,y)\sim D}[x \in M_i(x) \cup M_j(x), h^*(x) = y] \le \eps/2$ and we have $\err_{ij}(h_w)>\err^*_{ij} + \eps/2$, which gives a contradiction. This implies that $g_{ij}(w) \cdot w^* \ge \eps^3/2$.
Furthermore, by \Cref{lm gaussian event mean}, we have $\norm{g_{ij}(w)} \le O(1)$.
Notice that by Hoeffding's inequality, for every $\delta>0$, by taking $d\poly(1/\eps)\log(1/\delta)$ examples from $D$, the empirical mean estimation $\widehat{g}_{ij}(w)$ satisfies $\norm{g_{ij}(w) - \widehat{g}_{ij}(w)}< \eps^3/4$.

We now show the correctness of \Cref{alg pairwise}, let $T = \poly(1/\eps)$.
Without loss of generality, we assume that $w^0_{ij}\cdot w>0$, because otherwise, by randomly sample $O(\log(1/\delta))$ random direction from $\s^{d-1}$, with probability at least $1-\delta/100$, there must be at least one of such $w^0_{ij}$.

Suppose that for $t \in [T]$, we have $\err_{ij}(h^t_{ij}) = \err_{ij}(\sign(w^t_{ij}\cdot x))>\tilde \Omega(\sqrt{\err^*_{ij}}) + \eps/4$.
Then, for $t \in [T]$, we have 
$g_{ij}(w_{ij}^t)\cdot w^*>\eps^3/2$, $\norm{g_{ij}(w_{ij}^t)}\le O(1)$
and with probability at least $1-\delta/(100T)$, $\norm{g_{ij}(w^t_{ij}) - \widehat{g}_{ij}(w^t_{ij})}< \eps^3/4$.
By \Cref{lm gradient descent}, we have 
\begin{align*}
    \norm{w^*-w^{t+1}_{ij}}^2 \le \norm{w^*-w^{t}_{ij}}^2 - \Omega(\eps^6).
\end{align*}
Since $\norm{w^0_{ij}-w^*} \le 2$, this implies that after $T$ rounds, $\norm{w^*-w^{t+1}_{ij}}^2 <0$, which gives contradiction.
This implies that there must be some $t \in [T]$ such that $\err_{ij}(\sign(w^t_{ij}\cdot x))\le \tilde O(\sqrt{\err^*_{ij}}) + \eps/4$.

Finally, by drawing $\poly(\log(1/\delta)/\eps)$ examples from $D$, for each $t \in [T]$, the empirical error of $\sign(w^t_{ij}\cdot x)$ is $\eps/4$ close to $\err_{ij}(\sign(w^t_{ij}\cdot x))$. Thus, the one with the smallest empirical error must have  $\err_{ij}(\sign(w^t_{ij}\cdot x))\le \tilde O(\sqrt{\err^*_{ij}}) + \eps$. This means with probability at least $1-\delta$, we learn a halfspace $\sign(\hat w_{ij}\cdot x)$ such that $\err_{ij}(\sign(\hat w_{ij}\cdot x))\le \tilde O(\sqrt{\err^*_{ij}}) + \eps$.

\end{proof}

\subsection{Proof of \Cref{th perceptron}}\label{app proof perceptron}

\begin{theorem}[Restatement of \Cref{th perceptron}]
Consider the problem of agnostic learning a $k$-multiclass linear classifier under the standard Gaussian distribution. There is an algorithm such that for every $\eps,\delta \in (0,1)$, it draws a set of $m = d\poly(1/\eps,k)\log(1/\delta)$ examples from $D$ runs in $\poly(m)$ time and outputs a hypothesis $\widehat{h}$ such that with probability at least $1-\delta$, $\err(\widehat{h}) \le \Tilde{O}(k^{3/2}\sqrt{\opt})+\eps$.
\end{theorem}

\begin{proof}[Proof of \Cref{th perceptron}]
    We show that \Cref{alg pairwise to pseudo} using \Cref{alg pairwise} as a halfspace learner satisfies the statement. For each $i \neq j \in [k]$, we know that with probability at least $1-\delta/k^2$, we have $\widehat{h}_{ij} = \sign(\widehat{w}_{ij}\cdot x)$ satisfies $\err_{ij}(\sign(\hat w_{ij}\cdot x))\le \tilde O(\sqrt{\err^*_{ij}}) + \eps/k^2 $, where $\err^*_{ij} = \inf_{h \in \H}(\err_{ij}(h))$. By \Cref{lm error decomposition}, we have that the pseudo $k$-multiclass linear classifier $h_{\widehat{W}}$ satisfies $\err(h_{\widehat{W}}) \le \tilde O(k^{3/2}\sqrt{\opt})+\eps$.
\end{proof}

\section{Omitted Details from \Cref{sec boost}}\label{app boost}

\subsection{Proof of \Cref{lm localization}}\label{app proof localization}

\begin{algorithm}[htbp]
 		\caption{\textsc{Pairwise Localization} (Learn to distinguish class $i$ from class $j$ for $k=3$)}\label{alg pairwise localization}
		\begin{algorithmic} [1]
\State\textbf{Input:} Sample access to distribution $D$ over $\R^d\times [k]$, labels $i,j\in [k]$, error $\eps$, confidence $\delta$. 
\State\textbf{Output:} $w_{ij} \in \s^{d-1}$ that approximately separates class $i$ and class $j$.
\State Let $C$ be a sufficiently large constant and let $N=d/\eps^C\log(1/\delta)$, $T=(1/\eps)^C$,
\State Let $\mathcal{S}=\{\ell\eps : 1\le \ell \le \lceil 1/\eps\rceil\}\cap(0,1]$.
\For{$\sigma \in \mathcal{S}$}
\State Initialize $t\gets 0$, $w_{ij}^{(t)} \sim \s^{d-1}$.
\While{$t<T$}
\State For $x\sim N(0,I)$, accept $x$ with probability
$p_t(x)=\exp(-(\sigma^{-2}-1)(w_{ij}^{(t)}\cdot x)^2/2)$.
\State Denote by $D'(w_{ij}^{(t)},\sigma)$ the distribution of the accepted samples.

\State Draw $N$ samples from $D'(w_{i,j}^{(t)})$ and let $\widehat{v}$ be the empirical estimate of $$v= \E_{(x,y)\sim D'(w_{ij}^{(t)})}[ x\Ind\{\{x \mid y=i, w_{ij}^{(t)}\cdot x \le 0\})] - \E_{(x,y)\sim D'(w_{ij}^{(t)})}[ x\Ind\{\{x \mid y=j, w_{ij}^{(t)}\cdot x \ge 0\})]$$
\State $w_{ij}^{(t+1)} \gets \proj_{\s^{d-1}}( w_{ij}^{(t)} + \mu \proj_{(w_{ij}^{(t)})^\perp} \widehat{v}$), $\mu= \epsilon/C$, $t\gets t+1$.

\EndWhile
\State Use $N$ fresh samples from $D$ to estimate the empirical error $\widehat{\err}_{ij}(w_{ij}^{(t)})$. \label{line: empriical error}
\State $h_{i,j}^\sigma= \argmin_{t \in [T]} \widehat{\err}_{ij}(h^t_{ij})$. \label{line: t-select}
\EndFor
\State\Return $\argmin_{\sigma \in \mathcal S} \widehat{\err}_{ij}(h_{i,j}^\sigma)$. \label{line: sigma-select}
\end{algorithmic}
\end{algorithm}

\begin{lemma}[Restatement of \Cref{lm localization}]
Consider the problem of agnostic learning a $3$-multiclass linear classifier under the standard Gaussian distribution. Given any $i \neq j \in [3]$ and $\eps,\delta \in (0,1)$, \Cref{alg pairwise localization} draws $m=d\poly(1/\eps)\log(1/\delta)$ examples from $D$, runs in $\poly(m)$ time and outputs a halfspace $\widehat{h}_{ij}$ such that with probability at least $1-\delta$, $\err_{ij}(\widehat{h}_{ij}) \le O(\mathcal{O}_{ij})+\eps$, where 
$\mathcal{O}_{ij} = \sum_{a\neq b:a \in \{i,j\}\; \mathrm{or }\; b \in \{i,j\}}\opt_{a,b}$, $\opt_{a,b} = \Pr_{(x,y) \in D}[f^*(x) = a, y=b]$.  
\end{lemma}

\begin{proof}[Proof of \Cref{lm localization}]

   Let $f^*(x) = \argmax_{i \in [3]} (w_{i}^*\cdot x)$ and let $h^*_{ij} = \sign(w^*_{ij} \cdot x)$, where $w^*_{ij} = (w^*_i - w^*_j)/\norm{w^*_i - w^*_j}$. 
    In the rest of the proof, we will also treat the image of a halfspace $\sign(w \cdot x)$ to be $\{i,j\}$. That is to say $\sign(w\cdot x) = i$ if $w \cdot x \ge 0$ and $\sign(w \cdot x) = j$, otherwise.

    Let $w \in \s^{d-1}$ be a unit vector such that $w\cdot w^*_{ij}>0$. We define 
\begin{align*}
    M_i(w) := \{x \in \mathbb{R}^d : h^*_{ij}(x)=i,\; w\cdot x \le 0\},
\qquad
M_j(w) := \{x \in \mathbb{R}^d : h^*_{ij}(x)=j,\; w\cdot x \ge 0\}\;.
\end{align*}
and define the mistake set $E_{ij}(w) = \{(x,y) \mid y = i, w\cdot x \le 0\} \cup \{(x,y) \mid y = j, w\cdot x \ge 0\}.$
Notice that for halfspace $h_w = \sign(w\cdot x)$, we have $\err_{ij}(h_w) = \Pr_{(x,y)\sim D}[(x,y) \in E_{ij}(w)]$

Define the localization matrix $\Sigma \in \R^{d \times d}$ to be a symmetric matrix such that $\Sigma w = \sigma^2 w$ and $\Sigma u = u$, for every $u \perp w$.
For each $x \in \R^d$, we write $x = \Sigma^{1/2} z$, $z \sim \cN(0,I)$ we will associate an artificial label $\tilde y(z):= y(x)$ for each $z \in \R^d$ and we define the artificial mistake set as 
\begin{align*}
    \widetilde E_{ij}(w) = \{(z,\tilde y) \mid \tilde y = i, w\cdot z \le 0\} \cup \{(z,\tilde y) \mid \tilde y = j, w\cdot z \ge 0\}.
\end{align*}
Denote by $D'(w)$ the distribution of $(z,\tilde y)$.

We prove the following claim.
\begin{claim}\label{cl error prob}
There is some large universal constant $C>0$ such that if $\Pr_{(x,y)\sim D}[E_{ij}(w)] \ge C \mathcal{O}_{ij} + \eps $, then by choosing $\sigma \in [2(C-1) \mathcal{O}_{ij}, 2(C-1)\mathcal{O}_{ij}+\eps]$, $\Pr_{(z,\tilde y)\sim D'(w)}[\widetilde{E}_{ij}(w)] \ge 1/10$.
\end{claim}

\begin{proof}[Proof of \Cref{cl error prob}]
We write $w^*_{ij} = a w + b u$, where $a,b>0$, $a^2+b^2 = 1$, $u \in \s^{d-1}, u \perp w.$
For each example $z$, with $x = \Sigma^{1/2} z$, we also associate an artificial clean label $y^*(z) = f^*(x)$.
This implies that $y^*(z) = i$ if and only if $w^*_{i\ell} \cdot \Sigma^{1/2} z \ge 0$, for $\ell \neq i$.
we will first show that $\Pr_{z \sim \cN(0,I)}[h_w(z) \neq y^*(z), y^*(z) \in \{i,j\}] \ge 1/6$.
Since $\Pr_{(x,y)\sim D}[E_{ij}(w)] \ge C \mathcal{O}_{ij} + \eps $, we know that 
\begin{align*}
    \Pr_{x \sim \cN(0,I)}[(w\cdot x)(w^*_{ij}\cdot x) \le 0] \ge \Pr_{(x,y)\sim D}[E_{ij}(w)] -  \sum_{a\neq b:a \in \{i,j\}\; \mathrm{or }\; b \in \{i,j\}}\opt_{a,b} \ge (C-1)\mathcal{O}_{ij} + \eps,
\end{align*}
Since $x \sim \cN(0,I)$, we know that $\Pr_{x \sim \cN(0,I)}[(w\cdot x)(w^*_{ij}\cdot x) \le 0] = \theta(w^*_{ij},w)/\pi \ge (C-1)\mathcal{O}_{ij} + \eps.$
Write $\theta = \theta(w^*_{ij},w)$ for short.
Since for $\theta \in (0,\pi/2)$, we know that $\sin \theta \ge 2\theta/\pi$, which implies that $b=\sin \theta \ge 2\left( (C-1)\mathcal{O}_{ij} + \eps \right)$.
Write
\begin{align*}
\tilde{w^*_{ij}} = \Sigma^{1/2} w^*_{ij} = a \sigma w + b u.    
\end{align*}
If we choose $ 2 (C-1)\mathcal{O}_{ij} \le \sigma \le  2\left( (C-1)\mathcal{O}_{ij} + \eps\right)$, because its smaller than $b/a$ then $b/\sqrt{b^2+(a\sigma)^2} \ge 1/\sqrt{2}$, which implies that
$\tilde \theta = \theta(\tilde{w^*_{ij}},w) \in (\pi/4,\pi/2)$.
Define the wedge 
\begin{align*}
A(w):=\{z\in\R^d:(w\cdot z)(\tilde w^*_{ij}\cdot z)\le 0\}.    
\end{align*}
This implies that 
\begin{align}\label{eq error halfspace}
    \Pr_{z\sim \cN(0,I)}[A(w)] = \tilde \theta /\pi \ge 1/4. 
\end{align}
Now notice that 
 \begin{align*}
    A(w) = \{z \mid (w\cdot z)(\Tilde{w}^*_{ij}\cdot z) \le 0 \} & = (\{z \mid (w\cdot z)(\Tilde{w^*_{ij}}\cdot z) \le 0 \} \cap \{z \mid y^*(z) \in  \{i,j\}\}) \\
     & \cup (\{z \mid (w\cdot z)(\Tilde{w}^*_{ij}\cdot z) \le 0 \} \cap \{z \mid y^*(z) \in [k] \setminus \{i,j\} \}).
 \end{align*}
Without loss of generality, we assume that $[k] \setminus \{i,j\} = \{k\}$. Notice that $y^*(z) = k$ implies that $w^*_{ki} \cdot \Sigma^{1/2} z \ge 0$, which is equivalent to $\tilde w^*_{ki} \cdot z \ge 0$.

Observe that $A(w)$ is centrally symmetric. If $z \in A(w)$, then $-z \in A(w)$.
, since both
$w\cdot z$ and $\tilde w_{ij}^*\cdot z$ change sign under $z\mapsto -z$, and hence their
product is unchanged. Since \(z\sim N(0,I)\) is symmetric under $z\mapsto -z$ and \(A=-A\), the two sets \(A\cap\{\tilde w_{ki}^*\cdot z\ge 0\}\) and \(A\cap\{\tilde w_{ki}^*\cdot z\le 0\}\) have the same Gaussian measure.

 Therefore, 
 \begin{align*}
 \Pr_{z \sim \cN(0,I)}[\{z \mid (w\cdot z)(\Tilde{w^*_{ij}}\cdot z) \le 0 \} \cap \{z \mid y^*(z) = k\}] &\le \Pr_{z \sim \cN(0,I)}[\{z \mid (w\cdot z)(\Tilde{w}^*_{ij}\cdot z) \le 0, \Tilde{w}^*_{ki}\cdot z \ge 0 \}\\ &= \Pr_{z \sim \cN(0,I)}[\{z \mid (w\cdot z)(\Tilde{w}^*_{ki}\cdot z) \le 0\}]/2\;.    
 \end{align*}
Combining with \eqref{eq error halfspace}, we have  
\begin{align*}
    \Pr_{z \sim \cN(0,I)}[\{z \mid (w\cdot z)(\Tilde{w}^*_{ij}\cdot z) \le 0 \} \cap \{z \mid y^*(z) \in\{ i,j\}\}] \ge 1/8.
\end{align*}
Since
\begin{align*}
 &   \Pr_{z \sim \cN(0,I)}[\widetilde{E}_{ij}(w)]\\ & = \Pr_{z \sim \cN(0,I)} [\{(z,\tilde y) \mid \tilde y = i, w\cdot z \le 0\} \cup \{(z,\tilde y) \mid \tilde y = j, w\cdot z \ge 0\}] \\
    & \geq \Pr_{z \sim \cN(0,I)}[\{z \mid (w\cdot z)(\Tilde{w}^*_{ij}\cdot z) \le 0 \} \cap \{z \mid y^*(z) \in  \{i,j\}\}] - \Pr_{z \sim \cN(0,I)}[y^*(z) \in  \{i,j\}, \tilde y \neq y^*] \\
    & \ge 1/8 - \Pr_{z \sim \cN(0,I)}[y^*(z) \in  \{i,j\}, \tilde y \neq y^*].
\end{align*}
We next upper bound $\Pr_{z \sim \cN(0,I)}[y^*(z) \in  \{i,j\}, \tilde y \neq y^*]$. By the definition of $y^*$ and $\tilde y$,
we have 
\begin{align*}
    \Pr_{(z,\tilde y) \sim D'(w)}[y^*(z) \in  \{i,j\}, \tilde y \neq y^*] = \Pr_{x \sim \cN(0,\Sigma)}[f^*(x) \in \{i,j\}, f^*(x) \neq y(x)]. 
\end{align*}
Notice that a sample $x$ from $\cN(0,\Sigma)$ can be obtained from the conditional distribution of the rejection sampling procedure described in \Cref{lem:rejection-sampling}.
Also by \Cref{lem:rejection-sampling} the total probability of acceptance is $\sigma$.
This implies that 
\begin{align*}
    \Pr_{x \sim \cN(0,\Sigma)}[f^*(x) \in \{i,j\}, f^*(x) \neq y(x)] &\le \Pr_{x \sim \cN(0,I)}[f^*(x) \in \{i,j\}, f^*(x) \neq y(x)]/\sigma \\
    & \le 
    \sum_{a \in \{i,j\}, b \neq a} \opt_{a,b}/\sigma \le \mathcal{O}_{ij}/\sigma \\
    & \le 1/(2(C-1)) \le 1/40,
\end{align*}
for a large enough constant $C>0$.
Here, the last inequality follows the choice of $\sigma \ge 2 (C-1)\mathcal{O}_{ij}$.
This implies
$\Pr_{z \sim \cN(0,I)}[\widetilde{E}_{ij}(w)] \ge 1/10$ and concludes the proof of \Cref{cl error prob}.
\end{proof}

We now make the choice of the localization parameter $\sigma$ algorithmic. The algorithm runs the update procedure for all guesses $\sigma \in \{\eps,2\eps,\dots,\lceil 1/\eps\rceil\eps\}\cap(0,1]$. Since \Cref{cl error prob} holds for every $\sigma$ in an interval of width at least $\eps$, at least one of these guesses satisfies the condition of \Cref{cl error prob}. In the rest of the proof, we fix such a guess $\sigma$ and analyze the corresponding run of the algorithm.

We define the random vector
\begin{align*}
    G_{ij}(w):=\proj_{w^\perp}(z\Ind\{w\cdot z\le 0,  \tilde y=i\}-z\Ind\{w \cdot z \ge 0, \tilde y = j\}).
\end{align*}
Denote by $g_{ij}(w) = \E_{(z,\tilde y)\sim D'(w)}[G_{ij}(w)]$ and $g^*_{ij}(w) = \E_{z\sim \cN(0,I)}[z \Ind\{(w\cdot z)(\tilde w^*_{ij}\cdot z) \le 0, \tilde y=y^*=i\} - z \Ind\{(w\cdot z)(\tilde w^*_{ij}\cdot z) \le 0, \tilde y=y^*=j\}]$.
We next lower bound the correlation between $g_{ij}(w)$ and $w^*_{ij}$.
Since $G_{ij}(w) \perp w$, we know that 
\begin{align*}
    g_{ij}(w) \cdot w^* & = g_{ij}(w) \cdot (a w + b u) = b g_{ij}(w) \cdot u \\&\ge b g^*_{ij}(w) \cdot u - b \E_{(z,\tilde y)\sim D'(w)}[\abs{z\cdot u}\Ind\{y^*\neq \tilde y, y^* \in \{i,j\}\;\mathrm{or}\; \tilde y \in \{i,j\}\}]\;.
\end{align*}
Notice that 
\begin{align*}
    \Pr_{(z,\tilde y) \sim D'(w)}[y^*\neq \tilde y, y^* \in \{i,j\}\;\mathrm{or}\; \tilde y \in \{i,j\}] &= \Pr_{x\sim \cN(0,\Sigma)}[f^*\neq y, f^* \in \{i,j\}\;\mathrm{or}\;  y \in \{i,j\}]\\& \le \mathcal{O}_{ij}/\sigma = O(1/C).
\end{align*}
By \Cref{fact:simpleperturbation}, we have 
\begin{align*}
    \E_{(z,\tilde y)\sim D'(w)}[\abs{z\cdot u}\Ind\{y^*\neq \tilde y, y^* \in \{i,j\}\;\mathrm{or}\; \tilde y \in \{i,j\}\}] =\tilde{O}(1/C).
\end{align*}
Recall that $\tilde w^*_{ij} = a \sigma w + bu$, which implies that for every $z$ such that $y^*(z) = i$, if $(w\cdot z)(\tilde w ^*_{ij}\cdot z) \le 0$, then we have $u \cdot z \ge 0$. 
Similarly, if $(w\cdot z)(\tilde w ^*_{ij}\cdot z) \le 0$, and $y^*(z) = j$ then we have $u \cdot z \le 0$.

By \Cref{lm convex body}, and Jensen's inequality, we have 
\begin{align*}
    &g^*_{ij}(w) \cdot u \\& {\ge} \sqrt{\frac{\pi}{2}}\left( \Pr_{(z,\tilde y)\sim D'(w)}[(w\cdot z)(\tilde w^*_{ij}\cdot z) {\le} 0, \tilde y{=}y^*{=}i]^2 {+} \Pr_{(z,\tilde y)\sim D'(w)}[(w\cdot z)(\tilde w^*_{ij}\cdot z) {\le} 0, \tilde y{=}y^*{=}j]^2\right) \\
    & \ge \Omega \left(\Pr_{(z,\tilde y)\sim D'(w)}[(w\cdot z)(\tilde w^*_{ij}\cdot z) \le 0, \tilde y=y^* \in \{i,j\}]^2\right)\;.
\end{align*}
By \Cref{cl error prob}, we have that $\Pr_{(z,\tilde y)\sim D'(w)}[\widetilde{E}_{ij}(w)] \ge 1/10$. Since
\begin{align*}
    \Pr_{(z,\tilde y)\sim D'(w)}[\widetilde{E}_{ij}(w)] \le  \Pr_{(z,\tilde y)\sim D'(w)}[(w\cdot z)(\tilde w^*_{ij}\cdot z) {\le} 0, \tilde y{=}y^* {\in} \{i,j\}] + \Pr_{(z,\tilde y)\sim D'(w)}[ \tilde y \in \{i,j\}, y^* {\neq} \tilde y],
\end{align*}
and 
\begin{align*}
\Pr_{(z,\tilde y)\sim D'(w)}[ \tilde y \in \{i,j\}, y^* \neq \tilde y] =  \Pr_{x \sim \cN(0,\Sigma)}[y(x) \in \{i,j\}, f^*(x) \neq y(x)]  \le \mathcal{O}_{ij}/\sigma = O(1/C).
\end{align*}
If $C$ is sufficiently large this implies that 
\begin{align*}
    \Pr_{(z,\tilde y)\sim D'(w)}[(w\cdot z)(\tilde w^*_{ij}\cdot z) \le 0, \tilde y=y^* \in \{i,j\}] \ge 1/20
\end{align*}
and 
$g^*_{ij}(w) \cdot u \ge \Omega(1)$.
This implies that when $\err_{ij}(h_w) \ge C \mathcal{O}_{ij} + \eps$, we have 
\begin{align*}
    g_{ij}(w) \cdot w^* \ge \Omega(b).
\end{align*}
Notice that if $b<\eps$, then we have $\Pr_{(x,y)\sim D}[(w\cdot x)(w^*_{ij}\cdot x)\leq 0] \le \eps/2$, which implies that $\err_{ij}(h_w)\leq \mathcal{O}_{ij} + \eps$ is small.
This implies that $g_{ij}(w) \cdot w^* \ge \Omega(\eps)$.
Furthermore, by \Cref{lm gaussian event mean}, we have $\norm{g_{ij}(w)} \le O(1)$.
Notice that by Hoeffding's inequality, for every $\delta>0$, by taking $d\poly(1/\eps)\log(1/\delta)$ examples from $D'(w)$, which can be done via a rejection sampling from $D$ with an acceptance rate at least $\eps$,
the empirical mean estimation $\widehat{g}_{ij}(w)$ satisfies $\norm{g_{ij}(w) - \widehat{g}_{ij}(w)}< \eps/100$.

We now show the correctness of \Cref{alg pairwise localization}, let $T = O(1/\eps^2)$.
Suppose that for $t \in [T]$, we have $\err_{ij}(h^t_{ij}) = \err_{ij}(\sign(w^t_{ij}\cdot x))>C \mathcal{O}_{ij}+ \eps$.
Then, for $t \in [T]$, we have 
$g_{ij}(w_{ij}^t)\cdot w^*_{ij}>\Omega(\eps)$, $\norm{g_{ij}(w_{ij}^t)}\le O(1)$
and with probability at least $1-\delta/T$, $\norm{g_{ij}(w^t_{ij}) - \widehat{g}_{ij}(w^t_{ij})}< \eps/100 $.
By \Cref{lm gradient descent}, we have 
\begin{align*}
    \norm{w^*_{ij}-w^{t+1}_{ij}}^2 \le \norm{w^*_{ij}-w^{t}_{ij}}^2 - \Omega(\eps^2).
\end{align*}
Since $\norm{w^0_{ij}-w^*_{ij}} \le 2$, this implies that after $T$ rounds, $\norm{w^*_{ij}-w^{t+1}_{ij}}^2 <0$, which gives contradiction.
This implies that there must be some $t \in [T]$ such that $\err_{ij}(\sign(w^t_{ij}\cdot x))\le O(\mathcal{O}_{ij}) + \eps$.

Finally, for each \(\sigma \in S\) and \(t \in [T]\), let
$\widehat{\err}_{ij}^{\,\sigma,t}$
denote the empirical estimate of \(\err_{ij}(\sign(w_{ij}^{\sigma,t}\cdot x))\) computed in Line \ref{line: empriical error} using \(N\) fresh samples from \(D\). Since \(|S| \le O(1/\eps)\) and \(T=O(1/\eps^2)\), by taking
$
N=\poly(1/\eps)\log(|S|T/\delta)
$, by Hoeffding's inequality together with a union bound implies that, with probability at least \(1-\delta\), simultaneously for all \(\sigma \in S\) and \(t \in [T]\),
\[
\Bigl|\widehat{\err}_{ij}^{\,\sigma,t}-\err_{ij}(\sign(w_{ij}^{\sigma,t}\cdot x))\Bigr|\le \eps.
\]
Now fix a value \(\sigma^\star \in S\) satisfying the condition from the above argument. By the above argument, there exists some \(t^\star \in [T]\) such that
\[
\err_{ij}(\sign(w_{ij}^{\sigma^\star,t^\star}\cdot x)) \le O(\mathcal{O}_{ij})+\eps.
\]
Hence, by concentration,
\[
\widehat{\err}_{ij}^{\,\sigma^\star,t^\star}
\le O(\mathcal{O}_{ij})+2\eps.
\]
Since \(h_{ij}^{\sigma^\star}\) is chosen in Line \ref{line: t-select} to minimize the empirical error over \(t\in[T]\), we have
\[
\widehat{\err}_{ij}(h_{ij}^{\sigma^\star})
\le \widehat{\err}_{ij}^{\,\sigma^\star,t^\star}
\le O(\mathcal{O}_{ij})+\eps.
\]
Finally, the algorithm outputs in Line \ref{line: sigma-select} the hypothesis corresponding to the value of \(\sigma\in S\) with minimum empirical error. Therefore,
\[
\widehat{\err}_{ij}(\widehat h_{ij})
\le \widehat{\err}_{ij}(h_{ij}^{\sigma^\star})
\le O(\mathcal{O}_{ij})+2\eps,
\]
and another application of the same concentration bound yields
\[
\err_{ij}(\widehat h_{ij}) \le O(\mathcal{O}_{ij})+3\eps.
\]
Setting $\eps$ multiplicatively less than a constant which we implicitly do in the algorithm completes the proof of \Cref{lm localization}.

\end{proof}

\subsection{Proof of \Cref{th k=3}}\label{app proof k=3}

\begin{theorem}[Restatement of \Cref{th k=3}]
Consider the problem of agnostic learning MLC over $\R^d$ with 3 labels. There is an algorithm that draws $m=d\poly(\log(1/\delta)/\eps)$ samples, runs in $\poly(m)$ time, and outputs a hypothesis with error $O(\opt) +\eps$ with probability at least $1-\delta$.    
\end{theorem}

\begin{proof}[Proof of \Cref{th k=3}]
    By \Cref{lm localization}, for $k=3$ we have that for every $i \neq j \in [k]$, we obtain a hypothesis $\widehat{h}_{ij} = \sign(w_{ij}\cdot x)$ such that $\err_{ij}(\widehat{h}_{ij}) \le C \mathcal{O}_{ij} + \eps/k^2$ through \Cref{alg pairwise localization}. 
    By using \Cref{alg pairwise localization} as a subroutine in \Cref{alg pairwise to pseudo}, the pseudo $k$-multiclass linear classifier $h_W$ obtained from $\{\widehat{h}_{ij}\}$ satisfies
    \begin{align*}
    \err(h_W)  & = \sum_{j = 1}^k \Pr_{(x,y)\sim D}[h_W(x) \neq y, y=j] \le \sum_{j = 1}^k \sum_{i \neq j} \Pr_{(x,y)\sim D}[w_{ji}\cdot x \le 0, y=j] \\
    & = \frac{1}{2} \sum_{(i,j):i \neq j} (\Pr_{(x,y)\sim D}[w_{ij}\cdot x <0, y=i ]+\Pr_{(x,y)\sim D}[w_{ij}\cdot x >0, y=j ]) \\
    &\le \frac{1}{2} \sum_{(i,j):i \neq j} \left(C\mathcal{O}_{ij}+\eps/k^2\right)  \leq \frac{C}{2} \sum_{(i,j):i \neq j} \left(\sum_{r \neq i} (\opt_{ri} + \opt_{ir}) + \sum_{r \neq j} (\opt_{rj} + \opt_{jr}) \right) + \eps\\
    & \le 2Ck\opt + \eps.
\end{align*}
Here, in the last inequality, we use the fact that $\sum_{(i,j): i \neq j} \sum_{r \neq i}\opt_{ri} \le k \opt$.
\end{proof}

\subsection{Proof of \Cref{lm error count}}\label{app proof error count}

\begin{lemma}[Restatement of \Cref{lm error count}]
    Let $f^*(x) = \argmax_{i \in [k]} (w^*_i \cdot x) : \R^d \to [k]$ be a multiclass linear classifier and let $h_w(x) = \sign(w \cdot x)$ be a halfspace. For $i,j \in [k], i \neq j$, let $\theta = \theta(w^*_{ij},w)$. 
    Write $w = a w^*_{ij} + b u$, with $a,b \in (0,1)$, $a^2+b^2 = 1$, $u \perp w^*_{ij}$ and $u \in \s^{d-1}$.
    Then, it holds
    \begin{align*}
        \abs{\err_{i,j}(h_w,f^*)-\int_{z \in B_{ij}} \phi_{d-1}(z) d z\int_0^ {\tan\theta \abs{u\cdot z} }\phi_1(t) dt }  \le \tilde O(\sqrt{\log k}\tan^2\theta/\Phi_{ij}).
    \end{align*}
Furthermore, for every $\theta_0<\theta$, it holds 
\begin{align*}
       \err_{i,j}(h_w,f^*)\ge \int_{z \in B_{ij}} \phi_{d-1}(z) d z\int_0^ {\tan\theta_0 \abs{u\cdot z} }\phi_1(t) dt   - \tilde O(\sqrt{\log k}\tan^2\theta_0/\Phi_{ij}),
\end{align*}
where $\Phi_{ij} = \min\left\{|\tan \theta^*_{ij}|,1 \right\}$.
\end{lemma}

\begin{proof}[Proof of \Cref{lm error count}]
Recall that for $i,j \in [k]$, $w^*_{ij} = (w^*_i-w^*_j)/\|w^*_i-w^*_j\|$. Denote by 
\begin{align*}
H_{ij} = \{x \mid w^*_{ij}\cdot x=0\},    
\end{align*}
 the subspace spanned by examples $x \perp w^*_{ij}$.
For a Gaussian example $x \sim \cN(0,I)$, we write
\begin{align*}
    x = z + t w^*_{ij},
\end{align*}
where $z \in H_{ij}$ sampled from a $d-1$ dimensional standard Gaussian and $t \sim \cN(0,1)$.
Notice that $(w^*_{ij}\cdot x) = t \sim \cN(0,1)$, which implies that $\sign(w^*_{ij}\cdot x) = \sign(t)$.
This implies that $(w^*_{ij}\cdot x)(w \cdot x)< 0$ if and only if $\sign(w \cdot x) \neq \sign(t)$.
Notice that $(w\cdot x) = (at + bu\cdot z)$. Without loss of generality, we assume that $u\cdot z \le 0$.
This implies that $(w^*_{ij}\cdot x)(w \cdot x)< 0$ if and only if $t \in [0,\tan\theta \abs{u\cdot z}]$.
This implies that
\begin{align*}
   & \Pr_{x\sim \cN(0,I)}[(w^*_{ij}\cdot x)(w\cdot x)<0, f^*(x) \in \{i,j\}]\\
    & = \int_{z \in H_{ij}} \phi_{d-1}(z) d z\int_0^ {\tan\theta \abs{u\cdot z} } \Ind\{f^*(x) \in \{i,j\}\} \phi_1(t) dt \\
    & = \int_{z \in B_{ij}} \phi_{d-1}(z) d z\int_0^ {\tan\theta \abs{u\cdot z} } \Ind\{f^*(x) \in \{i,j\}\} \phi_1(t) dt\\&\quad + \int_{z \in H_{ij} \setminus B_{ij}} \phi_{d-1}(z) d z\int_0^ {\tan\theta \abs{u\cdot z} } \Ind\{f^*(x) \in \{i,j\}\} \phi_1(t) dt.
\end{align*}

We first show that the second term above is small.

To see this, we fix $x = z + t w^*_{ij}$, with $u \cdot z = -r, z \in H_{ij}\setminus B_{ij}$ and $t \in [0,r \tan \theta]$, which implies that $(w^*_{ij}\cdot x)(w\cdot x)<0$. 
For every $l \in [k]\setminus \{i,j\}$, we write  $w^*_{i\ell} = a_{i\ell} w^*_{ij} + b_{i\ell} u_{i\ell}$, with $a_{i\ell}^2+b_{i\ell}^2=1$, $u_{i\ell} \perp w^*_{ij}$, $u_{i\ell} \in \s^{d-1}$.
\begin{align*}
    w^*_{i\ell} \cdot x = (a_{i\ell} w^*_{ij} + b_{i\ell} u_{i\ell}) \cdot (z + t w^*_{ij}) = b_{i\ell}(u_{i\ell}\cdot z) + a_{i\ell} t.
\end{align*}
Since $\abs{u_{i\ell}\cdot z}$ is the distance between $z$ and the halfspace $\{z\in H_{ij}\mid u_{i\ell}\cdot z = 0 \}$
, 
and $z \in H_{ij}\setminus B_{ij}$, we know that $w^*_{\ell i} \cdot x>0$, if $\abs{t}<\abs{b_{i\ell}/a_{i\ell}} \abs{u_{i\ell}\cdot z}$. And when this happens, it must be the case that $f^*(x) \not \in \{i,j\}$.
This implies that for every $z \in H_{ij}\setminus B_{ij}$, with such that $\abs{u_{s\ell}\cdot z} \le r \tan\theta/ \Phi_{ij}$, for some $s \in \{i,j\}$ and $\ell \in [k] \setminus \{i,j\}$, we have 
\begin{align*}
    \int_0^ {\tan\theta \abs{u\cdot z} } \Ind\{f^*(x) \in \{i,j\}\} \phi_1(t) dt = 0.
\end{align*}
Denote by $Z_{ij}(r)$ the set of $z \in H_{ij} \setminus B_{ij}$ that does not satisfy the property above.
Therefore, it suffices to look at the probability mass of points with $z$ that is in $Z_{ij}(r)$.
Define $S_{ij}(r):= \{z \in H_{ij} \mid u_{s\ell}\cdot z \le 2r \tan\theta/ \Phi_{ij}, s \in \{i,j\}$ and $\ell \in [k] \setminus \{i,j\}\}$.
We notice that for all $z \in Z_{ij}(r)$, we have $z \in S_{ij}(r)$ and $\mathrm{dis}(z,\partial S_{ij}(r)) \le 2r \tan\theta/ \Phi_{ij}.$
This implies that 
\begin{align}\label{eq boundary mass}
    \int_{Z_{ij}(r)} \phi_{d-1}(z) dz \le \int_{z \in S_{ij}(r), \mathrm{dis}(z,\partial S_{ij}(r)) \le 2r \tan\theta/ \Phi_{ij}} \phi_{d-1}(z) dz \le O(\sqrt{\log k} r\tan\theta/\Phi_{ij})
\end{align}

Let
\[
I_{\mathrm{out}}
:=
\int_{z \in H_{ij} \setminus B_{ij}} \phi_{d-1}(z)\, dz
\int_0^{\tan\theta \abs{u\cdot z} } \Ind\{f^*(x) \in \{i,j\}\} \phi_1(t)\, dt .
\]
Then
\begin{align*}
I_{\mathrm{out}}
&=
\int_{z \in H_{ij} \setminus B_{ij}} \phi_{d-1}(z)\, dz
\int_0^{\tan\theta \abs{u\cdot z} } \Ind\{f^*(x) \in \{i,j\}\} \phi_1(t)\, dt
 \\
&\le
\int_{\substack{z \in H_{ij} \setminus B_{ij}\\ \abs{u\cdot z}\le r}} \phi_{d-1}(z)\, dz
\int_0^{\tan\theta \abs{u\cdot z} } \Ind\{f^*(x) \in \{i,j\}\} \phi_1(t)\, dt
\;+\;
\Pr[\abs{u\cdot z}>r]
 \\
&\le
\int_{\substack{z \in H_{ij} \setminus B_{ij}\\ \abs{u\cdot z}\le r\\ z \in Z_{ij}(r)}}
\phi_{d-1}(z)\, dz
\int_0^{r\tan\theta} \phi_1(t)\, dt
\;+\;
\exp(-r^2/2)/r
 \\
&\le
O(\sqrt{\log k})(r\tan\theta)^2/\Phi_{ij}
\;+\;
\exp(-r^2/2)/r \;.
\end{align*}
In the first inequality, we split according to whether $\abs{u\cdot z}\le r$ or $\abs{u\cdot z}>r$, and bound the second part by \Cref{fact bias}.
In the second inequality, we use the argument above: if $z\in H_{ij}\setminus B_{ij}$ is not in $Z_{ij}(r)$, 
then for every $t\in[0,r\tan\theta]$ we have $f^*(z+t w^*_{ij})\notin\{i,j\}$, so such points do not contribute to the integral.
In the last inequality, we use \eqref{eq boundary mass}.

By choosing $r = 2\sqrt{\log (1/\tan\theta)}$, we obtain that 
\begin{align*}
    \int_{z \in H_{ij} \setminus B_{ij}} \phi_{d-1}(z) d z\int_0^ {\tan\theta \abs{u\cdot z} } \Ind\{f^*(x) \in \{i,j\}\} \phi_1(t) dt \le O(\sqrt{\log k}\tan^2\theta \log(1/\tan\theta)/\Phi_{ij}).
\end{align*}
We next show that 
\begin{align*}
   \abs{\int_{z \in B_{ij}} \phi_{d-1}(z) d z\int_0^ {\tan\theta \abs{u\cdot z} } (\Ind\{f^*(x) \in \{i,j\}\}-1) \phi_1(t) dt 
   } \le \tilde O(\sqrt{\log k}\tan^2\theta/\Phi_{ij}).
\end{align*}
Similarly, by symmetry, we will fix
$x = z + t w^*_{ij}$, with $u \cdot z = -r, z \in  B_{ij}$ and $t \in [0,r \tan \theta]$, which implies that $(w^*_{ij}\cdot x)(w\cdot x)<0$.
Notice that if $\mathrm{dis}(z,\partial B_{ij})>r\tan \theta/\Phi_{ij}$, then it must be the case that $w^*_{il}\cdot x>0$ and $w^*_{jl}\cdot x>0$, meaning $f^*(x) \in \{i,j\}$.
This implies that 
\begin{align*}
    \int_0^ {\tan\theta \abs{u\cdot z} } \Ind\{f^*(x) \in \{i,j\}\} \phi_1(t) dt = \int_0^ {\tan\theta \abs{u\cdot z} }\phi_1(t) dt
\end{align*}
and 
\begin{align*}
   & \abs{\int_{z \in B_{ij}} \phi_{d-1}(z) d z\int_0^ {\tan\theta \abs{u\cdot z} } \Ind\{f^*(x) \in \{i,j\}\} \phi_1(t) dt - \int_{z \in B_{ij}} \phi_{d-1}(z) d z\int_0^ {\tan\theta \abs{u\cdot z} }\phi_1(t) dt} \\
   \le  &   \int_{z \in B_{ij},\mathrm{dis}(z,\partial B_{ij})<\abs{u\cdot z}\tan \theta/\Phi_{ij}} \phi_{d-1}(z) d z\int_0^ {\tan\theta \abs{u\cdot z}}\phi_1(t) dt \\
   \le & \int_{z \in B_{ij},\mathrm{dis}(z,\partial B_{ij})<r\tan \theta/\Phi_{ij}} \phi_{d-1}(z) d z\int_0^ {r\tan\theta }\phi_1(t) dt + \exp(-r^2/2)/r \\
   \le & O(r\tan\theta\,\Gamma(B_{ij})/\Phi_{ij})\cdot O(r\tan\theta) + \exp(-r^2/2)/r \;.
\end{align*}

By choosing $r = 2\sqrt{\log (1/\tan\theta)}$, we obtain that
\begin{align*}
    \abs{\int_{z \in B_{ij}} \phi_{d-1}(z) d z\int_0^ {\tan\theta \abs{u\cdot z} } (\Ind\{f^*(x) \in \{i,j\}\}-1) \phi_1(t) dt } \le \tilde O(\sqrt{\log k}\tan^2\theta/\Phi_{ij}) \;.
\end{align*}
This implies that 
\begin{align*}
        \abs{\Pr_{x\sim \cN(0,I)}[(w^*_{ij}\cdot x)(w\cdot x)<0, f^*(x) \in \{i,j\}]-\int_{z \in B_{ij}} \phi_{d-1}(z) d z\int_0^ {\tan\theta \abs{u\cdot z} }\phi_1(t) dt }  \le  O(\sqrt{\log k}\tan^2\theta) \;.
\end{align*}
Finally, note that for $\theta \ge \theta_0$, the error at angle $\theta$ is at least the error at angle $\theta_0$, since decreasing the angle from $\theta$ to $\theta_0$ only shrinks the conic disagreement region. Indeed,
\begin{align*}
   & \Pr_{x\sim \cN(0,I)}[(w^*_{ij}\cdot x)(w\cdot x)<0, f^*(x) \in \{i,j\}]\\
    & = \int_{z \in H_{ij}} \phi_{d-1}(z) d z\int_0^ {\tan\theta \abs{u\cdot z} } \Ind\{f^*(x) \in \{i,j\}\} \phi_1(t) dt \\
    & = \int_{z \in H_{ij}} (\Ind\{z\in B_{ij}\} +\Ind\{z\in H_{ij}\setminus B_{ij}\} \phi_{d-1}(z) d z\int_0^ {\tan\theta \abs{u\cdot z} } \Ind\{f^*(x) \in \{i,j\}\} \phi_1(t) dt \\
    & \ge \int_{z \in B_{ij}} \phi_{d-1}(z) d z\int_0^ {\tan\theta_0 \abs{u\cdot z} } \Ind\{f^*(x) \in \{i,j\}\} \phi_1(t) dt \\
    & \ge \int_{z \in B_{ij}} \phi_{d-1}(z) d z\int_0^ {\tan\theta_0 \abs{u\cdot z} }\phi_1(t) dt   - \tilde O(\sqrt{\log k}\tan^2\theta_0/\Phi_{ij}).
\end{align*}
\end{proof}

\subsection{Proof of \Cref{col surface}}\label{app proof surface}
In this section, we give the proof of \Cref{col surface}. We restate \Cref{col surface} as follows.

\begin{lemma}[Restatement of \Cref{col surface}]
    Let $f^*(x) = \argmax_{i \in [k]} (w^*_i \cdot x) : \R^d \to [k]$ be a multiclass linear classifier and let $h_w(x) = \sign(w \cdot x)$ be a halfspace. For $i,j \in [k], i \neq j$, let $\theta = \theta(w^*_{ij},w)$. 
    Then, it holds
    \begin{align*}
        & \Pr_{x\sim \cN(0,I)}[(w^*_{ij}\cdot x)(w\cdot x)<0, f^*(x) \in \{i,j\}] \le 2\sqrt{e} \tan \theta \mathcal{T}_{ij}\sqrt{\log(1/\mathcal{T}_{ij})}+\tilde O(\sqrt{\log k}\tan^2\theta/\Phi_{ij}) \\
        & \Pr_{x\sim \cN(0,I)}[(w^*_{ij}\cdot x)(w\cdot x)<0, f^*(x) \in \{i,j\}] \ge \tan \theta \mathcal{T}_{ij}^2/20 - \tilde O(\sqrt{\log k}\tan^2\theta/\Phi_{ij})
    \end{align*}
In particular, if $\theta \ge \theta_0 = \tilde\Theta  (\mathcal{T}^2_{ij}\Phi_{ij}/\sqrt{\log k})$,
then it always holds 
$\Pr_{x\sim \cN(0,I)}[(w^*_{ij}\cdot x)(w\cdot x)<0, f^*(x) \in \{i,j\}] \ge \tilde \Omega(\mathcal{T}_{ij}^4\Phi_{ij}/\sqrt{\log k})$.
\end{lemma}

\begin{proof}[Proof of \Cref{col surface}]
First we prove the upper bound.
    By \Cref{lm error count}, we have that 
    \begin{align*}
        &\Pr_{x\sim \cN(0,I)}[(w^*_{ij}\cdot x)(w\cdot x)<0, f^*(x) \in \{i,j\}]\\ & \le \int_{z \in B_{ij}} \phi_{d-1}(z) d z\int_0^ {\tan\theta \abs{u\cdot z} }\phi_1(t) dt + \tilde O(\sqrt{\log k}\tan^2\theta/\Phi_{ij}) \\
        & \le \tan\theta \int \abs{u\cdot z} \Ind\{z\in B_{ij}\} \phi_{d-1}(z) dz + \tilde O(\sqrt{\log k}\tan^2\theta/\Phi_{ij}) \\
        & \le 2\sqrt{e}\tan\theta \mathcal{T}_{ij}\sqrt{\log(1/\mathcal{T}_{ij})} + \tilde O(\sqrt{\log k}\tan^2\theta/\Phi_{ij}).
    \end{align*}
Here, the last inequality follows from \Cref{fact:simpleperturbation}.
For the lower bound, we have 
\begin{align*}
        & \Pr_{x\sim \cN(0,I)}[(w^*_{ij}\cdot x)(w\cdot x)<0, f^*(x) \in \{i,j\}] \\
        &\ge \int_{z \in B_{ij}} \phi_{d-1}(z) d z\int_0^ {\tan\theta \abs{u\cdot z} }\phi_1(t) dt - \tilde O(\sqrt{\log k}\tan^2\theta/\Phi_{ij}) \\
        & \ge \int_{z \in B_{ij},\tan\theta \abs{u\cdot z}<1} \phi_{d-1}(z) d z\int_0^ {\tan\theta \abs{u\cdot z} }\phi_1(t) dt - \tilde O(\sqrt{\log k}\tan^2\theta/\Phi_{ij}) \;.
    \end{align*}    
Notice that $\phi_1(t) \ge 1/5$ for $t \in [0,1]$, which implies that if $\tan\theta\abs{u\cdot z} \le 1$, then 
\begin{align*}
    \int_{z \in B_{ij},\tan\theta \abs{u\cdot z}<1} \phi_{d-1}(z) d z\int_0^ {\tan\theta \abs{u\cdot z} }\phi_1(t) dt \ge \tan\theta \int_{z \in B_{ij},\tan\theta \abs{u\cdot z}<1} \abs{u\cdot z} \phi_{d-1}(z) d z/5.
\end{align*}
Moreover, notice that 
\begin{align*}
    \tan\theta \int_{z \in B_{ij},\tan\theta \abs{u\cdot z}\ge 1} \abs{u\cdot z} \phi_{d-1}(z) d z & 
    \le \tan\theta \int_{\abs{u\cdot z}\ge 1/\tan\theta} \abs{u\cdot z} \phi_{d-1}(z) dz\\
    &\le O(\tan\theta \exp(-1/(2\tan^2\theta)/2))\;,
\end{align*}
where the last inequality follows from $u\cdot z\sim N(0,1)$ and the identity
$\int_a^\infty t\phi_1(t)\,dt=\phi_1(a)=O(e^{-a^2/2})$.
The above implies that 
\begin{align*}
        & \Pr_{x\sim \cN(0,I)}[(w^*_{ij}\cdot x)(w\cdot x)<0, f^*(x) \in \{i,j\}] \\
        & \ge \tan\theta \int_{z \in B_{ij}} \abs{u\cdot z} \phi_{d-1}(z) d z/5 - \tan\theta \int_{z \in B_{ij},\tan\theta \abs{u\cdot z} \ge 1} \abs{u\cdot z} \phi_{d-1}(z) d z/5 - \tilde O(\sqrt{\log k}\tan^2\theta/\Phi_{ij}) \\
        & \ge \tan\theta \int_{z \in B_{ij}} \abs{u\cdot z} \phi_{d-1}(z) dz/5- \tilde O(\sqrt{\log k}\tan^2\theta/\Phi_{ij}) \;.
    \end{align*}
By the anti-concentration of the standard normal distribution,
$\Pr_{z\sim \cN(0,I)}[\abs{u\cdot z}<\mathcal{T}_{ij}/2] \le \mathcal{T}_{ij}/2$, which implies that $\Pr_{z\sim \cN(0,I)}[\abs{u\cdot z}<\mathcal{T}_{ij}/2 \mid B_{ij}] \le 1/2$.
This gives,
\begin{align*}
    \tan\theta \int_{z \in B_{ij}} \abs{u\cdot z} \phi_{d-1}(z) \ge \tan\theta \mathcal{T}_{ij} \E_{z \sim \cN(0,I)}[\abs{u\cdot z} \mid B_{ij}] \ge \tan\theta \mathcal{T}^2_{ij}/4.
\end{align*}
Thus, we have 
\begin{align*}
    \Pr_{x\sim \cN(0,I)}[(w^*_{ij}\cdot x)(w\cdot x)<0, f^*(x) \in \{i,j\}] \ge \tan \theta \mathcal{T}_{ij}^2/20 - \tilde O(\sqrt{\log k}\tan^2\theta/\Phi_{ij}).
\end{align*}

In particular, by \Cref{lm error count}, we have 
\begin{align*}
 & \Pr_{x\sim \cN(0,I)}[(w^*_{ij}\cdot x)(w\cdot x)<0, f^*(x) \in \{i,j\}]\\ & \ge \int_{z \in B_{ij}} \phi_{d-1}(z) d z\int_0^ {\tan\theta_0 \abs{u\cdot z} }\phi_1(t) dt   - \tilde O(\sqrt{\log k}\tan^2\theta_0/\Phi_{ij}) \\
  & \ge \tan \theta_0 \mathcal{T}_{ij}^2/20 - \tilde O(\sqrt{\log k}\tan^2\theta_0/\Phi_{ij}) \\
  &\ge \Omega(\mathcal{T}_{ij}^4\Phi_{ij}/\sqrt{\log k}).
\end{align*}

\end{proof}

\subsection{Proof of \Cref{lm localization surface}}\label{app proof localization surface}

Before presenting the proof of \Cref{lm localization surface}, we give the following structural lemma to show that after localization, the critical angle of a multiclass linear classifier will not change too much.

\begin{lemma}\label{lm critical angle}
    Let $u,v \in \s^{d-1}$ be two unit vectors and let $\theta$ be the acute angle between $u,v$. Let $w \in \s^{d-1}$ such that $\alpha:=\sin \theta(u,w) \le 1/(16C)$, where $C>1$. Let $\Sigma = I+(\sigma^2-1)ww^\top$, where $\sigma = C\alpha$. Let $\tilde \theta$ be the acute angle between $\Sigma^{1/2}u,\Sigma^{1/2}v$. Then we have $\min\{\tan \tilde \theta,1\} \ge \Omega(\min\{\tan \theta,1\}/C)$.
\end{lemma}

\begin{proof}[Proof of \Cref{lm critical angle}]
    For every $x \in \s^{d-1}$, we write $x = \cos\theta(w,x) w + \sin \theta(w,x) z$, where $z \in \s^{d-1}, z \perp w$. 
    Define the map $F(x):= \Sigma^{1/2}x/\norm{\Sigma^{1/2}x}$.
    That is to say
    \begin{align*}
        F(x) = \cos\theta'(w,x) w + \sin \theta'(w,x) z,
    \end{align*}
     with $\tan \theta'(w,x) = \tan \theta(w,x)/\sigma.$
    Define 
    \begin{align*}
        S=\{x \in \s^{d-1} \mid \abs{\sin\theta(x,w)} \in [\alpha/4,4\alpha]\}.
    \end{align*}
    We first show that for any $x^{(1)},x^{(2)} \in S$, we have $\sin(\theta(x^{(1)},x^{(2)})/2) \le O(\sigma) \sin(\theta(F(x^{(1)}),F(x^{(2)}))/2)$

    Let $\theta^{(i)} = \theta(w,x^{(i)})$ and $\tilde\theta^{(i)} = \theta(w,F(x^{(i)}))$, for $i=1,2$. Write $x^{(i)} = \cos \theta^{(i)} w + \sin \theta^{(i)} z^{(i)}$, where $z^{(i)} \perp w$. Then, 
    we have 
    \begin{align*}
        \cos \theta(x^{(1)},x^{(2)})  = \cos \theta^{(1)}\cos\theta^{(2)} + \sin \theta^{(1)}\sin\theta^{(2)}(z^{(1)}\cdot z^{(2)}).       
    \end{align*}
This implies that 
\begin{align*}
    2\sin^2(\theta(x^{(1)},x^{(2)})/2)&= 1- \cos \theta(x^{(1)},x^{(2)}) = 1-\cos(\theta^{(1)}-\theta^{(2)}) + \sin\theta^{(1)}\sin\theta^{(2)}(1-(z^{(1)}\cdot z^{(2)})) \\
    & = 2\sin^2((\theta^{(1)}-\theta^{(2)})/2) + \sin\theta^{(1)}\sin\theta^{(2)}(1-(z^{(1)}\cdot z^{(2)})) \\
    & \le (\theta^{(1)}-\theta^{(2)})^2/2 + \sin\theta^{(1)}\sin\theta^{(2)}(1-(z^{(1)}\cdot z^{(2)}))
\end{align*}
    and 
\begin{align*}
    2\sin^2(\theta(F(x^{(1)}),F(x^{(2)}))/2) & = 1- \cos \theta(F(x^{(1)}),F(x^{(2)}))\\
    &= 1-\cos(\tilde\theta^{(1)}-\tilde\theta^{(2)}) + \sin\tilde\theta^{(1)}\sin\tilde\theta^{(2)}(1-(z^{(1)}\cdot z^{(2)})) \\
    & = 2\sin^2((\tilde\theta^{(1)}-\tilde\theta^{(2)})/2) + \sin\tilde\theta^{(1)}\sin\tilde\theta^{(2)}(1-(z^{(1)}\cdot z^{(2)})) \\
    & \ge (\tilde\theta^{(1)}-\tilde\theta^{(2)})^2/5 + \sin\tilde\theta^{(1)}\sin\tilde\theta^{(2)}(1-(z^{(1)}\cdot z^{(2)})) \;.
\end{align*}
We first show that $(\theta^{(1)}-\theta^{(2)})^2 \le O(\alpha^2)(\tilde\theta^{(1)}-\tilde\theta^{(2)})^2.$
Notice that the map $G(\theta(x,w)) = \theta(F(x),w)$ is a differentiable and monotone map in $\theta(x,w)$.
By the mean value theorem, there is some $\theta^* \in (\theta^{(1)},\theta^{(2)})$ such that 
\begin{align*}
    \tilde\theta^{(1)}-\tilde\theta^{(2)} = G'(\theta^*)(\theta^{(1)}-\theta^{(2)}).
\end{align*}
So, it suffices to bound $\abs{G'(\theta^*)}$.
Recall that $\tan \theta(F(x),w) = \tan \theta(x,w)$, which implies that $G(\theta) = \arctan(\tan \theta/\sigma)$.
Thus, we have 
\begin{align*}
G'(\theta) = \sigma/(\sigma^2\cos^2\theta+\sin^2\theta) \ge \sigma/(\sigma^2 + \sin^2\theta).    
\end{align*}
Since $\sin \theta \in [\alpha/4,4\alpha]$ and $\sigma = C\alpha$, we know that $G'(\theta) \ge 1/(16C\alpha)$.
This implies that $(\theta^{(1)}-\theta^{(2)})^2 \le O(C^2\alpha^2)(\tilde\theta^{(1)}-\tilde\theta^{(2)})^2.$

On the other hand, we have 
\begin{align*}
    \sin\tilde\theta^{(1)}\sin\tilde\theta^{(2)} & = \cos \tilde\theta^{(1)} \cos \tilde\theta^{(2)} \tan \tilde\theta^{(1)} \tan \tilde\theta^{(2)} = \cos \tilde\theta^{(1)} \cos \tilde\theta^{(2)} \tan \theta^{(1)} \tan \theta^{(2)}/\sigma^2\\
    & \ge \Omega(\sin\theta^{(1)}\sin\theta^{(2)}/\sigma^2) \;.
\end{align*}
Together we obtain
\begin{align}\label{eq map}
    \sin(\theta(x^{(1)},x^{(2)})/2) \le O(\sigma) \sin(\theta(F(x^{(1)}),F(x^{(2)}))/2) \;.
\end{align}
We now use \eqref{eq map} to show \Cref{lm critical angle}. 
Since $\sin\theta(u,w) = \alpha$, we know that if $\sin\theta(u,v) \le \alpha/16$, then it must be the case that $v \in S$. This implies that 
\begin{align*}
   \tan(\theta(F(u),F(v)))\ge \sin(\theta(F(u),F(v))/2) \ge \Omega(\sigma^{-1}) \sin(\theta(u,v)/2) \ge \Omega(\tan\theta(u,v)) \;.
\end{align*}
We next consider the case when $\sin\theta(u,v) \ge \alpha/16$.
Recall that the map $G(\theta(x,w)) = \theta(F(x),w)$ is a differentiable and monotone map in $\theta(x,w)$. 
This implies that for $x \in S$, the image of $G$ is $[G(\arcsin(\alpha/4)),G(\arcsin(4\alpha))]$.
Since $\sin \theta(u,w) = \alpha$, if 
\begin{align*}
\sin\theta(F(u),F(v)) \le \sin( \min\{G(\arcsin(4\alpha)-G(\arcsin(\alpha),G(\arcsin(\alpha))-G(\arcsin(\alpha/4)\}/4),   
\end{align*}
 then it must be the case that $v \in S$.
In this case, we have 
\begin{align*}
\tan(\theta(F(u),F(v)))\ge \sin(\theta(F(u),F(v))/2) \ge \Omega(\sigma^{-1}) \sin(\theta(u,v)/2) \ge \Omega(\alpha/\sigma) = \Omega(1/C).    
\end{align*}
This concludes the proof of \Cref{lm critical angle}.
\end{proof}

\begin{lemma}[Restatement of \Cref{lm localization surface}]   
Let $f^*(x) = \argmax_{i \in [k]} (w^*_i \cdot x) : \R^d \to [k]$ be a multiclass linear classifier and let $h_w(x) = \sign(w \cdot x)$ be a halfspace. For $i,j \in [k], i \neq j$, let $\theta = \theta(w^*_{ij},w)$ with $\sin \theta < \sqrt{\mathcal{T}_{ij}}/C$ for a sufficiently large constant $C>0$. 
    Then we have 
    \begin{align*}
        \Pr_{x\sim \cN(0,\Sigma)}[(w^*_{ij}\cdot x)(w\cdot x)<0, f^*(x) \in \{i,j\}] \ge  \tilde \Omega(\mathcal{T}_{ij}^{4.5}\Phi_{ij}/\sqrt{\log k})\;,
    \end{align*}
where $\Sigma = I+(\sigma^2-1)ww^\top$ for any $\sigma \in [2\sin\theta/\sqrt{\mathcal{T}_{ij}},4\sin\theta/\sqrt{\mathcal{T}_{ij}}]$.
\end{lemma}

\begin{proof}[Proof of \Cref{lm localization surface}]
Let $x \sim \cN(0,\Sigma)$, where $\Sigma^{1/2} w = \sigma w$ and $\Sigma^{1/2} u = u$ for every $u \perp w$.
Write $x = \Sigma^{1/2} z$, with $z \sim \cN(0,I)$.
For each $z$, we associate an artificial label $y^*(z) = f^*(x)$. In other words, 
$y^*(z) = \tilde f(z) = \argmax_{i \in [k]}(\Sigma^{1/2}w^*_i \cdot z)$ is also labeled by a multiclass linear classifier.
Denote by $\tilde w^*_{ij} = \Sigma^{1/2} w^*_{ij}$, $\tilde B_{ij} = \{z \mid \tilde w^*_{ij}\cdot z = 0, \tilde w^*_{ir}\cdot z \ge 0, \tilde w^*_{jr}\cdot z \ge 0, r \in [k] \setminus\{i,j\} \}$ and $\tilde{\mathcal{T}}_{ij}, \tilde\Phi_{ij}$ be the corresponding quantities for the multiclass linear classifier $\tilde f$.
This implies that
\begin{align*}
    \Pr_{x\sim \cN(0,\Sigma)}[(w^*_{ij}\cdot x)(w\cdot x)<0, f^*(x) \in \{i,j\}] = \Pr_{z\sim \cN(0,I)}[(\tilde w^*_{ij}\cdot z)(w\cdot z)<0, y^*(z) \in \{i,j\}].
\end{align*}
Let $\tilde \theta = \theta(w,\tilde w^*_{ij})$. By \Cref{col surface}, we have that 
\begin{align*}
    \Pr_{z\sim \cN(0,I)}[(\tilde w^*_{ij}\cdot z)(w\cdot z)<0, y^*(z) \in \{i,j\}] \ge \tan \tilde \theta \tilde{\mathcal{T}}_{ij}^2/20 - \tilde O(\sqrt{\log k}\tan^2\tilde\theta/\tilde \Phi_{ij}) \;.
\end{align*}
We show that $\tilde{\mathcal{T}}_{ij}$ is comparable to $\mathcal{T}_{ij}$.
Let $H = \{x \mid w^*_{ij} \cdot x = 0\}$. Write $w = a w^*_{ij} + bu$, where $u \in \s^{d-1} \cap H$.
Without loss of generality, we assume that $u = \e_1$.
Let $\{\e_1,\dots,\e_{d-1}\}$ be a standard basis of $H$. 
For every $x \in H$, we write $x = x_1\e_1 +x'$.
Notice that 
\begin{align*}
\Sigma^{-1/2} x & = \Sigma^{-1/2} (x_1\e_1) +x' = (I + (\sigma^{-1}-1)ww^\top) (x_1\e_1) +x' 
\\
&= x + (\sigma^{-1}-1) bx_1w  \;. 
\end{align*}
In particular, this implies that $\e_1^\top \Sigma^{-1} \e_1 = \rho^2:= a^2 + (b/\sigma)^2$.

Recall that for $x = \Sigma^{1/2} z$, $z \in \tilde B_{ij}$, it must be the case that $x \in B_{ij}$.
This implies that 
\begin{align*}
    \tilde{\mathcal{T}}_{ij} = \Pr_{z \sim \cN(0,I)}[z \in \tilde B_{ij}] = \Pr_{x \sim \cN(0,\Sigma')}[x \in  B_{ij}],
\end{align*}
where the covariance matrix $\Sigma'$ satisfies $\Sigma'\e_i = \e_i$, with $i =2,\dots,d-1$ and $\e_1^\top\Sigma'\e_1 = \rho^{-2}$.
Recalling that $\mathcal{T}_{ij} = \Pr_{x \sim \cN(0,I)}[x \in B_{ij}]$, we have that
$\abs{\mathcal{T}_{ij}-\tilde{\mathcal{T}}_{ij}} \le \dtv(\cN(0,I),\cN(0,\Sigma'))\le \sqrt{\mathrm{KL}(\cN(0,\Sigma')||\cN(0,I))/2}$.
Since $\Sigma'$ has only one eigenvalue different from that of $I$, we know that 
\begin{align*}
    \mathrm{KL}(\cN(0,\Sigma')||\cN(0,I)) = \frac{1}{2}(\rho^{-2}-1 -\log(\rho^{-2})) \le (\rho^{2}-1)^2/2 \;.
\end{align*}
Recall that $\rho^2 = a^2 + (b/\sigma)^2 = 1 + b^2(\sigma^{-2}-1)$, which gives that 
\begin{align*}
    \mathrm{KL}(\cN(0,\Sigma')||\cN(0,I)) \le (b^2(\sigma^{-2}-1))^2/2 \;.
\end{align*}
We thus obtain that 
\begin{align*}
    \abs{\mathcal{T}_{ij}-\tilde{\mathcal{T}}_{ij}} \le (b/\sigma)^2/2 \le \mathcal{T}_{ij}/8.
\end{align*}
This gives 
\begin{align*}
    \Pr_{z\sim \cN(0,I)}[(\tilde w^*_{ij}\cdot z)(w\cdot z)<0, y^*(x) \in \{i,j\}] \ge \tan \tilde \theta \mathcal{T}_{ij}^2/30 - \tilde O(\sqrt{\log k}\tan^2\tilde\theta/\tilde \Phi_{ij}).
\end{align*}
By \Cref{lm critical angle}, we have that $\tilde \Phi_{ij} \ge \Omega(\sqrt{\mathcal{T}_{ij}}\Phi_{ij})$.
Since $\tan \tilde \theta = b/(a \sigma) = \sigma^{-1} \tan \theta$, we know that when $\sin \theta<\sqrt{\mathcal{T}_{ij}}/C$, for some large constant $C>1$, for $\sigma \in[2\sin\theta/\sqrt{\mathcal{T}_{ij}},4\sin\theta/\sqrt{\mathcal{T}_{ij}}]$, we have
$ \tan \tilde{\theta}  = \Omega\!\left(\sqrt{\mathcal T_{ij}}\right)$. Therefore,  \Cref{col surface} implies that  
\begin{align*}
    \Pr_{z\sim \cN(0,I)}[(\tilde w^*_{ij}\cdot z)(w\cdot z)<0, y^*(x) \in \{i,j\}]  \ge \Omega(\mathcal{T}_{ij}^{4.5}\Phi_{ij}/\sqrt{\log k}).
\end{align*}
\end{proof}

\subsection{Proof of \Cref{lm surface localization}}\label{app proof surface localization}

\begin{lemma}[Restatement of \Cref{lm surface localization}]
Consider the problem of agnostic learning a $k$-multiclass linear classifier under the standard Gaussian distribution. Given any $i \neq j \in [k]$ and $\eps,\delta \in (0,1)$, \Cref{alg pairwise localization surface} draws $m=d\poly(k/\eps,\log(1/\delta))$ examples from $D$, runs in $\poly(m)$ time and with probability at least $1-\delta$ outputs a halfspace $\widehat{h}_{ij}$ such that $\err_{ij}(\widehat{h}_{ij}) \le \poly(\mathcal{T}_{ij}^{-1}\Phi_{ij}^{-1})\log k \mathcal{O}_{ij}+\eps$, where 
$\mathcal{O}_{ij} = \sum_{a\neq b:a \in \{i,j\}\; \mathrm{or }\; b \in \{i,j\}}\opt_{a,b}$, $\opt_{a,b} = \Pr_{(x,y) \in D}[f^*(x) = a, y=b]$.  
\end{lemma}

\begin{proof}[Proof of \Cref{lm surface localization}]
    We first show that by running \Cref{alg pairwise} for $i,j$ with parameter $\eps/2$, we are able to get a halfspace with normal vector $w_{ij}$ such that either $\err_{ij}(h_{w_{ij}}) \le \poly(\mathcal{T}_{ij}^{-1}\Phi_{ij}^{-1})\mathcal{O}_{ij}\log k+\eps$ or 
    $\theta(w_{ij},w^*_{ij}) \le \tilde \Theta(\mathcal{T}^2_{ij} \Phi_{ij}/\sqrt{\log k})$.

    By \Cref{lm pairwise}, it follows that with probability at least $1-\delta/2$, we have 
    \begin{align*}
        \err_{ij}(h_{w_{ij}}) \le \tilde O (\sqrt{\mathcal{O}_{ij}}) + \eps/2 \;.
    \end{align*}
    If $\tilde O (\sqrt{\mathcal{O}_{ij}}) \le \eps/2$, we are done, because in this case, we have $\err_{ij}(h_{w_{ij}}) \le \eps \le \poly(\mathcal{T}_{ij}^{-1}\Phi_{ij}^{-1})\mathcal{O}_{ij}+\eps$. 
    So, we can, without loss of generality, assume that $\tilde O (\sqrt{\mathcal{O}_{ij}}) \ge \eps/2$.
    Now, if $\poly(\mathcal{T}_{ij}^{-1}\Phi_{ij}^{-1})\mathcal{O}_{ij}\log k \le \tilde O (\sqrt{\mathcal{O}_{ij}})$.
    Then it must be the case that $\tilde O(\sqrt{\mathcal{O}_{ij}}) \le  \tilde\Theta(\mathcal{T}^{4}_{ij}\Phi_{ij}/\log k)$ and $\err_{ij}(h_{w_{ij}}) \le \tilde \Theta(\mathcal{T}^{4}_{ij}\Phi_{ij}/\log k)$.
    But by \Cref{col surface}, if $\theta(w_{ij},w^*_{ij})> \tilde\Theta(\mathcal{T}^2_{ij}\Phi_{ij}/\sqrt{\log k})$, then 
    \begin{align*}
        \err_{ij}(h_{w_{ij}}) \ge \Pr_{x\sim \cN(0,I)}[(w^*_{ij}\cdot x)(w_{ij}\cdot x)<0, f^*(x) \in \{i,j\}]-\mathcal{O}_{ij} \ge \tilde \Theta(\mathcal{T}^{4}_{ij}\Phi_{ij}/\sqrt{\log k}),
        \end{align*}
    which gives a contradiction.

Now given $w \in \s^{d-1}$ such that $\theta(w,w^*_{ij})<\tilde \Theta(\mathcal{T}^2_{ij}\Phi_{ij}/\sqrt{\log k})$, we define 
the localization matrix $\Sigma \in \R^{d \times d}$ to be a symmetric matrix such that $\Sigma w = \sigma^2 w$ and $\Sigma u = u$, for every $u \perp w$.
For $\phi \in (0,\pi/2)$ such that $\phi \ge \theta(w,w^*_{ij})$, we choose $\sigma = \min\{ 3\sin \phi/\sqrt{\mathcal{T}_{ij}},1\}$.

For each $x \in \R^d$, we write $x = \Sigma^{1/2} z$, $z \sim \cN(0,I)$ we will associate an artificial label $\tilde y(z):= y(x)$ for each $z \in \R^d$ and we define the artificial mistake set as 
\begin{align*}
    \widetilde E_{ij}(w) = \{(z,\tilde y) \mid \tilde y = i, w\cdot z \le 0\} \cup \{(z,\tilde y) \mid \tilde y = j, w\cdot z \ge 0\}.
\end{align*}
Denote by $D'(w)$ the distribution of $(z,\tilde y)$.

We define the random vector
\begin{align*}
    G_{ij}(w):=\proj_{w^\perp}(z\Ind\{w\cdot z\le 0,  \tilde y=i\}-z\Ind\{w \cdot z \ge 0, \tilde y = j\}).
\end{align*}
Denote by $g_{ij}(w) = \E_{(z,\tilde y)\sim D'(w)}[G_{ij}(w)]$ and $$g^*_{ij}(w) = \E_{z\sim \cN(0,I)}[z \Ind\{(w\cdot z)(\tilde w^*_{ij}\cdot z) \le 0, \tilde y=y^*=i\} - z \Ind\{(w\cdot z)(\tilde w^*_{ij}\cdot z) \le 0, \tilde y=y^*=j\}] \;.$$
Assuming that $\err_{ij}(h_w) \ge \poly(\mathcal{T}_{ij}^{-1}\Phi_{ij}^{-1})\mathcal{O}_{ij}+\eps$,
we next lower bound the correlation between $g_{ij}(w)$ and $w^*_{ij}$.
Notice that $\phi>\eps$, because otherwise, $\theta(w,w^*_{ij})<\eps$, which implies that $\err_{ij}(h_w) \le \mathcal{O}_{ij}+\eps$.
Thus, we will consider two cases.
In the first case, we have $\sin\theta < 3\sin\phi/4$.
In this case, 
we have $\sin \phi - \sin \theta \ge \eps/16$.
By \Cref{lm gaussian event mean}, we have that $\norm{g_{ij}(w)} \le O(1)$. 
Also by Hoeffding's inequality, for every $\delta>0$, by taking $d\poly(1/\eps)\log(1/\delta)$ examples from $D'(w)$, which can be done via a rejection sampling from $D$ with an acceptance rate at least $\eps$,
the empirical mean estimation $\widehat{g}_{ij}(w)$ satisfies $\norm{g_{ij}(w) - \widehat{g}_{ij}(w)}< \eps/100$.
This implies that by setting the step size as $\mu = \eps^2/C$, we have that after the update $w' = \proj_{\s^{d-1}}(w+\mu \hat{g}_{ij}(w))$, we have
$\norm{w-w'} \le \eps^2/C$.
This in turn implies that $\sin \theta(w',w^*_{ij}) \le \sin \phi - \eps^4/C^2$.
In the rest of the proof, we will assume that $\sin \theta(w,w^*_{ij}) \ge 3\sin \phi/4$. This implies that the choice of $\sigma \in [2\sin \theta(w,w^*_{ij})/\sqrt{\mathcal{T}_{ij}},4\sin \theta(w,w^*_{ij})/\sqrt{\mathcal{T}_{ij}}]$.

Since $G_{ij}(w) \perp w$, we know that 
\begin{align*}
    g_{ij}(w) \cdot w^* & = g_{ij}(w) \cdot (a w + b u) = b g_{ij}(w) \cdot u \\&\ge b g^*_{ij}(w) \cdot u - b \E_{(z,\tilde y)\sim D'(w)}[\abs{z\cdot u}\Ind\{y^*\neq \tilde y, y^* \in \{i,j\}\;\mathrm{or}\; \tilde y \in \{i,j\}\}]\;.
\end{align*}

Notice that 
\begin{align*}
    \Pr_{(z,\tilde y) \sim D'(w)}[y^*\neq \tilde y, y^* \in \{i,j\}\;\mathrm{or}\; \tilde y \in \{i,j\}] = \Pr_{x\sim \cN(0,\Sigma)}[f^*\neq y, f^* \in \{i,j\}\;\mathrm{or}\;  y \in \{i,j\}] \le \mathcal{O}_{ij}/\sigma.
\end{align*}
By \Cref{fact:simpleperturbation}, we have 
\begin{align*}
    \E_{(z,\tilde y)\sim D'(w)}[\abs{z\cdot u}\Ind\{y^*\neq \tilde y, y^* \in \{i,j\}\;\mathrm{or}\; \tilde y \in \{i,j\}\}] =\tilde{O}(\mathcal{O}_{ij}/\sigma).
\end{align*}

By \Cref{col surface}, since $\theta(w,w^*_{ij}) \le \tilde \Theta(\mathcal{T}^2_{ij}\Phi_{ij}/\sqrt{\log k})$, we know that 
\begin{align*}
 \tilde\Omega(\mathcal{T}_{ij}^{-7.5}\Phi_{ij}^{-2} \log k)\mathcal{O}_{ij}   \le \err_{ij}(h_w) \le \tilde O(\mathcal{T}_{ij})\sin\theta(w,w^*_{ij}) + \mathcal{O}_{ij}, 
\end{align*}
which implies that $\sin \theta \ge \Omega(\mathcal{T}_{ij}^{-8.5}\Phi_{ij}^{-2} \log k)\mathcal{O}_{ij}$ and $\tilde{O}(\mathcal{O}_{ij}/\sigma)\le \tilde O (\mathcal{T}_{ij}^{9}\Phi_{ij}^2/\log(k))$.

Recall that $\tilde w^*_{ij} = a \sigma w + bu$, which implies that for every $z$ such that $y^*(z) = i$, if $(w\cdot z)(\tilde w ^*_{ij}\cdot z) \le 0$, then we have $u \cdot z \ge 0$. 
Similarly, if $(w\cdot z)(\tilde w ^*_{ij}\cdot z) \le 0$, and $y^*(z) = j$ then we have $u \cdot z \le 0$.

By \Cref{lm convex body}, and Jensen's inequality, we have 
\begin{align*}
    &g^*_{ij}(w) \cdot u \\& \ge \sqrt{ \frac{\pi}{2}}\left( \Pr_{(z,\tilde y)\sim D'(w)}[(w\cdot z)(\tilde w^*_{ij}\cdot z) {\le} 0, \tilde y{=}y^*{=}i]^2 + \Pr_{(z,\tilde y)\sim D'(w)}[(w\cdot z)(\tilde w^*_{ij}\cdot z) {\le} 0, \tilde y{=}y^*{=}j]^2\right) \\
    & \ge \Omega \left(\Pr_{(z,\tilde y)\sim D'(w)}[(w\cdot z)(\tilde w^*_{ij}\cdot z) \le 0, \tilde y=y^* \in \{i,j\}]^2\right)\;.
\end{align*}
By \Cref{lm localization surface}, we have that $\Pr_{(z,\tilde y)\sim D'(w)}[\widetilde{E}_{ij}(w)] \ge \Omega(\mathcal{T}_{ij}^{4.5}\Phi_{ij}/\sqrt{\log k})$. Since
\begin{align*}
    \Pr_{(z,\tilde y)\sim D'(w)}[\widetilde{E}_{ij}(w)] \le  \Pr_{(z,\tilde y)\sim D'(w)}[(w\cdot z)(\tilde w^*_{ij}\cdot z) {\le} 0, \tilde y{=}y^* {\in} \{i,j\}] + \Pr_{(z,\tilde y)\sim D'(w)}[ \tilde y {\in} \{i,j\}, y^* {\neq} \tilde y],
\end{align*}
and 
\begin{align*}
\Pr_{(z,\tilde y)\sim D'(w)}[ \tilde y \in \{i,j\}, y^* \neq \tilde y] =  \Pr_{x \sim \cN(0,\Sigma)}[y(x) \in \{i,j\}, f^*(x) \neq y(x)]  \le \mathcal{O}_{ij}/\sigma 
\end{align*}

\begin{align*}
    \Pr_{(z,\tilde y)\sim D'(w)}[(w\cdot z)(\tilde w^*_{ij}\cdot z) \le 0, \tilde y=y^* \in \{i,j\}] \ge \Omega(\mathcal{T}_{ij}^{4.5}\Phi_{ij}/\sqrt{\log k}) \;,
\end{align*}
and 
$g^*_{ij}(w) \cdot u \ge \Omega(\mathcal{T}_{ij}^{9}\Phi^2_{ij}/\log k)$.
This implies that whenever $\err_{ij}(h_w) \ge \tilde O(\mathcal{T}_{ij}^{-7.5}\Phi_{ij}^{-2} \log k)\mathcal{O}_{ij} + \eps$, we have 
\begin{align*}
    g_{ij}(w) \cdot w^* \ge b \tilde \Omega(\mathcal{T}_{ij}^{9}\Phi^2_{ij}/\log k).
\end{align*}
Therefore, $g_{ij}(w) \cdot w^* \ge \poly(\eps)$.
Furthermore, by \Cref{lm gaussian event mean}, we have $\norm{g_{ij}(w)} \le O(1)$.
Notice that by Hoeffding's inequality, for every $\delta>0$, by taking $d\poly( 1/\eps,\log(1/\delta))$ examples from $D'(w)$, which can be done via a rejection sampling from $D$ with an acceptance rate at least $\eps$,
the empirical mean estimation $\widehat{g}_{ij}(w)$ satisfies $\norm{g_{ij}(w) - \widehat{g}_{ij}(w)}< \eps/100$.
By \Cref{lm gradient descent}, this implies that 
the update satisfies 
\begin{align*}
    \norm{w^*_{ij}-w'}^2 \le \norm{w^*_{ij}-w}^2 - \poly(\eps),
\end{align*}
and thus $\sin\theta(w',w^*_{ij}) \le \sin(\phi) - \poly(\eps)$.
In summary, if we have $\norm{\hat{g}_{ij}(w)-g_{ij}(w)}\le \eps/100$, then as long as $\sin \phi \ge \sin\theta(w,w^*)$, and $\err_{ij}(h(w)) \ge \tilde O(\mathcal{T}_{ij}^{-7.5}\Phi_{ij}^{-2} \log k)\mathcal{O}_{ij} + \eps$, we have $\sin \theta(w',w) \le \sin \phi - \poly(\eps)$.

We now show the correctness of \Cref{alg pairwise localization surface}, let $T = 1/\poly(\eps)$.
Suppose that for $t \in [T]$, we have $\err_{ij}(h^t_{ij}) = \err_{ij}(\sign(w^t_{ij}\cdot x))>\tilde O(\mathcal{T}_{ij}^{-7.5}\Phi_{ij}^{-2} \log k)\mathcal{O}_{ij} + \eps$.
Then, for $t \in [T]$, with probability at least $1-\delta/T$, we have $\sin \theta(w^{(t)}_{ij},w^*_{ij}) \le \sin \phi_{t} = \sin \phi_0 - t \poly(\eps)$.
This implies that after $T$ rounds, $\sin \theta(w^{(T)}_{ij},w^*_{ij})<0$, which gives contradiction.
This implies that there must be some $t \in [T]$ such that $\err_{ij}(\sign(w^t_{ij}\cdot x))\le \tilde O(\mathcal{T}_{ij}^{-7.5}\Phi_{ij}^{-2} \log k)\mathcal{O}_{ij} + \eps$.

By running a hypothesis selection, we know that \Cref{alg pairwise localization surface} output a hypothesis with error at most $\tilde O(\mathcal{T}_{ij}^{-7.5}\Phi_{ij}^{-2} \log k)\mathcal{O}_{ij} + \eps$.
\end{proof}

\subsection{Proof of \Cref{th surface area}}\label{app proof surface area}

\begin{theorem}[Restatement of \Cref{th surface area}]
    Consider the problem of agnostically learning $k$-MLC over $\R^d$, there is an algorithm that draws $n=d\poly(k/\eps)\log(1/\delta)$ samples, runs in $\poly(n)$ time, and with probability at least $1-\delta$ outputs a hypothesis with error $k\log k\poly(\mathcal{T}^{-1}\Phi^{-1})\opt +\eps$, where $\mathcal{T} = \min_{(i,j):i \neq j \in [k]}\mathcal{T}_{ij}$ and $\Phi = \min_{(i,j):i \neq j \in [k]}\Phi_{ij}$, with $\Phi_{ij}:=\min \left\{|\tan \theta^*_{ij}|,1 \right\}$.
\end{theorem}

\begin{proof}[Proof of \Cref{th surface area}]
    We show that for $i\neq j \in [k]$, \Cref{alg pairwise localization surface} runs in $\poly(dk/\eps)$ time and outputs a halfspace $\hat{h}_{ij}$ such that $\err_{ij}(\hat{h}_{ij})\le \log k\poly(\mathcal{T}_{ij}^{-1}\Phi_{ij}^{-1})\mathcal{O}_{ij}+\eps/k^2$.
Suppose this is true, then we have 
 \begin{align*}
    \err(h_W)  & = \sum_{j = 1}^k \Pr_{(x,y)\sim D}[h_W(x) \neq y, y=j] \le \sum_{j = 1}^k \sum_{i \neq j} \Pr_{(x,y)\sim D}[w_{ji}\cdot x \le 0, y=j] \\
    & = \frac{1}{2} \sum_{(i,j):i \neq j} (\Pr_{(x,y)\sim D}[w_{ij}\cdot x <0, y=i ]+\Pr_{(x,y)\sim D}[w_{ij}\cdot x >0, y=j ]) \\
    &\le \frac{1}{2} \sum_{(i,j):i \neq j} \left( \poly(\mathcal{T}_{ij}^{-1}\Phi_{ij}^{-1})\mathcal{O}_{ij}+\eps/k^2\right) \\
    &\leq \frac{\log k \poly(\mathcal{T}^{-1}\Phi^{-1})}{2} \sum_{(i,j):i \neq j} \left(\sum_{r \neq i} (\opt_{ri} + \opt_{ir}) + \sum_{r \neq j} (\opt_{rj} + \opt_{jr}) \right) + \eps\\
    & \le k \log k\poly(\mathcal{T}^{-1}\Phi^{-1})\opt + \eps.
\end{align*}
\end{proof}

\subsection{Discussion on Effective Decision Boundary and Critical Angles} \label{sec:randommlcs}
In this section, we provide a detailed discussion on the effective decision boundary and critical angles. We restate the definition below.

\begin{definition}[Effective Decision Boundary and Critical Angle]
    Let $f^*(x) = \argmax_{i \in [k]} (w^*_i \cdot x)$ be a multiclass linear classifier. For $i,j \in [k], i \neq j$, we define the effective decision boundary of $f^*$ for classes $i,j$ as 
    \begin{align*}
     B_{ij}(f^*):=\{x \mid w^*_{ij}\cdot x = 0, w^*_{ir}\cdot x \ge 0, w^*_{jr}\cdot x \ge 0, \forall r \in [k] \setminus \{i,j\}\},   
    \end{align*}
    where for $i\neq j \in [k]$, $w^*_{ij} = (w^*_i-w^*_j)/\|w^*_i-w^*_j\|$.
The Gaussian measure of $B_{ij}$ is defined as $\mathcal{T}_{ij} \eqdef \Pr_{x \sim \cN(0,I)}[B_{ij}]$, where $x$ is drawn from a $(d-1)$-dimensional standard Gaussian supported over $H_{ij}:=\{x \mid w^*_{ij}\cdot x = 0\}$. Define the critical angle $\tan \theta^*_{ij} = \min_{r \in [k]\setminus\{i,j\}}\{\min\{|\tan\theta(w^*_{ij},w^*_{ir})| ,|\tan\theta(w^*_{ji},w^*_{jr})|\}\}$ and $\Phi_{ij} = \min\{\tan \theta^*_{ij},1\}$
\end{definition}

We prove that, with high probability, a random multiclass linear classifier of sufficiently high dimension has boundary mass and critical angles (see \Cref{def:boundary}) at least \(1/\poly(k)\).
The intuition is as follows: A multiclass linear classifier is fully determined by the scores \(w_i \cdot x\). In sufficiently large dimension, the vectors \(w_i\) are nearly orthogonal with high probability. This immediately gives us that the critical angles are non-trivial. Moreover, it implies that the classifier is close to the MLC \( \arg\max_{i \in [k]} x_i\). For this canonical classifier, by independence, each point on a boundary between any two classes is equally likely to belong in any class. Consequently, a \(1/\poly(k)\) fraction of the boundary mass survives inside the corresponding class regions. 

We prove these two ingredients in turn: first, we show that the critical angles are non-trivial in \Cref{lem:random-mlc-critical-angles}, and then we show that the boundary mass is non-trivial in \Cref{lem:random-mlc-boundary-mass}. Finally, combining these two bounds via a union bound gives \Cref{thm:random-mlc-regular}.

\begin{lemma}[Random multiclass linear classifiers have non-trivial critical angles]
\label{lem:random-mlc-critical-angles}
Let $C>0$ be a sufficiently large universal constant. Let $w_1,\ldots,w_k \in \mathbb{S}^{d-1}$ be drawn independently and uniformly at random, and define the multiclass linear classifier $f(x)=\arg\max_{i\in[k]} w_i\cdot x$. For each pair $i,j\in[k]$, let $\theta^*_{i,j}$ denote the critical angle of $f$, as in \Cref{def:boundary}, and define $\Phi_{i,j}=\min\{\tan\theta^*_{i,j},1\}$. If $d\ge C\log(k/\delta)$, then with probability at least $1-\delta$, simultaneously for all pairs $(i,j), i\neq j$, we have $\Phi_{i,j}\ge 1/C$.
\end{lemma}

\begin{proof}
For $i \neq j$, define
$    w_{ij} := {(w_i-w_j)}/{\|w_i-w_j\|_2}.$
Define the acute critical-angle parameter
$
    \theta^*_{ij}
    \eqdef
    \min_{r\notin\{i,j\}}
    \min
    \left\{
         \theta(w_{ij},w_{ir}),
         \theta(w_{ij},w_{jr})
    \right\}$.

Note that since the $w_i$'s are independent and  distributed uniformly on the unit sphere,
we know that $\abs{w_i\cdot w_j}^2 \sim \mathrm{Beta}(1/2,(d-1)/2)$ (see Lemma~B.1 in \cite{kontonis2024gain}).
This implies that 
\begin{align*}
\pr[|w_i\cdot w_{j}|>\eps] \leq e^{\Omega(-d\eps^2)} \;.
\end{align*}
Therefore, for $\eps=c>0$, for a sufficiently large constant $C>0$, $d=C(1/c^2)\log(k/\delta)$ dimension suffices to have that $|w_i\cdot w_{j}|<c$ for all $i\neq j,i,j\in [k]$ simultaneously with probability at least $1-\delta$.
Now conditioning on this event for the rest of the proof.

We first start by proving that the critical angle is non-trivial with high-probability. 
Let distinct $i,j,r\in[k]$. 
We have 
\[w_{ij}\cdot w_{ir}
    =
    \frac{ (w_i-w_j)\cdot(w_i-w_r)
    }{
        \|w_i-w_j\|_2\|w_i-w_r\|_2
    }.
\]
The numerator satisfies
$
(w_i-w_j)\cdot (w_i-w_r)
    =
    1- w_i\cdot w_r- w_j\cdot w_i+ w_j\cdot w_r,
$.
 Hence, 
\[
    \left| (w_i-w_j)\cdot (w_i-w_r) -1
    \right|
    \le 3c.
\]
Moreover,
\[
    \|w_i-w_j\|_2^2
    =
    2-2 w_i\cdot w_j
    \in [2-2c,2+2c],
\]
and the same holds for $\|w_i-w_r\|_2^2$. Therefore
\[
    | u_{ij}\cdot w_{ir}|
    \le
    \frac{1+3c}{2(1-c)}
    \le \frac{2}{3},
\]
for $c>0$ sufficiently small.
Therefore we have that $\cos(\theta(w_{ij}w_{ir} ))\leq 2/3$ and $\sin(\theta(w_{ij}w_{ir} ))\geq \sqrt{1-(2/3)^2}$, which implies that $\tan(\theta(w_{ij}w_{ir} ))=\Omega(1)$ for all $r$. Similarly we have that the same $\theta(w_{ij}w_{jr})$. Therefore $\Phi_{i,j}\geq 1/C$, for all  $i\neq j,i,j\in[k]$. 
\end{proof}

\begin{lemma}[Random multiclass linear classifiers have non-trivial boundary mass]
\label{lem:random-mlc-boundary-mass}
Let $C>0$ be a sufficiently large universal constant. Let $w_1,\ldots,w_k \in \mathbb{S}^{d-1}$ be drawn independently and uniformly at random, and define the multiclass linear classifier $f(x)=\arg\max_{i\in[k]} w_i\cdot x$. For each pair $i,j\in[k]$, let $\mathcal T_{i,j}$ denote the mass of the effective decision boundary of $f$, as in \Cref{def:boundary}. If $d\ge C\log(k/\delta)$, then with probability at least $1-\delta$, simultaneously for all pairs $(i,j), i\neq j$, we have $\mathcal T_{i,j}\ge k^{-O(C)}$.
\end{lemma}

\begin{proof}
For $i \neq j$, define
$    w_{ij} := (w_i-w_j)/{\|w_i-w_j\|_2}.$
As in \Cref{def:boundary} let $B_{ij}(f)$ be the effective decision boundary
\[
    B_{ij}(f)
    :=
    \left\{
        x \in \{w_{ij}\cdot x=0\}
        :
        w_i\cdot x \ge w_r\cdot x
        \text{ and }
        w_j\cdot x \ge w_r\cdot x
        \text{ for all } r\notin\{i,j\}
    \right\}.
\]
Let
$   \mathcal T_{ij}
    :=
    \Pr_{x\sim N(0,I_{H_{ij}})}
    \left[x\in B_{ij}(f)\right],
    H_{ij}:=\{x:w_{ij}\cdot x=0\}.
$

Note that since the $w_i$'s are independent and  distributed uniformly on the unit sphere, we know that $\abs{w_i\cdot w_j}^2 \sim \mathrm{Beta}(1/2,(d-1)/2)$ (see Lemma~B.1 in \cite{kontonis2024gain}).
This gives
\begin{align*}
\pr[|w_i\cdot w_{j}|>\eps] \leq e^{\Omega(-d\eps^2)} \;.
\end{align*}
Therefore, for $\eps=c>0$, $d=C(1/c^2)\log(k/\delta)$ dimension suffices to have that $|w_i\cdot w_{j}|<c$ for all $i,j\in [k],i\neq j$ simultaneously with probability at least $1-\delta$.
Now we will be conditioning on this event for the rest of the proof.

We proceed to lower bound the mass of the effective decision boundary. 
 Fix a pair $i\neq j$. Let
$
    x\sim N(0,I_d),
    s_\ell := w_\ell\cdot x.
$
Then $(s_1,\ldots,s_k)$ is a centered Gaussian vector with covariance
$
  \cov(s_i,s_j)=  w_i\cdot w_j.
$
The condition $x\in H_{ij}$ is equivalent to
    $s_i=s_j$.
Therefore, we can write $\mathcal T_{ij}$ as
\[
    \mathcal T_{ij}
    =
    \Pr\left[
        s_i=s_j\ge s_r
        \text{ for all } r\notin\{i,j\}
        \mid s_i=s_j
    \right].
\]
Define
$   D:=s_i-s_j,
    y:=\frac{s_i+s_j}{2}.
$
Conditioning on $D=0$, the common score of classes $i$ and $j$ is $y$.
Furthermore, 
\[
    \var(y\mid D=0)
    =
    \var(y)
    =
    \frac{1+ w_i\cdot w_j}{2}
    \in
    \left[
        \frac{1-c}{2},
        \frac{1+c}{2}
    \right].
\]
For each $r\notin\{i,j\}$, the conditional distribution of $s_r$ given
$D=0$ and $y=t$ is Gaussian. Its conditional variance is at most $1$,
and its
conditional mean is linear in $t$, i.e., 
\[
    \E[s_r\mid d=0,y=t]=\beta_r t, \beta_r \eqdef \cov(s_r,y)/\var(y) \;.
\]
Since
\[
  \cov(s_r,y)
    =
    \frac{w_r\cdot w_i+ w_r\cdot w_j
    }{2},
\]
and since $\var(y)\ge (1-c)/2$, we have that
\[
    |\beta_r|
    \le
    \frac{2c}{1-c}
    \le 3c.
\]
Now choose
$
    A:=c_0\sqrt{\log k},
    ,h:=\frac{1}{\sqrt{\log k}}.
$
Since $y\mid D=0$ is a one-dimensional Gaussian with variance $\Theta(1)$,
there is a universal constant $C_1>0$ such that
\[
    \Pr[y\in[A,A+h]\mid D=0]
    \ge
    \frac{1}{k^{C_1}\sqrt{\log k}}.
\]
Condition on any value $t\in[A,A+h]$. For each $r\notin\{i,j\}$,
\[
    \E[s_r\mid D=0,y=t]\le 3c t.
\]
Choosing $c>0$ sufficiently small, we have
\[
    t-\E[s_r\mid D=0,y=t]\ge \frac{t}{2}.
\]
Since the conditional variance of $s_r$ is at most $1$, the Gaussian tail bound
implies
\[
    \Pr[s_r>t\mid D=0,y=t]
    \le
    \exp(-c_1 t^2)
    \le
    k^{-c_2}
\]
for universal constants $c_1,c_2>0$.
Taking the constant $c_0$ in the definition 
of $A$ sufficiently large ensures $c_2\ge 3$. Hence, by a union bound over all
$r\notin\{i,j\}$, 
\[
    \Pr[
        \exists r\notin\{i,j\}: s_r>y
        \mid D=0,\; y\in[A,A+h]
    ]
    \le
    k\cdot k^{-3}
    \le
    \frac12 .
\]
Therefore,
\[
\begin{aligned}
    \mathcal T_{ij}
    &=
    \Pr[
        s_r\le y
        \text{ for all } r\notin\{i,j\}
        \mid D=0
    ] \\
    &\ge
    \Pr[y\in[A,A+h]\mid D=0]\cdot \frac12 \\
    &\ge
    \frac{1}{k^{C_1}\sqrt{\log k}}.
\end{aligned}
\]
The argument holds uniformly for every pair $i<j$.
Thus, 
$\min_{i<j}\mathcal T_{ij}
    \ge
    \frac{1}{k^{C'}},$
for a universal constant $C'>0$.
\end{proof}
\begin{corollary}
    [Random multiclass linear classifiers are regular]
\label{thm:random-mlc-regular}
Let $C>0$ be a sufficiently large universal constant. Let $w_1,\ldots,w_k \in \mathbb{S}^{d-1}$ be drawn independently and uniformly at random, and define the multiclass linear classifier $f(x)=\arg\max_{i\in[k]} w_i\cdot x$. For each pair $i,j\in[k]$, let $\mathcal T_{i,j}$ and $\theta^*_{i,j}$ denote, respectively, the mass of the effective decision boundary and the critical angle of $f$, as in \Cref{def:boundary}, and define $\Phi_{i,j}=\min\{\tan\theta^*_{i,j},1\}$. If $d\ge C\log(k/\delta)$, then with probability at least $1-\delta$, simultaneously for all relevant pairs $(i,j)$, we have $\mathcal T_{i,j}\Phi_{i,j}\ge k^{-O(C)}$.
\end{corollary}
\begin{proof}
Apply \Cref{lem:random-mlc-critical-angles} with failure probability $\delta/2$ and \Cref{lem:random-mlc-boundary-mass} with failure probability $\delta/2$. By a union bound, with probability at least $1-\delta$, both conclusions hold simultaneously for all relevant pairs $(i,j)$. Hence, for every relevant pair $(i,j), i\neq j$,
\[
    \mathcal T_{i,j}\Phi_{i,j}
    \ge k^{-O(C)}/C
    =
    k^{-O(C)}.
\]
This proves the corollary.
\end{proof}

\end{document}